\documentclass{article}

\PassOptionsToPackage{numbers, compress}{natbib}

\usepackage[main, final]{neurips_2026}
\usepackage{subcaption}

\DeclareCaptionLabelFormat{subfigfull}{\figurename~\thefigure#2}
\usepackage[T1]{fontenc}    
\usepackage{xcolor}         
\usepackage[colorlinks=true,linkcolor=blue!50!black,citecolor=blue!50!black,urlcolor=blue!50!black,pdftitle={OSCAR: Offline Spectral Covariance-Aware Rotation for 2-bit KV Cache Quantization},pdfauthor={Zhongzhu Zhou, Donglin Zhuang, Jisen Li, Ziyan Chen, Shuaiwen Leon Song, Ben Athiwaratkun, Xiaoxia Wu}]{hyperref}       
\usepackage{url}            
\usepackage{makecell}
\usepackage{booktabs}       
\usepackage{threeparttable
} 
\usepackage{nccmath}
\usepackage{enumitem}
\usepackage{caption} 
\usepackage{amsfonts}       
\usepackage{amsmath}        
\usepackage{amsthm}
\usepackage{amssymb}
\usepackage{nicefrac}       
\usepackage{microtype}      
\usepackage{colortbl}       
\usepackage{graphicx}       
\usepackage{algorithm}
\usepackage{algorithmic}
\usepackage{multirow}
\usepackage{wrapfig}

\graphicspath{{figures/}}

\newcommand{\softmax}{\mathrm{softmax}}
\newcommand{\diag}{\mathrm{diag}}
\newcommand{\tr}{\mathrm{tr}}
\newcommand{\argmax}{\mathrm{argmax}}
\newcommand{\clip}{\mathrm{clip}}
\newtheorem{proposition}{Proposition}
\newtheorem{theorem}{Theorem}
\newtheorem{lemma}{Lemma}

\newif\ifoscarlogo
\oscarlogotrue
\newcommand{\oscarlogo}{%
  \ifoscarlogo\raisebox{-0.32\height}{\includegraphics[height=1.75em]{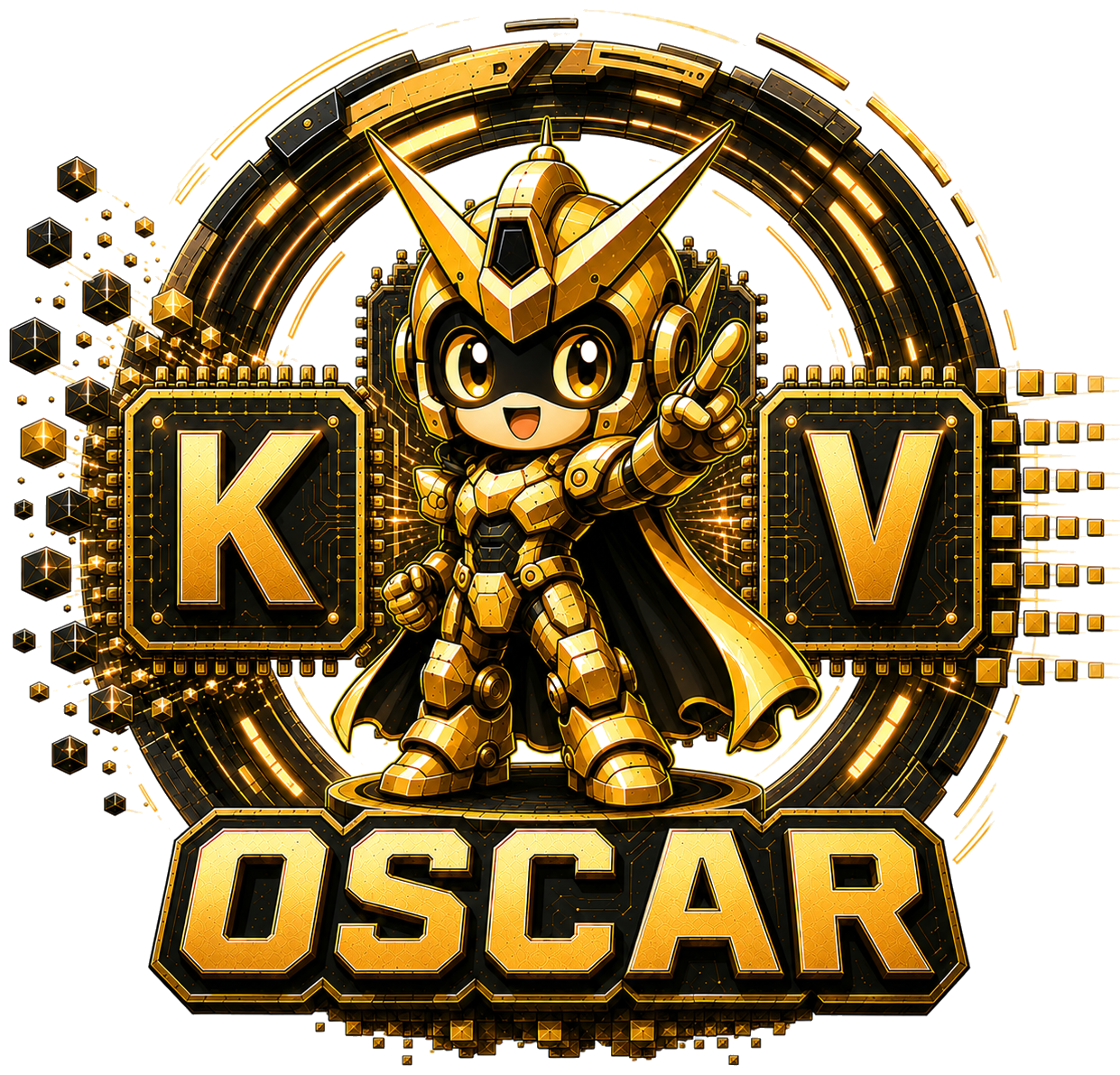}}\hspace{0.3em}\fi}

\title{\texorpdfstring{\oscarlogo}{}OSCAR: Offline Spectral Covariance-Aware Rotation for 2-bit KV Cache Quantization}

\author{%
  Zhongzhu Zhou$^{*,1,2}$ \quad
  Donglin Zhuang$^{*,1,2}$ \quad
  Jisen Li$^{1,3}$ \quad
  Ziyan Chen$^{2}$ \\
  Shuaiwen Leon Song$^{1}$ \quad
  Ben Athiwaratkun$^{1}$ \quad
  Xiaoxia Wu$^{\dagger,1}$ \\[6pt]
  $^{1}$Together AI \quad
  $^{2}$University of Sydney \quad
  $^{3}$University of Illinois Urbana-Champaign \\[4pt]
  $^{*}$Equal contribution. \quad $^{\dagger}$Corresponding author.%
}

\begin{document}

\maketitle

\begin{abstract}


INT2 KV-cache quantization is attractive for long-context LLM serving, yet making it both accurate and practical remains challenging. Data-oblivious rotations such as Hadamard transforms flatten channel outliers, but still degrade sharply at INT2 because they ignore how downstream attention queries the cache. We propose OSCAR, an ultra-low-bit KV-cache quantization method that estimates \textit{attention-aware covariance} targets offline and uses them to derive fixed per-layer rotations and clipping thresholds. By aligning the quantization basis with the second-moment structures consumed by attention, OSCAR preserves attention fidelity at 2-bit precision. Beyond theoretical analysis under a frozen-error surrogate, we build a serving-ready OSCAR system with fused Triton write and decode kernels that preserve paged KV-cache layouts and prefix caching, integrating into modern serving engines such as SGLang and vLLM.

We evaluate OSCAR on reasoning and coding benchmarks with generation traces of up to 32K tokens across five tasks. On Qwen3-4B-Thinking-2507 and Qwen3-8B, OSCAR narrows the BF16 accuracy gap to 3.78 and 1.21 points, respectively, where Hadamard-only INT2 collapses, and matches BF16 within run-to-run variance on Qwen3-32B and GLM-4.7 (358B parameters). On RULER-NIAH up to 128K tokens, OSCAR substantially outperforms QuaRot-style INT2 across context lengths. System-wise, OSCAR compresses KV-cache memory by nearly $8\times$, boosts multi-request throughput by up to $7.8\times$ at large batch sizes within the same memory budget, and accelerates batch-size-1 decoding by up to $3.1\times$ over BF16 by cutting KV memory traffic.

\end{abstract}
\begin{center}
{\normalsize
\href{https://github.com/FutureMLS-Lab/OSCAR}{\raisebox{-0.18em}{\includegraphics[height=1.1em]{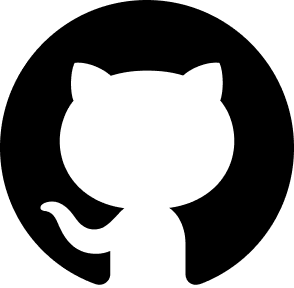}}\hspace{0.25em}\textbf{Code}}
\quad | \quad \href{https://oscar-quantize.github.io/}{\raisebox{-0.18em}{\includegraphics[height=1.1em]{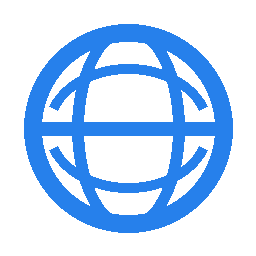}}\hspace{0.25em}\textbf{Website}} \quad | \quad
\href{https://huggingface.co/Zhongzhu/OSCAR-RotationZoo}{\raisebox{-0.18em}{\includegraphics[height=1.1em]{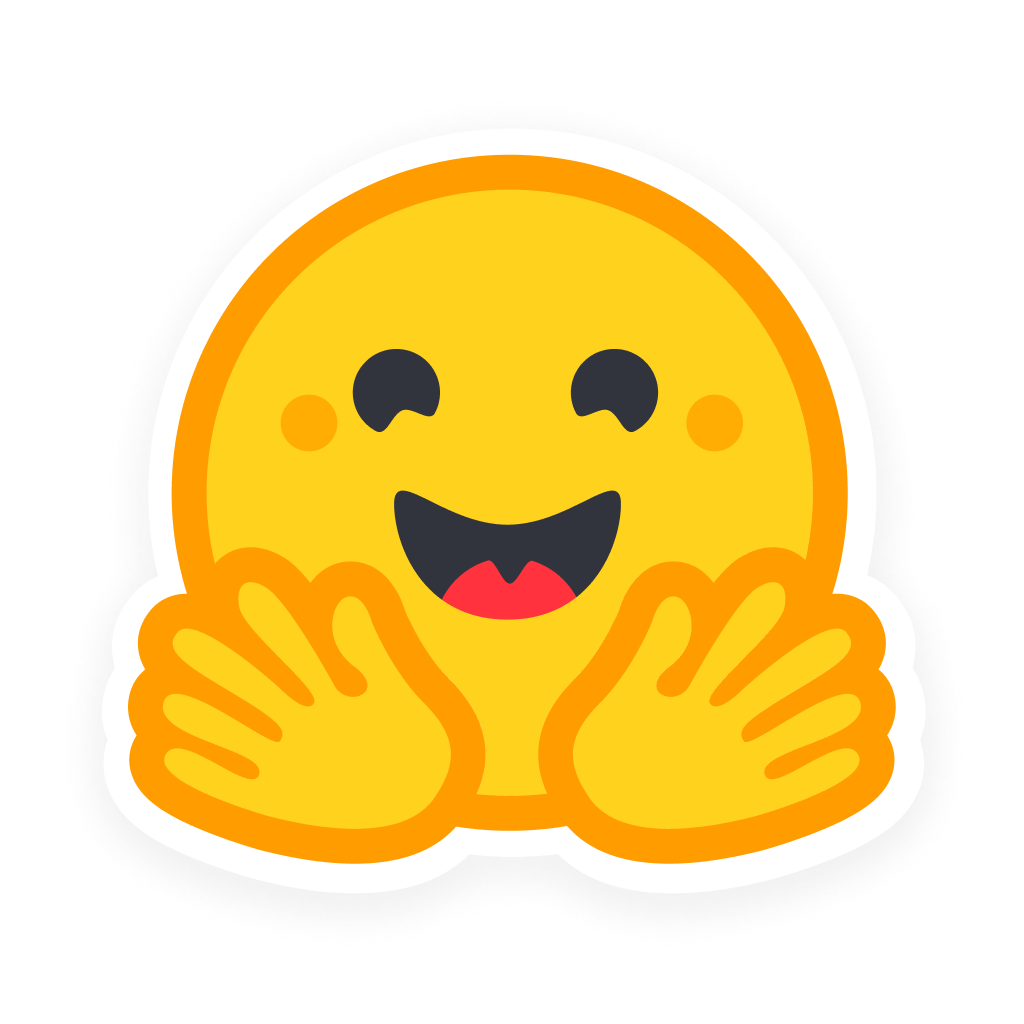}}\hspace{0.25em}\textbf{RotationZoo}}}
\end{center}
\vspace{-0.8em}

\section{Introduction}
\label{sec:introduction}

Long-context inference has made the key--value (KV) cache one of the primary bottlenecks in serving large language models.
During autoregressive decoding, the cache grows linearly with context length, batch size, and model depth, and every generated token must read cached keys and values from GPU memory.
Compressing the KV cache is therefore a direct way to support larger batch sizes and reduce memory bandwidth traffic~\citep{zhang2023h2o,sharma2024minikv,ge2023model,hooper2024kvquant,yue2024wkvquant,xia2025kitty,su2025rotatekv,li2025a}.
Among the available design points, token-wise INT2 quantization offers a compelling system trade-off: it provides a large memory reduction while retaining a regular, fixed-width layout that maps cleanly onto paged attention kernels.

Aggressively compressing KV caches to ultra-low precision (e.g., INT2) remains difficult because KV activations exhibit severe channel-wise outliers: a small subset of channels carries large magnitudes while most channels remain modest~\citep{hooper2024kvquant}. Under 2-bit quantization, these outliers stretch the quantization scale and collapse normal values onto one or two quantization levels, degrading attention quality. Rotation-based quantization mitigates this issue by applying a fixed orthogonal transform, such as a Hadamard rotation, to spread extreme values across channels and produce a more uniform dynamic range~\citep{ashkboos2024quarot,tseng2024quip,liu2024spinquant,jia2026saw}. Because an orthogonal rotation preserves tensor shapes and requires no per-channel routing or irregular sparse metadata, it remains naturally compatible with paged KV-cache layouts~\citep{vllm} and FlashAttention-style fused decode kernels~\citep{dao2022flashattention,dao2023flashattention2,shah2024flashattention,zadouri2026flashattention}: each KV vector is rotated into a better-conditioned basis before quantization and rotated back during attention~\citep{jia2026saw}.


\begin{figure}[t]
    \centering
    \includegraphics[width=0.96\textwidth]{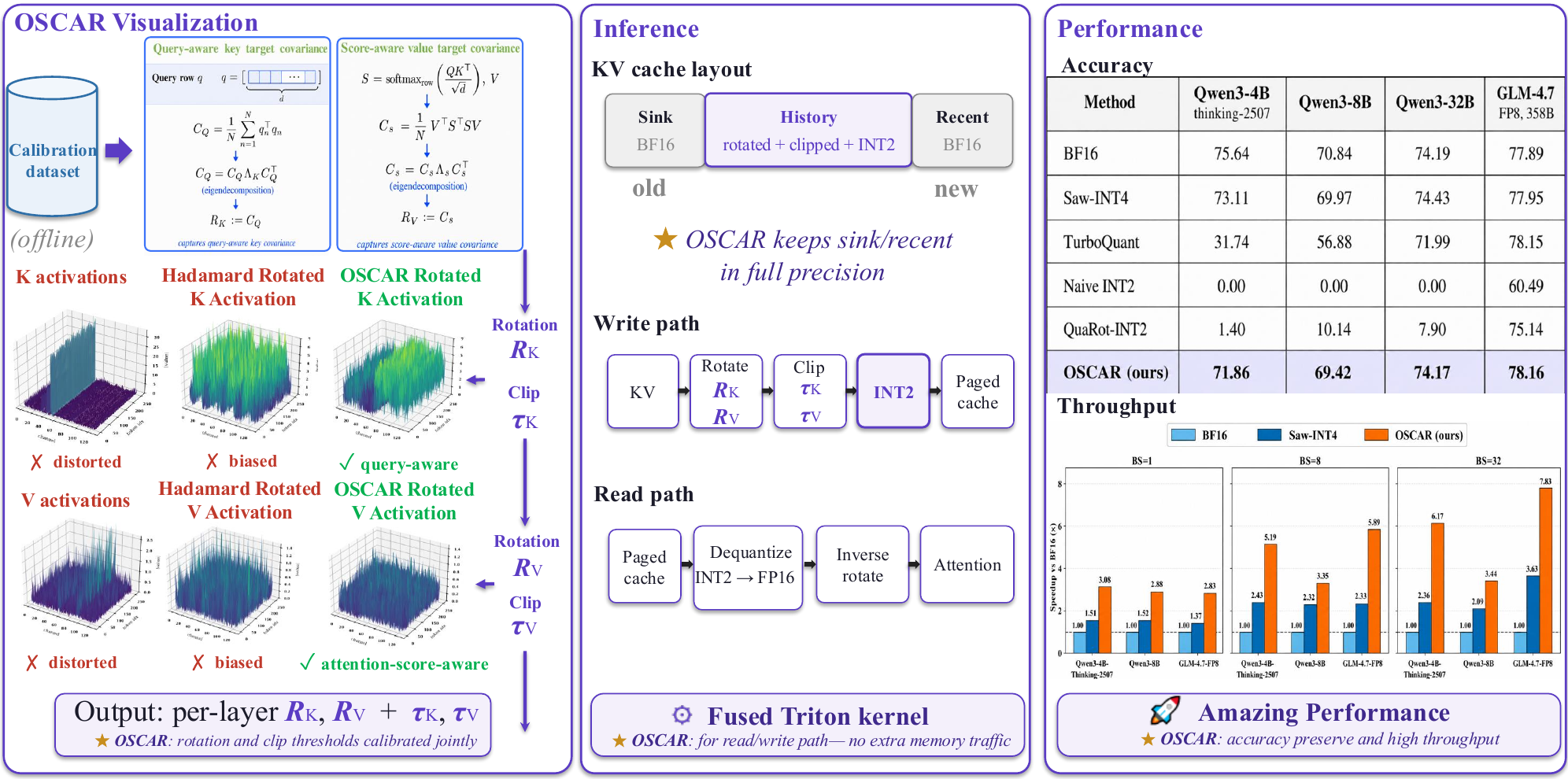}
    \caption{\textbf{OSCAR pipeline overview.}
    Offline, OSCAR estimates attention-aware key/value covariance rotation and shows how the resulting rotation makes KV activations more uniform: Hadamard mixing flattens raw peaks, while the OSCAR rotation separates directions that matter more or less to attention.
    Online, the serving path keeps sink and recent tokens in BF16 while applying the fixed rotate--clip--INT2 path to history KV tokens inside the paged SGLang cache.
    The right panels summarize that OSCAR preserves accuracy across models while achieving throughput close to the QuaRot(INT2) serving path and higher throughput than BF16.}
    \label{fig:OSCAR_pipeline}
\end{figure}
However, a data-oblivious rotation only equalizes channel magnitudes without considering which feature directions matter to attention.
At INT2, where only four quantization levels are available, this distinction is critical: quantization error along eigenvectors heavily queried by $Q$ perturbs attention logits far more than error along weakly queried directions. Because attention consumes keys through query inner products and values through score-weighted sums rather than raw Euclidean reconstruction, the rotation basis should be derived from attention statistics themselves.
Motivated by this observation, we propose OSCAR, an INT2 KV quantizer that estimates attention-aware covariance targets from a lightweight offline calibration pass and uses them to construct separate key and value rotations alongside per-token clipping thresholds. Figure~\ref{fig:OSCAR_pipeline} (left) illustrates that while Hadamard rotation partially flattens raw channel spikes, combining attention-aware covariance alignment with Hadamard mixing produces a much smoother activation distribution for 2-bit quantization.



Our contributions are summarized as follows 
\begin{itemize}[leftmargin=*, itemsep=2pt, topsep=2pt, parsep=0pt]
    \item \textbf{We identify the missing target in rotation-based INT2 KV quantization.}
    Data-oblivious rotations flatten raw activation outliers, but INT2 generation quality depends on preserving attention logits and score-weighted value outputs; consequently, the rotation basis should be derived from downstream attention statistics rather than raw cache reconstruction.
    \item \textbf{We propose OSCAR, an attention-aware calibration framework for 2-bit KV caches.}
    From a single offline calibration pass, OSCAR estimates query- and score-aware second-moment targets for keys and values and composes their eigenbases with Hadamard mixing and bit-reversal grouping. We analyze these targets under a frozen-error surrogate objective and evaluate OSCAR across models from 4B to 358B parameters, retaining near-BF16 reasoning and coding accuracy at 2.28 bits per KV element.
    \item \textbf{We build a serving-ready INT2 KV-cache system.}
    OSCAR stores historical KV tokens in a uniform rotated INT2 paged layout while keeping attention sinks and a sliding recent window in BF16~\citep{xiao2024streamingllm,barbero2025first,liu2024kivi,duanmu2024skvq}.
    Integrated into SGLang and vLLM~\citep{oscar2026sglangpr,oscar2026vllmpr} with custom fused Triton write and decode kernels that preserve prefix caching~\citep{sglang}, OSCAR delivers up to $7.8\times$ higher job throughput and $3.1\times$ faster batch-1 decoding at 100K context length.
    

\end{itemize}

Due to space constraints, we provide a full discussion in Appendices~\ref{sec:related_work} and~\ref{sec:discussion} on how our design builds on and differs from prior work.
\vspace{-0.15cm}
\section{Preliminaries and Motivation}
\label{sec:motivation_prelim}
\vspace{-0.1cm}

\textbf{Attention and KV cache.} We use \emph{row-vector} notation throughout. For simplicity, we define single-head attention~\citep{vaswani2017attention} as follows. Given a sequence of $T$ hidden states $\{x_t\}_{t=1}^T$ with $x_t\in\mathbb{R}^{1\times d}$ and projection weights $W_Q,W_K,W_V\in\mathbb{R}^{d\times d}$, the query, key, and value matrices are
$Q=[q_1;\dots;q_T]\in\mathbb{R}^{T\times d}, K=[k_1;\dots;k_T]\in\mathbb{R}^{T\times d}, V=[v_1;\dots;v_T]\in\mathbb{R}^{T\times d}$,
where $d$ is the head dimension and $q_t=x_tW_Q\in\mathbb{R}^{1\times d}$, $k_t=x_tW_K\in\mathbb{R}^{1\times d}$, and $v_t=x_tW_V\in\mathbb{R}^{1\times d}$. The attention score matrix is $S=\softmax_{\mathrm{row}}(QK^\top/\sqrt d)\in\mathbb{R}^{T\times T}$, and the attention output is $O=SV\in\mathbb{R}^{T\times d}$, or equivalently
$o_i=\sum_{t=1}^T S_{i,t}v_t$. During autoregressive inference, $K_{1:t}=[k_1;\dots;k_t]$ and $V_{1:t}=[v_1;\dots;v_t]$ are cached across steps as the KV cache. Grouped-query attention (GQA)~\citep{ainslie2023gqa}, covariance derivations, and serving procedures are detailed in Appendices~\ref{app:theory} and~\ref{app:algorithm_flow}. 

%


\textbf{Quantization notation.}
A $b$-bit quantizer consists of a quantization map $\mathcal Q^+:\mathbb R\to\mathcal C_b$ and a dequantization map $\mathcal Q^-:\mathcal C_b\to\mathbb R$, where $\mathcal C_b$ is a discrete set of $2^b$ representable codes.  Their composition gives the quantize--dequantize map $\mathcal Q(x)=\mathcal Q^-(\mathcal Q^+(x))$.  A symmetric uniform map can be written as
$
\mathcal Q(x)
=
s\cdot
\clip\!\left(
\left\lfloor {x}/{s}\right\rceil,
-\tau,\tau
\right),
$
where $s>0$ is the quantization scale and $\tau$ is the clipping limit. 
For a matrix $X$ consisting of row vectors, we apply quantization row-wise.\footnote{Row-wise quantization is used in our theoretical surrogate; in implementation, asymmetric affine INT2 quantization is applied per token over contiguous channel groups of size 128, 64, or 32 (Appendix~\ref{app:scale-determination}).} Given an orthogonal rotation matrix $R$~\citep{ashkboos2024quarot}, the reconstructed caches are $\widehat K =\mathcal Q(KR)R^\top$ and $\widehat V=\mathcal Q(VR)R^\top$. 

\textbf{Why raw-cache reconstruction is not enough.}
Rotation is effective for KV-cache quantization because it spreads large channel values into a basis with a more uniform dynamic range~\citep{ashkboos2024quarot}. However, simple data-oblivious rotations such as Hadamard or random orthogonal transforms are often insufficient: they smooth the cached tensors, but do not identify which directions are most important for downstream attention. A standard tensor-reconstruction view minimizes $\|K-\mathcal Q(K)\|_F^2$ and $\|V-\mathcal Q(V)\|_F^2$, yet attention does not consume $K$ and $V$ through their Euclidean reconstruction errors.  Keys are used through logits, and values are used through the attention-weighted aggregation.
For keys $K$, the downstream logit distortion is
\useshortskip
\begin{equation}
\label{eq:K}
\|QK^\top-Q\widehat K^\top\|_F^2
=\tr\!\left((K-\widehat K)Q^\top Q(K-\widehat K)^\top\right),
\end{equation}
which is controlled by the query covariance $Q^\top Q$ rather than by $K^\top K$ alone.  For values $V$ with quantized $\widehat V =\mathcal Q(V)$, the downstream output distortion is
\useshortskip
\begin{equation}
\label{eq:V}
    \|SV-S\widehat V\|_F^2
=
\tr\!\left((V-\widehat V)^\top S^\top S(V-\widehat V)\right),
\end{equation}
which depends on how attention scores weight the value rows. 
Thus, if the goal is to reduce empirical attention distortion rather than raw-cache reconstruction error, the rotation should be estimated from target covariance induced by the attention computation itself.

\begin{figure}[t]
    \centering
    \includegraphics[width=0.98\textwidth]{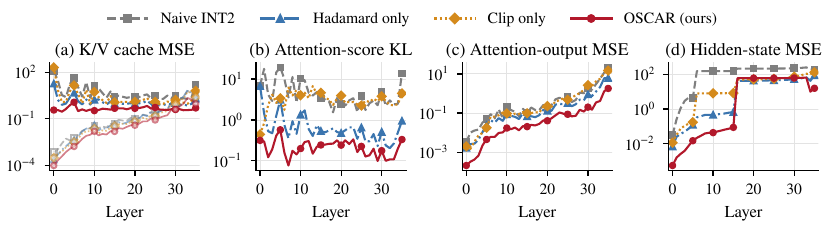}
    \caption{Quantization error. OSCAR limits it at every stage (Qwen3-4B-Thinking-2507, AIME).
    \textbf{(a)} MSE in quantized $K$ and $V$ caches (keys solid, values dashed).
    \textbf{(b)} $\mathrm{KL}(p_{\mathrm{FP16}}\,\|\,p_q)$ between FP16 and quantized attention score distributions.
    \textbf{(c)} MSE at the attention-block output (post-$W_O$).
    \textbf{(d)} MSE in the propagated hidden states across layers.
    Curves compare naive INT2, Hadamard-only preprocessing, clip-only outlier control, and full OSCAR.}
    \label{fig:motivation_target_mse}
\end{figure}

Figure~\ref{fig:motivation_target_mse} illustrates this gap across all 36 layers of Qwen3-4B-Thinking-2507.
Naive INT2, Hadamard-only rotation, and clipping alone still leave substantial attention-score KL divergence and block-output MSE.
In contrast, calibrating the rotation from attention-induced target covariances before quantization consistently lowers distortion across all four stages. A concrete numerical walkthrough of each rotation factor is given in Appendix~\ref{app:intuition_uhp}.


\vspace{-0.15cm}
\section{Algorithm Design: Offline Calibration and Justification}
\label{sec:method}
\label{subsec:offline_calibration}
\vspace{-0.1cm}
The method has two phases: an offline calibration phase that estimates target covariances, constructs per-layer rotation matrices, and sets per-token clipping thresholds; and an online serving phase that applies these fixed transforms inside a mixed-precision paged cache layout.
All target covariances and rotations are estimated once offline from a small calibration set of post-RoPE activations.



\textbf{Query-aware key target covariance.}
For a query row $q$, with the attention-aware key target covariance  $Q^\top Q$ introduced in equation \eqref{eq:K}, its empirical estimator is
$
C_Q
=
\frac{1}{N}\sum_{n=1}^N q_n^\top q_n;
$ we apply eigen decomposition as in App.~\ref{app:covariance_pca}, $C_Q=U_Q\Lambda_QU_Q^\top$, and define the rotation as $R_k:=U_Q$.

\textbf{Score-aware value target covariance.}
For the  attention matrix $S=\softmax_{\mathrm{row}}(QK^\top/\sqrt d)$, we heuristically define the target covariance in equation \eqref{eq:V}:
$
C_S
=
\frac{1}{N}V^\top S^\top S V.
$
We then compute the eigendecomposition as in App.~\ref{app:covariance_pca}, $C_S=U_S\Lambda_SU_S^\top$, and define the raw rotation as $R_v:=U_S$.

\textbf{OSCAR rotations.}
Following rotation-based low-bit quantization methods \citep{ashkboos2024quarot,chee2023quip}, we compose the base rotation with a Hadamard transform $H_{\mathrm{Had}}$ (App.~\ref{app:hadamard_transform}) and a bit-reversal permutation $P_{\mathrm{br}}$ to form the final rotations for keys and values:
\vspace{-0.025cm} 
\useshortskip
\begin{equation}
R_K = U_Q\, H_{\mathrm{Had}}\, P_{\mathrm{br}},
\qquad
R_V = U_S\, H_{\mathrm{Had}}\, P_{\mathrm{br}}.
\label{eq:method_composed_rot}
\end{equation}
Here $R_k$ and $R_v$ are the covariance eigenbases, $H_{\mathrm{Had}}$ redistributes channel outlier energy across coordinates, and $P_{\mathrm{br}}$~\citep{cooley1965fft,vanloan1992computational} applies a deterministic bit-reversal permutation across channels to balance quantization groups (see Appendix~\ref{app:intuition_uhp} for an intuitive walkthrough and Appendix~\ref{app:permutation} for permutation ablations).

\textbf{Scale determination and per-token clipping.}
Our quantization backend follows standard post-training quantization practice. We use affine asymmetric INT2 quantization with scale and zero point for both keys and values, together with percentile-based clipping to control outliers. These mechanisms are common in low-bit activation and KV-cache quantization pipelines~\citep{liu2024kivi,hooper2024kvquant,xiao2023smoothquant,frantar2022gptq}, with technical details shown in App.~\ref{app:scale-determination}.

\textbf{Optimality of targets under a frozen diagonal residual surrogate.}
\label{app:target_derivation}
Below we show that under a frozen diagonal residual surrogate, the target covariances $C_Q
=
\frac{1}{N}\sum_{n=1}^N q_n^\top q_n$ and $
C_S
=
\frac{1}{N}V^\top S^\top S V
$ minimize the surrogate attention distortion. The proof, the streaming score-weighted estimator used for $C_S$ on long calibration sequences, and an empirical check verifying that the deployed quantization residual is strongly diagonal-dominant on real activations ($98.4\%$ diagonal energy for keys and $95.5\%$ for values) are provided in Appendices~\ref{app:ambient_basis_simple_variants}--\ref{app:residual_check}.

\begin{theorem}[Optimal base rotation under a frozen diagonal residual surrogate]
\label{thm:ambient_basis_simple_variants}
Consider the frozen-error surrogate objectives for the key and value base rotations:
$
\tilde{\mathcal L}_K(R_k) = \tr(R_k^\top C_Q R_k E_K), 
R_k^\top R_k = I_d$ and $\tilde{\mathcal L}_V(R_v)
=
\tr\!\left(R_v^\top C_S R_v E_V\right),
R_v^\top R_v = I_{d},$
where $C_Q = U_Q \Lambda_Q U_Q^\top$ and $C_S = U_S \Lambda_S U_S^\top$ have eigenvalues sorted in non-increasing order, and the frozen residual second moments are fixed diagonal matrices
$E_K = \operatorname{diag}(\mu_1,\dots,\mu_d)$ with $\mu_1 \le \cdots \le \mu_d$ and
$E_V = \operatorname{diag}(\eta_1,\dots,\eta_{d})$ with $\eta_1 \le \cdots \le \eta_{d}$.
Then $R_k = U_Q$ and $R_v = U_S$ minimize $\tilde{\mathcal L}_K(R_k)$ and $\tilde{\mathcal L}_V(R_v)$ over orthogonal rotations.
\end{theorem}

\vspace{-0.15cm}
\section{System Design: Online Serving with 2-bit KV Cache}
\label{sec:system}
\vspace{-0.1cm}

\textbf{KV Cache Layout.}
We integrate OSCAR into the SGLang~\citep{sglang} serving stack as an INT2 KV-cache mode with full compatibility with paged-attention~\cite{vllm}.
The implementation preserves two short high-precision windows: the first $S_0$ tokens, which behave as attention sinks, and the most recent $W$ tokens before the current position.
The rest of the middle context is stored in INT2 after the fixed OSCAR rotation.
Thus, at decoding position $t$, the logical cache consists of
\[
    \underbrace{[1,S_0]}_{\mathrm{bf16}\text{ sink}}
    \;\Vert\;
    \underbrace{[S_0+1,t-W]}_{\mathrm{int2}\text{ history}}
    \;\Vert\;
    \underbrace{[t-W+1,t]}_{\mathrm{bf16}\text{ recent}}
\]

\textbf{KV Cache Update.} During prefill, the runtime writes cache rows through a fully fused Triton kernel~\citep{Triton}.
For a token with BF16 rows $k_t,v_t$ and per-token clipping thresholds $\tau_t^{(K)},\tau_t^{(V)}$, the stored INT2 codes are
\vspace{-0.025cm} 
\useshortskip

\[
    k_t^+ = \mathcal Q_2^+(\clip(k_tR_K,\tau_t^{(K)})),
    \qquad
    v_t^+ = \mathcal Q_2^+(\clip(v_tR_V,\tau_t^{(V)})),
\]
with four 2-bit values packed per byte.

New decode tokens are first written to the recent window as $k_tR_K$ and $v_tR_V$.
As decoding advances, the oldest recent token is demoted into the INT2 middle region by a fused Triton kernel that applies the same clip--quantize operation.
OSCAR further eliminates runtime value-rotation overhead by folding $R_V$ and $R_V^\top$ offline into the projection weights $W_V$ and $W_O$.


\textbf{Decoding Attention Kernel.}
During decoding, OSCAR partitions each request's cache indices into BF16 (sink + recent) and INT2 segments on the GPU.
The INT2 kernel unpacks bytes, applies the stored scale/zero parameters, and accumulates them in floating point.
Existing decoding attention kernels typically consist of two kernel launches~\cite{dao2023flashdecoding}, one for parallel processing on KV cache segments along the sequence dimension.
The second one is used to merge partial attention results from segments with online softmax~\cite{dao2022flashattention}. OSCAR introduces an additional kernel for BF16 KV cache attention, then reuses the second merge kernel to piggyback on the merge of high-precision partial results.
Since the BF16 (sink + recent) segment has orders of magnitude fewer elements than the INT2 segment, the overhead is negligible.


\section{Experiments}
\label{sec:experiments}


\subsection{Experimental Setup}
\label{subsec:exp_setup}

\textbf{Models \& Benchmarks.}
We evaluate OSCAR on four model configurations: \textbf{Qwen3-4B-Thinking-2507}, \textbf{Qwen3-8B}, \textbf{Qwen3-32B}, and \textbf{GLM-4.7-FP8}~\citep{Qwen3,glm4}.
These models span a compact reasoning model, mid-sized dense models, and a 358B MoE model, testing OSCAR across different scales and INT2 sensitivity regimes. We evaluate accuracy on five reasoning and coding benchmarks:
AIME25~\citep{AIME25}, GPQA-Diamond~\citep{rein2024gpqa}, HumanEval~\citep{chen2021evaluating}, LiveCodeBench v6~\citep{jain2024livecodebench}, and MATH500~\citep{hendrycks2021measuring,lightman2024verify}.

\begin{wraptable}{r}{0.5\textwidth}
  \vspace{-0.5em}
  \centering
   \small
    \caption{\textbf{AIME25 at 32K generation.}}
     \label{table:kitty-comapre}
  \setlength{\tabcolsep}{1.5pt}
  \begin{tabular}{l c c c}
  \toprule
  \textbf{Method} & \textbf{BPE} & \textbf{Qwen3-8B} & \textbf{Qwen3-32B} \\
  \midrule
  Original BF16 & 16.00 & 66.00$\pm$7.33 & 72.59$\pm$7.41 \\
  KIVI-KV2~\citep{liu2024kivi} & 2.25 & 52.33$\pm$9.00 & 57.41$\pm$9.26 \\
  KIVI-KV2*~\citep{liu2024kivi} & 2.26 & 57.67$\pm$9.00 & 59.05$\pm$12.38 \\
  Kitty~\citep{xia2025kitty} & 2.39 & 59.67$\pm$10.33 & 69.26$\pm$9.26 \\
  \rowcolor{blue!8}
  \textbf{OSCAR} & 2.38 & 66.67$\pm$3.33 & 74.00$\pm$5.48 \\
  \bottomrule
  \end{tabular}
\end{wraptable}

We also evaluate long-context retrieval with \textbf{RULER}~\cite{hsieh2024ruler} to test the long-sequence robustness. Particularly, RULER-NIAH is the cleanest stress test: the answer is explicitly present in the prompt, and the main question is whether quantized history tokens can be attended to in a long context.

\textbf{Hardware and Framework.}
Qwen3-4B-Thinking and Qwen3-8B are served on a single NVIDIA H100 (80\,GB); Qwen3-32B and GLM-4.7-FP8 are served on $2\times$H100 and $8\times$H100 with tensor parallelism, respectively.
All system-level runs use our SGLang~\citep{sglang} implementation.

\textbf{Generation Protocol \& Calibration.}
We use temperature $T{=}0.6$ ($T{=}1.0$), top-$p{=}0.95$, and top-$k{=}20$ for Qwen3-family (GLM-4.7), with thinking mode enabled for reasoning benchmarks.
Unless otherwise noted, each Qwen configuration is run with \textbf{5 independent seeds} and GLM-4.7-FP8 with \textbf{3 runs}; we report mean $\pm$ standard deviation.
All accuracy evaluations use a maximum generation length of \textbf{32768 tokens} and run end-to-end inside SGLang, using the same execution path as our system measurements.
All OSCAR parameters are estimated once from a small MMLU/GPQA-style calibration set.
For each model, we run one calibration pass and dump per-layer $Q,K,V$ activations (8878 tokens $\times$ number of layers), from which we compute the key/value rotations and calibration-derived clipping ratios, then reuse the same parameters for all benchmarks.
No task-specific calibration is used.

\textbf{Baselines.}
We use two baseline groups.
(1) \textit{Group A} contains channel-wise KV methods such as KIVI and Kitty.
These methods require residual buffers, channel-wise scales, promoted channels, or custom page layouts, and we do not have paged/fused kernels for them at 32K generation length.

For the only shared 32K accuracy setting, we show the reported Qwen3-8B and Qwen3-32B AIME25 results in the Table~\ref{table:kitty-comapre}; The BPE values include INT payload, BF16 scale/zero metadata with group size 128, and BF16 initial tokens where applicable; the Kitty row is the 12.5\% key-channel boost variant. (2) \textit{Group B} contains rotation-based methods: FP16/BF16, naive INT2/INT4, QuaRot-style Hadamard rotation, block-diagonal Hadamard (Saw-INT4)~\citep{jia2026saw}, and TurboQuant~\cite{vllm2026turboquantpr}.
For TurboQuant, we use the official vLLM implementation from PR~\#38479~\citep{vllm2026turboquantpr} and run the same 32K-generation evaluation; for fairness, we disable its mixed precision setup and quantize all layers. 

For OSCAR we always pair rotation with sink ($S{=}64$) and recent-window ($R{=}256$) BF16 protection, calibration-derived clip ratios ($c_K{=}0.96$ and $c_V{=}0.92$, see the calibration paragraph above), and per-token grouped asymmetric INT2 quantization.
In Table~\ref{tab:main_accuracy}, BPE counts INT payload plus BF16 scale/zero metadata with block size 128. For OSCAR, we report effective BPE at 128K context length, including the BF16 sink/recent tokens ($S{=}64$, $R{=}256$).\footnote{Qwen3-4B-Thinking-2507 and GLM-4.7-FP8 use group size $G{=}128$ (2.28 BPE at 128K context); Qwen3-8B and Qwen3-32B use $G{=}64$ (2.53 BPE at 128K when accounting for finer-grained BF16 scale/zero metadata; full $G\in\{128,64,32\}$ runs are reported in Appendix~\ref{app:qwen3_8b_full_runs}).} LiveCodeBench v6 is evaluated with the same 32K generation cap as the other tasks; this truncates some long code-generation outputs, so the reported LCB numbers are lower than they would be under a longer 128K generation budget.

\begin{table}[!t]
  \centering
  \scriptsize
  \renewcommand{\arraystretch}{0.92}
  \caption{INT2 KV-cache quantization methods compared on four model configurations and five benchmarks. Entries are $\mu \pm \sigma$ over 5 seeds, except GLM-4.7-FP8 (3 runs)*. BPE denotes effective bits per KV element at 128K context length; ``no MP'' means mixed precision is disabled. ``Drop'' is the average-column gap to the BF16 upper bound; smaller is better.}
  \label{tab:main_accuracy}
  \setlength{\tabcolsep}{3pt}
  \resizebox{\textwidth}{!}{
  \begin{tabular}{l l c c c c c c c c}
  \toprule
  \textbf{Model} & \textbf{Method} & \textbf{BPE} & \textbf{GPQA} & \textbf{HumanE} & \textbf{LCB v6} & \textbf{AIME25} & \textbf{MATH500} & \textbf{Mean} & \textbf{Drop} \\
  \midrule
  \multirow{6}{*}{\rotatebox[origin=c]{90}{\textbf{\makecell[c]{Qwen3-4B\\Thinking-2507}}}}
   & BF16                  & 16.00 & 67.27$\pm$1.80 & 94.05$\pm$0.54 & 48.66$\pm$2.20 & 74.67$\pm$1.83 & 93.55$\pm$0.33 & 75.64 & -- \\
   & Saw-INT4 ~\citep{jia2026saw} & 4.25 & 66.37$\pm$2.19 & 89.78$\pm$0.80 & 46.20$\pm$1.94 & 70.00$\pm$4.08 & 93.19$\pm$0.51 & 73.11 & -2.53 \\
   & TurboQuant (no MP)~\citep{zandieh2025turboquant} & 3.25& 41.41 & 31.83 & 0.58 & 16.67 & 68.20 & 31.74 & -43.90 \\
   & QuaRot-INT2~\citep{ashkboos2024quarot}           & 2.25 & 0.34$\pm$0.48 & 0.98$\pm$0.10 & 0.00$\pm$0.00 & 0.00$\pm$0.00 & 5.67$\pm$0.19 & 1.40 & -74.24 \\
   & Naive INT2            & 2.25 & 0.00 & 0.00 & 0.00 & 0.00 & 0.00 & 0.00 & -75.64 \\
  \rowcolor{blue!8}
   & \textbf{OSCAR (ours)}   & 2.28 & 64.95$\pm$1.16 & 92.24$\pm$1.02 & 45.38$\pm$1.97 & 64.00$\pm$3.65 & 92.75$\pm$0.39 & 71.86 & -3.78 \\
  \midrule
  \multirow{6}{*}{\rotatebox[origin=c]{90}{\textbf{\makecell[c]{Qwen3\\-8B}}}}
   & BF16                  & 16.00 & 56.67$\pm$2.30 & 85.95$\pm$1.01 & 49.01$\pm$2.13 & 70.00$\pm$3.33 & 92.59$\pm$0.62 & 70.84 & -- \\
   & Saw-INT4~\citep{jia2026saw} & 4.25 & 54.85$\pm$3.17 & 86.44$\pm$0.42 & 47.95$\pm$2.15 & 68.00$\pm$2.98 & 92.63$\pm$0.27 & 69.97 & -0.87 \\
   & TurboQuant (no MP)~\citep{zandieh2025turboquant}            & 3.25& 55.05 & 74.63 & 21.05 & 46.67 & 87.00 & 56.88 & -13.96 \\
   & QuaRot-INT2~\citep{ashkboos2024quarot}           & 2.25 & 14.98$\pm$0.63 & 9.80$\pm$0.76 & 0.58$\pm$0.58 & 2.22$\pm$1.57 & 23.13$\pm$1.88 & 10.14 & -60.70 \\
   & Naive INT2            & 2.25 & 0.00 & 0.00 & 0.00 & 0.00 & 0.00 & 0.00 & -70.84 \\
  \rowcolor{blue!8}
   & \textbf{OSCAR (ours)}   & 2.28 & 55.05$\pm$1.47 & 87.88$\pm$0.44 & 46.32$\pm$1.91 & 66.67$\pm$3.33 & 92.22$\pm$0.83 & 69.63 & -1.21 \\
  \midrule
  \multirow{6}{*}{\rotatebox[origin=c]{90}{\textbf{\makecell[c]{Qwen3\\-32B}}}}
   & BF16                  & 16.00 & 58.49$\pm$1.21 & 91.19$\pm$1.11 & 59.06$\pm$1.80 & 68.67$\pm$5.58 & 93.55$\pm$0.41 & 74.19 & -- \\
   & Saw-INT4~\citep{jia2026saw} & 4.25 & 59.29$\pm$0.76 & 90.85$\pm$0.47 & 56.02$\pm$2.41 & 72.67$\pm$4.90 & 93.31$\pm$0.10 & 74.43 & +0.24 \\
   & TurboQuant (no MP)~\citep{zandieh2025turboquant} & 3.25& 58.69 & 88.41 & 55.56 &  66.67 & 90.60 & 71.99 & -2.20 \\
   & QuaRot-INT2~\citep{ashkboos2024quarot} & 2.25 & 19.70$\pm$1.09 & 1.83$\pm$0.33 & 0.39$\pm$0.55 & 0.00$\pm$0.00 & 17.60$\pm$1.66 & 7.90 & -66.29 \\
   & Naive INT2            & 2.25 & 0.00 & 0.00 & 0.00 & 0.00 & 0.00 & 0.00 & -74.19 \\
  \rowcolor{blue!8}
   & \textbf{OSCAR (ours)}   & 2.28 & 60.40$\pm$3.47 & 90.12$\pm$0.71 & 53.57$\pm$0.67 & 74.00$\pm$5.48 & 92.75$\pm$0.46 & 74.17 & -0.02 \\
  \midrule
  \multirow{6}{*}{\rotatebox[origin=c]{90}{\textbf{\makecell[c]{GLM-4.7\\FP8, 358B}}}}
   & BF16                  & 16.00 & 73.23$\pm$1.33 & 91.46$\pm$0.65 & 49.12$\pm$0.59 & 80.00$\pm$3.33 & 95.66$\pm$0.61 & 77.89 & -- \\
   & Saw-INT4~\citep{jia2026saw} & 4.25 & 73.74$\pm$1.01 & 91.30$\pm$0.92 & 50.49$\pm$1.47 & 78.89$\pm$1.92 & 95.32$\pm$0.12 & 77.95 & +0.06 \\
   & TurboQuant (no MP)~\citep{zandieh2025turboquant}            & 3.25& 66.67 & 90.24 & 58.48 & 80.00 & 95.39 & 78.15 & +0.26 \\
   & QuaRot-INT2~\citep{ashkboos2024quarot}           & 2.25 & 68.01$\pm$3.79 & 89.87$\pm$1.78 & 48.54$\pm$4.79 & 78.89$\pm$1.92 & 90.40$\pm$0.35 & 75.14 & -2.75 \\
   & Naive INT2            & 2.25 & 54.55$\pm$3.54 & 86.50$\pm$1.43 & 37.03$\pm$6.41 & 37.78$\pm$1.92 & 86.60$\pm$0.80 & 60.49 & -17.40 \\
  \rowcolor{blue!8}
   & \textbf{OSCAR (ours)}   & 2.28 & 73.57$\pm$0.58 & 91.06$\pm$0.25 & 52.63$\pm$1.01 & 78.89$\pm$1.92 & 94.66$\pm$0.76 & 78.16 & +0.27 \\
  \bottomrule
  \end{tabular}%
  }
  \parbox{\textwidth}{\footnotesize *TurboQuant entries are single-run results as its vLLM path is too slow for repeated 32K evaluations under our time budget.}
\end{table}

\subsection{Main Results}
\label{subsec:exp_main_acc}
Table~\ref{tab:main_accuracy} reports the main accuracy comparison on Qwen3-4B-Thinking-2507, Qwen3-8B, Qwen3-32B, and GLM-4.7-FP8.

For OSCAR we use the configuration selected by the ablation studies in Section~\ref{subsec:exp_ablation}: sink$=$64, recent$=$256, calibration-derived clip thresholds, attention-aware key/value covariance, and asymmetric INT2 quantization. The key observation is that OSCAR is the only near-2-bit method that remains close to the BF16 accuracy frontier under the 32K-generation evaluation.
On Qwen3-4B-Thinking-2507 and Qwen3-8B, OSCAR reduces the BF16 gap to 3.78 and 1.21 points, respectively, while TurboQuant drops 43.90 and 13.96 points and rotation-only INT2 baselines largely collapse.
On Qwen3-32B and GLM-4.7-FP8, OSCAR is essentially tied with BF16 at 2.28 BPE, while Saw-INT4 uses 4.25 BPE; small per-task differences of 1--2 points in either direction reflect sampling and reasoning-length variation under a fixed 32K cap~\citep{lotfi2026quantized} (further models, tool use, and baseline comparisons are reported in Appendix~\ref{app:additional_results}).

\textbf{Long-Context Robustness.} Table~\ref{tab:long_context_ruler} reports RULER-NIAH accuracy from 4k to 128k tokens on the same serving-compatible baselines used in the main accuracy comparison.
The expected pattern is simple: short contexts should remain close to BF16, while longer contexts expose accumulated attention-logit error.
OSCAR should degrade more slowly than rotation-only INT2 baselines because its rotation target is chosen from the covariance seen by attention, not from raw-cache reconstruction, and because long contexts amplify accumulated KV quantization error.  Figure~\ref{fig:long_context_kl_scaling} provides a direct KL-based check before the retrieval results.
Across both Qwen3-4B-Thinking-2507 and Qwen3-8B, OSCAR is more stable as sequence length increases: its attention distribution stays closer to FP16, while naive INT2 and Hadamard-only rotation drift more quickly.
On Qwen3-4B-Thinking-2507, QuaRot-INT2 is already near zero at short contexts, whereas OSCAR stays close to BF16 through 16k and retains non-trivial retrieval at 64k and 128k.
On Qwen3-8B, QuaRot-INT2 remains usable only at 4k--8k and collapses after 16k, while OSCAR keeps substantially higher retrieval accuracy throughout the sweep.
The GLM-4.7-FP8 preliminary run shows a complementary large-model case: all three methods remain strong on this retrieval-only task, and OSCAR matches the BF16 curve up to 128k.
Together, the RULER results and KL curves indicate that attention-aware covariance calibration mainly helps when long histories make small KV errors accumulate over many steps.

\begin{table}[!htb]
  \centering
  \footnotesize
  \renewcommand{\arraystretch}{0.92}
  \caption{Long-context retrieval accuracy on \textbf{RULER-NIAH}~\citep{hsieh2024ruler} across context lengths from 4k to 128k tokens, on the Group~B serving-compatible baselines. Each number is averaged over 3 random seeds, except GLM-4.7-FP8 preliminary results (one run, as GLM-4.7 is very stable).}
  \label{tab:long_context_ruler}
  \setlength{\tabcolsep}{4pt}
  \resizebox{0.92\textwidth}{!}{%
  \begin{tabular}{l l c c c c c c}
  \toprule
  \textbf{Model} & \textbf{Method} & \textbf{4k} & \textbf{8k} & \textbf{16k} & \textbf{32k} & \textbf{64k} & \textbf{128k} \\
  \midrule
  \multirow{3}{*}{Qwen3-4B}
    & BF16          & 100.0$\pm$0.0 & 100.0$\pm$0.0 & 99.7$\pm$2.0 & 99.3$\pm$1.0 & 85.3$\pm$5.0 & 81.0$\pm$1.4 \\
    & QuaRot-INT2~\citep{ashkboos2024quarot}   & 0.0$\pm$0.0 & 0.8$\pm$0.2 & 0.0$\pm$0.0 & 0.0$\pm$0.0 & 15.6$\pm$3.2 & 0.0$\pm$0.0 \\
  \rowcolor{blue!8}
  (-Thinking-2507)  & \textbf{OSCAR}  & 99.7$\pm$0.1 & 100.0$\pm$0.0 & 97.8$\pm$0.2 & 87.6$\pm$0.1 & 61.9$\pm$1.3 & 39.5$\pm$1.0 \\
  \midrule
  \multirow{3}{*}{Qwen3-8B}
    & BF16          & 100.0$\pm$0.0 & 99.6$\pm$0.0 & 98.9$\pm$0.1 & 97.3$\pm$0.0 & 79.2$\pm$0.1 & 78.2$\pm$1.0 \\
    & QuaRot-INT2~\citep{ashkboos2024quarot}   & 74.0$\pm$1.0 & 80.6$\pm$0.3 & 19.0$\pm$2.0 & 9.8$\pm$1.0 & 0.0$\pm$0.0 & 0.0$\pm$0.0 \\
  \rowcolor{blue!8}
    & \textbf{OSCAR}  & 99.5$\pm$0.2 & 97.8$\pm$0.3 & 93.9$\pm$0.7 & 86.3$\pm$0.6 & 61.9$\pm$1.8 & 45.0$\pm$0.0 \\
  \midrule
  \multirow{3}{*}{GLM-4.7-FP8\textsuperscript{$\dagger$}}
    & BF16          & 100.0 & 100.0 & 100.0 & 100.0 & 98.8 & 97.2 \\
    & QuaRot-INT2~\citep{ashkboos2024quarot}   & 100.0 & 99.7 & 100.0 & 100.0 & 98.8 & 96.3 \\
  \rowcolor{blue!8}
    & \textbf{OSCAR}  & 100.0 & 100.0 & 100.0 & 100.0 & 98.8 & 97.2 \\
  \bottomrule
  \end{tabular}
  }
\end{table}

\begin{figure}[!htb]
    \centering
    \includegraphics[width=0.84\textwidth]{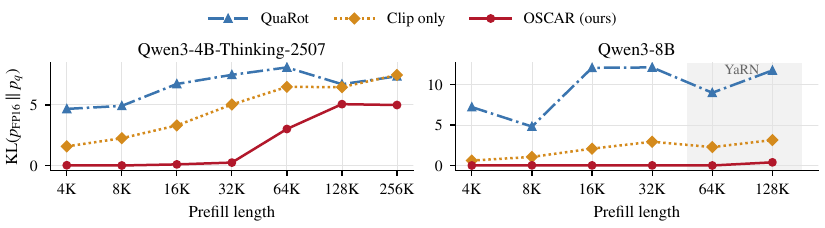}
    \caption{\textbf{OSCAR keeps attention distributions stable as context length grows.}
    KL divergence against the FP16 attention distribution across context lengths for clip-only INT2, QuaRot (Hadamard), and OSCAR on Qwen3-4B-Thinking-2507 and Qwen3-8B.}
    \label{fig:long_context_kl_scaling}
\end{figure}


\subsection{Ablation Studies}
\label{subsec:exp_ablation}
\textbf{Rotation Analysis: Decomposing $R$ and Comparing Rotation Targets.}
\label{subsubsec:abl_rotation}
OSCAR's composed rotation is $R = U\cdot H_{\mathrm{Had}}\cdot P_{\mathrm{br}}$, the product of an attention-aware eigenbasis $U$, a Hadamard transform $H_{\mathrm{Had}}$, and a bit-reversal permutation $P_{\mathrm{br}}$.
Table~\ref{tab:ablation_composed} removes each factor of $R$ in turn (\emph{top block}) and, on the same footing, changes the PCA target used to compute $U$ (\emph{bottom block}): OSCAR's attention-aware targets are replaced by raw-cache reconstruction targets ($K^\top K$/$V^\top V$), fixed Hadamard rotations, or no learned rotation.

Two observations emerge.
First, both the attention-aware eigenbasis $U$ and the Hadamard component $H_{\mathrm{Had}}$ contribute substantially; the bit-reversal permutation $P_{\mathrm{br}}$ does not change accuracy in floating-point math but improves quantization geometry by interleaving large and small eigenvalues so that per-group quantization sees a more uniform range.
Second, none of the alternative rotation targets (matched Hadamard $H_{\mathrm{Had}}\cdot P_{\mathrm{br}}$, raw $K^\top K$/$V^\top V$ reconstruction targets, or no rotation) matches the attention-aware $C_Q$/$C_S$ targets at the same INT2 budget.
This isolates the central claim of the paper: \emph{which} covariance matrix one diagonalizes matters more than \emph{whether} one diagonalizes at all.

\textbf{Sink and Recent Window Sizes.} We sweep $(S, R) \in \{(0,0), (32,128), (64,256), (128,512),$ $(256,1024)\}$ on Qwen3-4B-Thinking-2507 and report accuracy with the additional BF16 KV memory that each protection window introduces (Table~\ref{tab:ablation_sink_recent}).
A clear knee emerges at $(S, R){=}(64, 256)$: smaller windows leave noticeable accuracy on the table, while larger windows provide negligible additional accuracy at the cost of substantially more BF16 KV memory.
$(64, 256)$ is a sweet spot and thus used for the main results.

\begin{table}[!htb]
\centering
\footnotesize
\renewcommand{\arraystretch}{0.88}
\caption{Decomposition of the composed OSCAR rotation $R = U\cdot H_{\mathrm{Had}}\cdot P_{\mathrm{br}}$ (\emph{top block}) and comparison against alternative rotation-target choices used by prior work (\emph{bottom block}), all at the same INT2 bit budget with sink/recent protection on Qwen3-8B. Entries are mean $\pm$ std over 3 runs.}
\label{tab:ablation_composed}
\setlength{\tabcolsep}{3pt}
\resizebox{0.81\textwidth}{!}{%
\begin{tabular}{l c c c c c c}
\toprule
\textbf{Configuration} & \textbf{GPQA} & \textbf{HumanE} & \textbf{LCB v6} & \textbf{AIME25} & \textbf{MATH500} & \textbf{Mean} \\
\midrule
FP16 reference & 57.07$\pm$0.94 & 84.71$\pm$0.99 & 48.07$\pm$2.20 & 66.00$\pm$2.79 & 88.28$\pm$0.46 & 68.83 \\
\midrule
\multicolumn{7}{l}{\emph{Decomposition of OSCAR's composed rotation}} \\
\rowcolor{blue!8}
Full OSCAR: $U\cdot H_{\mathrm{Had}}\cdot P_{\mathrm{br}}$         & 55.89$\pm$1.27 & 87.97$\pm$0.19 & 46.00$\pm$1.88 & 67.78$\pm$1.92 & 92.38$\pm$0.20 & 70.01 \\
w/o $P_{\mathrm{br}}$                                   & 52.86$\pm$1.62 & 86.71$\pm$0.84 & 45.22$\pm$0.34 & 63.33$\pm$8.82 & 91.85$\pm$1.45 & 68.00 \\
w/o $H_{\mathrm{Had}}$ (only $U\cdot P_{\mathrm{br}}$)               & 51.35$\pm$1.91 & 83.78$\pm$0.37 & 19.49$\pm$2.05 & 21.11$\pm$1.92 & 82.97$\pm$0.80 & 51.74 \\
\midrule
\multicolumn{7}{l}{\emph{Alternative rotation-target choices}} \\
Matched Hadamard baseline ($H_{\mathrm{Had}}\cdot P_{\mathrm{br}}$, w/o $U$) & 40.74$\pm$3.44 & 38.54$\pm$1.18 & 4.48$\pm$0.68 & 15.56$\pm$1.92 & 64.80$\pm$1.29 & 32.82 \\
OSCAR with $K^\top K$/$V^\top V$ PCA target              & 38.72$\pm$1.27 & 35.53$\pm$0.67 & 3.70$\pm$0.68 & 13.33$\pm$3.33 & 64.33$\pm$0.53 & 31.12 \\
no rotation (only clip + sink + recent)                 & 18.52$\pm$3.29 & 0.00 & 0.39$\pm$0.68 & 2.22$\pm$1.92 & 0.00 & 4.23 \\
\bottomrule
\end{tabular}%
}
\end{table}

\begin{table}[!htb]
  \centering
  \footnotesize
  \renewcommand{\arraystretch}{0.88}
  \caption{Sink and recent window sizes on Qwen3-4B-Thinking-2507 (INT2 with OSCAR rotation and calibration-derived clip). $(S, R)$ denote the number of BF16 sink tokens and recent-window tokens, respectively. ``Extra BF16 KV'' reports the protected BF16-token fraction, $(S+R)/128\mathrm{K}$.}
  \label{tab:ablation_sink_recent}
  \setlength{\tabcolsep}{6pt}
  \resizebox{0.81\textwidth}{!}{%
  \begin{tabular}{l c c c c c c c}
  \toprule
  $(S, R)$ & \textbf{GPQA} & \textbf{HumanE} & \textbf{LCB v6} & \textbf{AIME25} & \textbf{MATH500} & \textbf{Mean} & \textbf{Extra BF16 KV} \\
  \midrule
  $(0, 0)$              & 0.00 & 0.00 & 0.00 & 0.00 & 0.00 & 0.00 & 0\,\% \\
  $(32, 128)$           & 57.58$\pm$2.20 & 91.91$\pm$0.19 & 40.74$\pm$2.36 & 55.56$\pm$5.09 & 92.65$\pm$0.95 & 67.69 & 0.12\,\% \\
  $\mathbf{(64, 256)}$  & \textbf{64.95$\pm$1.16} & \textbf{92.24$\pm$1.02} & \textbf{45.38$\pm$1.97} & \textbf{64.00$\pm$3.65} & \textbf{92.75$\pm$0.39} & \textbf{71.86} & \textbf{0.24\,\%} \\
  $(128, 512)$          & 65.32$\pm$0.77 & 93.17$\pm$0.32 & 45.22$\pm$0.33 & 67.78$\pm$1.92 & 93.32$\pm$0.12 & 72.96 & 0.49\,\% \\
  $(256, 1024)$         & 64.98$\pm$1.27 & 93.13$\pm$0.81 & 46.20$\pm$2.03 & 67.78$\pm$1.92 & 93.32$\pm$0.12 & 73.08 & 0.98\,\% \\
  \bottomrule
  \end{tabular}%
  }
\end{table}

\textbf{Clip Threshold.} The clip thresholds are estimated per layer from rotated calibration activations rather than tuned on downstream tasks. Table~\ref{tab:ablation_clip} sweeps global clip ratios on Qwen3-4B-Thinking and shows that the calibration-derived point $(c_K,c_V)=(0.96,0.92)$ is within 0.15 points of the best downstream grid point, confirming that OSCAR's clip thresholds do not need to be hand-tuned.

\textbf{Calibration Data Regime.} We also test whether OSCAR is sensitive to calibration volume or domain. Table~\ref{tab:ablation_calibration} shows that the default 8k GPQA-Diamond calibration is close to the best setting, supporting a small one-time calibration pass.

\begin{table}[!htb]
\centering
\scriptsize
\renewcommand{\arraystretch}{0.92}
\begin{minipage}[t]{0.35\textwidth}
\centering
\caption{Validation of calibration-derived clip thresholds (mean accuracy on Qwen3-4B-Thinking-2507). Rows: $c_K$; columns: $c_V$; bold: operating point $(0.96,0.92)$.}
\label{tab:ablation_clip}
\setlength{\tabcolsep}{3.5pt}
\resizebox{\linewidth}{!}{%
\begin{tabular}{@{}l|c c c c c@{}}
\toprule
$c_K \setminus c_V$ & 0.88 & 0.92 & 0.96 & 0.98 & 1.00 \\
\midrule
0.88 & 67.20 & 68.27 & 68.10 & 67.56 & 68.25 \\
0.92 & 69.15 & 70.00 & 70.46 & 70.74 & 70.33 \\
0.96 & 68.77 & \textbf{70.59} & 70.70 & 68.17 & 69.14 \\
0.98 & 65.22 & 65.76 & 64.12 & 64.79 & 67.31 \\
1.00 & 60.35 & 59.19 & 58.15 & 58.63 & 56.02 \\
\bottomrule
\end{tabular}%
}
\end{minipage}\hfill
\begin{minipage}[t]{0.57\textwidth}
\centering
\caption{Calibration-data regime on Qwen3-4B-Thinking-2507 (INT2 with OSCAR rotation, calibration-derived clip, sink/recent protection), sweeping volume and domain.}
\label{tab:ablation_calibration}
\setlength{\tabcolsep}{2.2pt}
\resizebox{\linewidth}{!}{%
\begin{tabular}{@{}l l c c c c c c@{}}
\toprule
\textbf{Domain} & \textbf{Tokens} & \textbf{GPQA} & \textbf{HumanE} & \textbf{LCB v6} & \textbf{AIME25} & \textbf{MATH500} & \textbf{Mean} \\
\midrule
\multirow{4}{*}{MMLU} & 2k & 60.61$\pm$2.20 & 92.52$\pm$0.14 & 39.57$\pm$2.36 & 55.56$\pm$5.09 & 93.12$\pm$0.46 & 68.28 \\
 & 8k & 60.44$\pm$2.04 & 92.07$\pm$0.76 & 40.74$\pm$0.89 & 60.00$\pm$5.77 & 92.52$\pm$0.58 & 69.15 \\
 & 16k & 64.14$\pm$0.51 & 92.93$\pm$0.53 & 43.86$\pm$2.55 & 61.11$\pm$10.72 & 92.38$\pm$0.53 & 70.88 \\
 & 32k & 61.28$\pm$1.54 & 92.20$\pm$0.32 & 46.00$\pm$2.36 & 57.78$\pm$5.09 & 92.25$\pm$0.31 & 69.90 \\
\midrule
WikiText & 8k & 61.11$\pm$0.51 & 92.72$\pm$0.63 & 42.11$\pm$2.11 & 58.89$\pm$3.85 & 92.65$\pm$0.50 & 69.50 \\
\textbf{GPQA} & \textbf{8k (def.)} & \textbf{62.96$\pm$1.27} & \textbf{92.56$\pm$0.00} & \textbf{44.64$\pm$1.47} & \textbf{62.22$\pm$1.92} & \textbf{92.65$\pm$0.31} & \textbf{71.01} \\
\bottomrule
\end{tabular}%
}
\end{minipage}
\end{table}


\subsection{End-to-End Serving Throughput}
\label{subsubsec:e2e_throughput}

\textbf{Kernel-Level Profiling.} We profile per-step decode latency on Qwen3-8B and GLM-4.7-FP8 to isolate the runtime cost of OSCAR. Table~\ref{tab:decode_kernel_profile} shows that OSCAR reduces attention time by shrinking KV traffic while adding only a small fused quantization cost.

\begin{table}[t]
\centering
\footnotesize
\renewcommand{\arraystretch}{0.90}
\caption{Per-decode kernel profiling on Qwen3-8B, 1$\times$H100, single decoding step latency in milliseconds.}
\label{tab:decode_kernel_profile}
\setlength{\tabcolsep}{2pt}
\resizebox{0.92\textwidth}{!}{%
\begin{tabular}{c|ccccc|ccccc}
\toprule
\multirow{2}{*}{\boldmath{$B$}} & \multicolumn{5}{c|}{\textsc{BF16}} & \multicolumn{5}{c}{\textsc{OSCAR}} \\
 & \textbf{GEMM} & \textbf{Attn} & \textbf{Quant} & \textbf{Other} & \textbf{Total} & \textbf{GEMM} & \textbf{Attn} & \textbf{Quant} & \textbf{Other} & \textbf{Total} \\
\midrule
1   & 5.6 (56.0\%) & 3.8 (37.9\%)  & -- & 0.6 (6.1\%) & 9.9  & 5.6 (69.9\%) & 1.3 (16.8\%)  & 0.5 (5.9\%) & 0.6 (7.5\%) & 8.0  \\
8   & 5.6 (53.0\%) & 4.4 (41.0\%)  & -- & 0.6 (6.0\%) & 10.7 & 5.7 (57.4\%) & 3.0 (30.1\%)  & 0.5 (4.9\%) & 0.8 (7.7\%) & 9.9  \\
16  & 5.7 (37.5\%) & 8.9 (58.1\%)  & -- & 0.7 (4.4\%) & 15.2 & 5.7 (49.7\%) & 4.6 (39.4\%)  & 0.5 (4.2\%) & 0.8 (6.7\%) & 11.6 \\
32  & 5.7 (24.3\%) & 17.0 (72.7\%) & -- & 0.7 (3.0\%) & 23.4 & 5.7 (37.0\%) & 8.5 (54.6\%)  & 0.5 (3.2\%) & 0.8 (5.2\%) & 15.5 \\
64  & \multicolumn{5}{c|}{capacity limited} & 5.8 (24.9\%) & 16.2 (69.1\%) & 0.5 (2.3\%) & 0.9 (3.7\%) & 23.4 \\
128 & \multicolumn{5}{c|}{capacity limited} & 5.9 (14.9\%) & 31.9 (81.1\%) & 0.6 (1.4\%) & 1.0 (2.6\%) & 39.3 \\
\bottomrule
\end{tabular}
}
\end{table}

\textbf{Representative Serving Throughput.} Table~\ref{tab:throughput_summary} reports throughput under 32 concurrent requests with 8192-token inputs and 1024-token outputs, giving the per-user and per-GPU throughput used to contextualize the speedup trends below.

\begin{table}[!htb]
\centering
\scriptsize
\renewcommand{\arraystretch}{0.90}
\caption{Representative end-to-end serving throughput under 32 concurrent requests (8192-token inputs, 1024-token outputs, $1\times$H100). U = tokens/s/user, G = tokens/s/GPU, and Acc.\ is the five-task mean from Table~\ref{tab:main_accuracy}.}
\label{tab:throughput_summary}
\setlength{\tabcolsep}{6pt}
\begin{tabular}{l c c c c c c}
\toprule
 & \multicolumn{3}{c}{\textbf{Qwen3-4B-Thinking-2507}}
 & \multicolumn{3}{c}{\textbf{Qwen3-8B}}\\
\cmidrule(lr){2-4}\cmidrule(lr){5-7}

\textbf{Method}
& U & G & Acc.
& U & G & Acc.\\

\midrule

BF16
& 41.1 & 1187.9 & 75.64
& 35.8 & 999.9  & 70.84\\

Naive INT2
& 66.4 & 1797.4 & 0.00
& 54.2 & 1415.1 & 0.00\\

QuaRot-INT2~\citep{ashkboos2024quarot}
& 66.3 & 1795.8 & 1.40
& 54.1 & 1411.9 & 10.14 \\

Saw-INT4~\citep{jia2026saw}
& 50.6 & 1444.1 & 73.11
& 43.0 & 1183.2 & 69.97\\

\textbf{OSCAR (ours)}
& 63.3 & 1723.9 & 71.86
& 52.5 & 1374.0 & 69.63 \\

\bottomrule
\end{tabular}
\end{table}

\textbf{Pure Decoding Speed.} We first study how effective our kernel is in pure decoding speed. We isolated the prefill time or time to first token (TTFT) by setting a full prefix-cache-hit setting with a batch size $1$, where all methods fully reuse the cached KV states throughout decoding. In this regime, throughput differences directly reflect decode efficiency rather than cache residency or scheduling effects. 
Figure~\ref{fig:length_sweep_throughput} (Left)  shows decode throughput normalized to BF16 across three models under varying context lengths. It shows that \textbf{(i)} OSCAR consistently outperforms Saw-INT4~\cite{jia2026saw} by up to 2x when the sequence length is 100k, reaching the theoretical bound, verifying the efficiency of our kernel design; \textbf{(ii)} The speedup of OSCAR increases with context length, from $1.98\times$ at 30k, $2.52\times$ at 60k,  to $3.08\times$ at 100k on Qwen3-4B. This trend scales up to 358B. As such, decoding speed becomes increasingly KV-bandwidth-bound at long contexts. By reducing KV-cache memory by a factor of $8$ (BF16 to INT2), OSCAR achieves substantially larger gains as the context length grows, while the online rotation overhead remains effectively hidden within the decode kernels. 


\begin{figure}[!htb]
\centering
\includegraphics[width=0.78\linewidth]{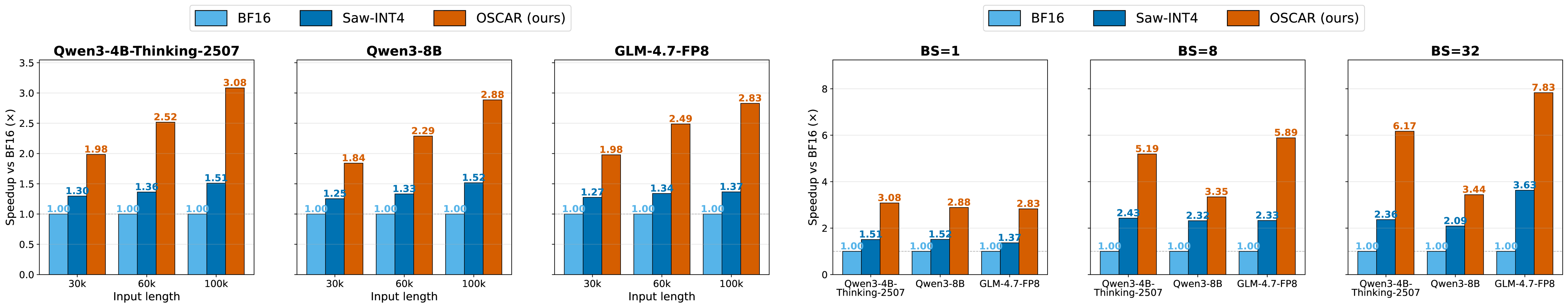}
\caption{
Left: Decode throughput speedup relative to BF16 at batch size $1$ with a full prefix-cache hit and 1K output sequence length. OSCAR achieves progressively larger speedups as the context increases. Right: Job-level throughput at $100\mathrm{k}$ sequence, as a function of batch size for Qwen3-4B/8B and GLM-4.7-FP8. Each group of bars compares BF16, Saw-INT4, and ours at $\mathrm{BS}\in\{1,8,32\}$. 
}
\label{fig:length_sweep_throughput}
\end{figure}

\textbf{Effect of Batch Size.} We next study how serving throughput scales with batch size under long-context decoding. 
Figure~\ref{fig:length_sweep_throughput} (right) reports job-level throughput with prefix warm up at a fixed input length ($100\mathrm{k}$ tokens) with batch sizes $\mathrm{BS}\in\{1,8,32\}$.
Prefix warm up is used to simulate popular agentic serving workload with long prefix reuse. 
Three observations stand out. 
\textbf{(i)} OSCAR consistently outperforms BF16 across all batch sizes. On GLM-4.7-FP8, OSCAR achieves a $2.83\times$ speedup over BF16 at $\mathrm{BS}{=}1$, which further increases to $7.83\times$ at $\mathrm{BS}{=}32$. This behavior highlights that the reduced KV-cache footprint of INT2 methods becomes increasingly beneficial under highly concurrent long-context serving workloads. 
\textbf{(ii)} OSCAR consistently outperforms Saw-INT4 across all models and batch sizes in Figure~\ref{fig:length_sweep_throughput}, showing that the fused online rotation does not introduce an observable throughput penalty.
\textbf{(iii)} Throughput improves substantially with larger batch sizes as OSCAR can accommodate more concurrent requests with less memory pressure. OSCAR's design also benefits prefix caching~\citep{sglang}, improving cache hit rates and job-level cache reuse.

\textbf{Serving Throughput with Prefix Cache from 0 to 100\%.} Figure~\ref{fig:prefix_cache_throughput} evaluates end-to-end serving throughput under concurrent multi-user workloads by sweeping prefix-cache hit ratio from cache disabled to near-100\% cache hits. 

\begin{figure}[!htb]
\centering
\includegraphics[width=0.78\linewidth]{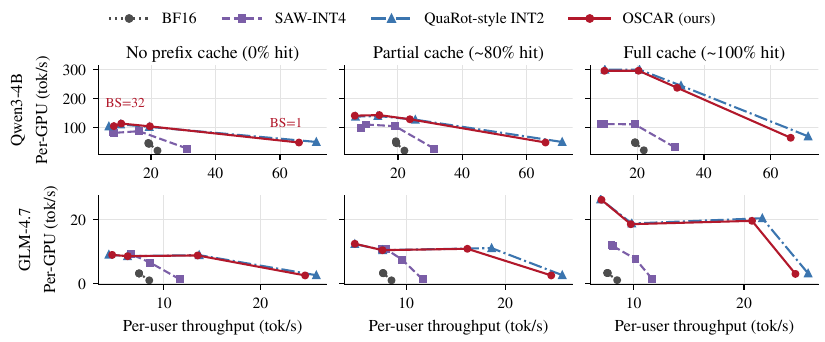}
\caption{Effect of prefix-cache hit ratio on end-to-end serving throughput (100k ISL, 1K OSL). Each subplot plots median per-user throughput $U$ (tok/s) against mean per-GPU throughput $G$ (tok/s) for $\mathrm{BS} \in \{1,8,16,32\}$ on Qwen3-4B-Thinking-2507 and GLM-4.7-FP8 across 0\%, $\sim$80\%, and $\sim$100\% prefix-cache hit regimes.}
\label{fig:prefix_cache_throughput}
\end{figure}
  
The results show that \textbf{(i)} increasing the prefix-cache hit ratio consistently expands the throughput frontier by reducing prefill recomputation, with the largest gains appearing at higher batch sizes; \textbf{(ii)} OSCAR remains on or near the efficiency frontier across all cache regimes, closely matching aggressive INT2 baselines while avoiding extra decode-time memory traffic or indirection; \textbf{(iii)} compared with BF16, low-bit KV quantization enables substantially higher per-GPU throughput because it improves decode efficiency and reduces KV-cache memory footprint, allowing larger concurrent batches under the same memory budget; and \textbf{(iv)} OSCAR achieves the best accuracy among INT2 methods while preserving the standard paged KV-cache abstraction and integrating cleanly into fused SGLang decode kernels, whereas QuaRot-INT2 can approach the throughput ceiling but suffers significant accuracy degradation as shown in Table~\ref{tab:main_accuracy}.

\begin{figure}[!htb]
  \centering
  \includegraphics[width=0.78\linewidth]{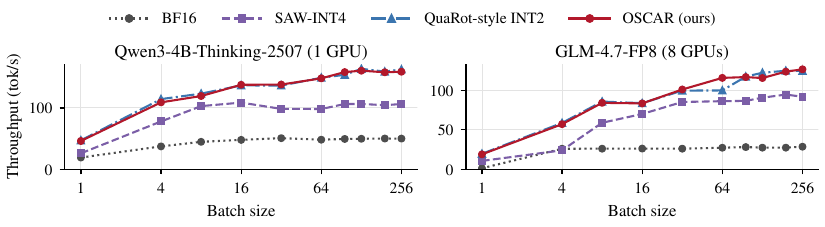}
  \caption{Long-context serving stress test with 100k-token inputs. OSCAR's uniform INT2 KV-cache scales up to $2^8$ requests with continued throughput gains.}
  \label{fig:scal-batch-size}
\end{figure}

\textbf{Memory and Extreme Batch Scalability.}
Finally, Figure~\ref{fig:scal-batch-size} evaluates memory usage and batch scalability under 100k-token inputs up to $2^8$ concurrent requests. By reducing the KV-cache footprint by nearly $8\times$ relative to BF16 while keeping uniform INT2 pages, OSCAR scales to $2^8$ requests on a single H100 (80\,GB) with continued throughput gains, whereas BF16 and INT4 baselines run out of memory or plateau early. Detailed ablations and discussion are in Appendices~\ref{app:ablation_detail_runs} and~\ref{sec:discussion}.

\section{Conclusion}
\label{sec:conclusion}

We introduced \textbf{OSCAR}, an INT2 KV-cache quantization method that chooses rotations from attention-aware target covariance rather than raw cache reconstruction. With one offline calibration pass, OSCAR fixes key/value rotations and clip thresholds, then uses a fused SGLang path to quantize history KV tokens while preserving sink and recent tokens. This simple target-and-kernel co-design recovers much of the BF16 accuracy lost by naive INT2 while keeping the memory and serving advantages of 2-bit KV caches.

\begin{ack}
We thank the anonymous reviewers and the area chair for their detailed feedback, which led to the matched-baseline comparisons, the residual diagnostics, and the long-context analysis added in this version.
This work was supported by Together AI, which also provided the compute used for all experiments.
Zhongzhu Zhou, Donglin Zhuang, Jisen Li, Shuaiwen Leon Song, Ben Athiwaratkun, and Xiaoxia Wu are affiliated with Together AI, which develops and operates LLM serving infrastructure; the authors declare no other competing interests.
\end{ack}

\bibliography{references}
\bibliographystyle{unsrtnat} 
\clearpage
\appendix

\clearpage

\section{Additional Details and Theoretical Analysis}
\label{app:theory}

\subsection{Hadamard Transform}
\label{app:hadamard_transform}

The Hadamard transform is an orthogonal mixing transform with entries of equal magnitude. We use the normalized Walsh-Hadamard matrix $H_{\mathrm{Had}}$, which satisfies $H_{\mathrm{Had}}^\top H_{\mathrm{Had}}=I$. Its dimension is inferred from the key or value vector being transformed. For power-of-two dimensions, it can be defined recursively by $H_1=[1]$ and
\[
H_{2m}
=
\frac{1}{\sqrt{2}}
\begin{bmatrix}
H_m & H_m\\
H_m & -H_m
\end{bmatrix}.
\]
In low-bit LLM quantization, Hadamard transforms are commonly used to spread outlier energy across channels before rounding~\citep{ashkboos2024quarot}. In OSCAR, $H_{\mathrm{Had}}$ is composed after the attention-aware covariance rotation to further smooth the rotated coordinates while keeping the transform fixed and efficient.

\subsection{Principal Component Analysis}
\label{app:covariance_pca}

Given a symmetric positive semidefinite matrix $A \in \mathbb{R}^{d\times d}$, principal component analysis (PCA) diagonalizes it via
\begin{equation*}
A = V \Lambda V^\top,
\qquad
\Lambda = \diag(\lambda_1,\dots,\lambda_r,\dots),\quad
\lambda_1 \ge \lambda_2 \ge \cdots \ge 0.
\end{equation*}

Below we recall the important proposition of PCA ~\cite{fan1951maximum}.

\begin{proposition}[Indication of top-r eigenvectors]
\label{prop:pca-top-r}
For any $1\le r\le d$, let
\[
V_r = [v_1,\dots,v_r]\in\mathbb{R}^{d\times r}.
\]
Then
\[
V_r = \argmax_{U^\top U = I_r} \tr(U^\top A U),
\]
and the optimal value is
\[
\max_{U^\top U = I_r} \tr(U^\top A U)=\sum_{i=1}^r \lambda_i.
\]
\end{proposition}

Note that in our setting, when $A$ is chosen as a target covariance induced by the attention computation, PCA provides an orthogonal basis aligned with the directions that are most important to preserve under KV compression.

\begin{proof}
\label{app:PCA proof}

Let $U\in\mathbb{R}^{d\times r}$ satisfy $U^\top U=I_r$. Since $A=V\Lambda V^\top$, we have
\[
\tr(U^\top A U)
=
\tr(U^\top V\Lambda V^\top U).
\]
Define
\[
B = V^\top U \in \mathbb{R}^{d\times r}.
\]
Because $V$ is orthogonal,
\[
B^\top B = U^\top V V^\top U = U^\top U = I_r.
\]
Hence the columns of $B$ are orthonormal. Using cyclicity of trace,
\[
\tr(U^\top A U)
=
\tr(B^\top \Lambda B)
=
\tr(\Lambda B B^\top).
\]
Write
\[
P = BB^\top \in \mathbb{R}^{d\times d}.
\]
Then $P$ is an orthogonal projection matrix of rank $r$, so
\[
0\le P_{ii}\le 1,\qquad \sum_{i=1}^d P_{ii} = \tr(P)=r.
\]
Therefore
\[
\tr(U^\top A U)
=
\tr(\Lambda P)
=
\sum_{i=1}^d \lambda_i P_{ii}.
\]
Set
\[
w_i = P_{ii},\qquad i=1,\dots,d.
\]
Then
\[
0\le w_i\le 1,\qquad \sum_{i=1}^d w_i=r,
\]
and thus
\[
\tr(U^\top A U)=\sum_{i=1}^d \lambda_i w_i.
\]
Since $\lambda_1\ge \lambda_2\ge \cdots \ge \lambda_d$, the quantity $\sum_{i=1}^d \lambda_i w_i$ is maximized by assigning full weight to the largest $r$ eigenvalues, namely
\[
w_1=\cdots=w_r=1,\qquad w_{r+1}=\cdots=w_d=0.
\]
Hence
\[
\tr(U^\top A U)\le \sum_{i=1}^r \lambda_i.
\]
Now choose $U=V_r$. Then
\[
V_r^\top A V_r
=
V_r^\top V\Lambda V^\top V_r
=
\diag(\lambda_1,\dots,\lambda_r),
\]
so
\[
\tr(V_r^\top A V_r)=\sum_{i=1}^r \lambda_i.
\]
Thus $V_r$ attains the maximum, and therefore
\[
V_r \in \argmax_{U^\top U=I_r} \tr(U^\top A U).
\]
\end{proof}

\subsection{Target Covariance}
\label{app:covariance}

Given a matrix $X=[x_1;\dots;x_N]\in\mathbb{R}^{N\times d}$ whose rows are feature vectors, its empirical covariance, assuming the rows have been centered, is
$$
C_X=\frac{1}{N}X^\top X=\frac{1}{N}\sum_{i=1}^N x_i^\top x_i\in\mathbb{R}^{d\times d}.
$$
In this paper, we use the uncentered second moment in the same algebraic form and refer to it as a covariance target. For any direction $u\in\mathbb{R}^{d}$ with $\|u\|_2=1$, the scalar
$$
u^\top C_Xu=\frac{1}{N}\sum_{i=1}^N (x_i u)^2
$$
measures the average squared magnitude of the rows of $X$ along direction $u$. Thus, the eigenvectors of $C_X$ identify directions where the row activations have large energy, and PCA uses these eigenvectors as a basis aligned with the dominant second-order structure of $X$.

In OSCAR, the covariance target is not chosen only from the raw cached tensors. The reason is that the rotation should reduce the downstream attention error, not merely the Euclidean reconstruction error of $K$ and $V$. For keys, the empirical logit distortion satisfies
$$
\|QK^\top-Q\widehat K^\top\|_F^2
=
\tr\!\left((K-\widehat K)Q^\top Q(K-\widehat K)^\top\right).
$$
This identity shows that key errors are weighted by the query covariance $Q^\top Q$. Heuristically, directions with large eigenvalues of $Q^\top Q$ are directions in which the queries have large energy, so quantization error in those directions causes larger perturbations to the attention logits. Therefore, we use $Q^\top Q$ as the key-side target covariance for determining the raw key rotation.

For values, the downstream output distortion satisfies
$$
\|SV-S\widehat V\|_F^2
=
\tr\!\left((V-\widehat V)^\top S^\top S(V-\widehat V)\right).
$$
Equivalently, the matrix
$$
V^\top S^\top S V=(SV)^\top(SV)
$$
is the target covariance of the attention-weighted value activations $SV$. It measures which value feature directions remain large after aggregation by the attention scores. Hence, $V^\top S^\top S V$ is an attention-induced value target covariance, while $V^\top V$ is only the raw-cache covariance of values.

Therefore, $Q^\top Q$ and $V^\top S^\top S V$ are covariance targets in our context: they are second-moment matrices induced by the downstream attention computation. This is why OSCAR estimates rotations from these attention-aware covariance targets rather than from raw-cache targets such as $K^\top K$ and $V^\top V$.

\subsection[Intuition of the OSCAR rotation factors]{Intuition of the combination $R_K = U_Q H_{\mathrm{Had}} P_K$}
\label{app:intuition_uhp}
\setcounter{MaxMatrixCols}{16}

The rotation applied to keys before quantization in OSCAR is the product of three orthogonal factors,
\[
  R_K \;=\; U_Q \, H_{\mathrm{Had}} \, P_K,
\]
where $U_Q$ is the eigenbasis of the key target covariance $C_Q = Q^\top Q$ (Sec.~\ref{app:covariance}), $H_{\mathrm{Had}}$ is the normalized Walsh--Hadamard matrix (Sec.~\ref{app:hadamard_transform}), and $P_K$ is a permutation matrix that we describe below. The same construction is used for $R_V$ with $C_S = V^\top S^\top S V$ in place of $C_Q$. This subsection explains, before proceeding to the formal proof in Sec.~\ref{app:ambient_basis_simple_variants}, why each factor is needed and why they appear in this order. Each factor addresses a distinct failure mode of per-group low-bit quantization.

\paragraph{Two effects of $R_K$.}
Recall from Sec.~\ref{app:covariance} that the rotated cache $\tilde K = K R_K$ is what the quantizer sees, and the resulting attention-logit distortion is
\[
  \bigl\| QK^\top - Q\widehat K^\top \bigr\|_F^{2}
  \;=\; \tr\!\left( R_K^\top C_Q R_K \cdot E_K \right),
\]
where $E_K$ is the residual covariance of $\tilde K - \mathcal Q(\tilde K)$. The rotation enters in two distinct places:
\begin{enumerate}[label=(\roman*),leftmargin=*]
  \item it shapes $R_K^\top C_Q R_K$, the query-importance metric in the rotated coordinates;
  \item it shapes $E_K$, because the quantizer's per-token, per-group min--max scale (Sec.~\ref{app:scale-determination}) depends on the channelwise distribution of $\tilde K$.
\end{enumerate}
The three factors $U_Q$, $H_{\mathrm{Had}}$, $P_K$ each address one of these effects in turn.

\paragraph{$U_Q$: align channels with query-importance directions.}
Eigendecompose
\[
  C_Q \;=\; U_Q \Lambda_Q U_Q^\top,
  \qquad \Lambda_Q = \diag(\lambda_1, \dots, \lambda_d),
  \qquad \lambda_1 \ge \cdots \ge \lambda_d \ge 0.
\]
Setting $R_K = U_Q$ diagonalizes the importance metric:
\[
  R_K^\top C_Q R_K \;=\; \Lambda_Q.
\]
Geometrically, the $j$-th channel of $\tilde K = K U_Q$ is the projection of each key onto the $j$-th query principal direction $u_j$, and $\lambda_j = \|Q u_j\|_2^{2}$ measures how much total query energy lies along $u_j$. Quantization noise on channel $j$ then propagates to the attention logits with weight $\lambda_j$: large-$\lambda$ channels are exactly those whose error is most amplified, while small-$\lambda$ channels can absorb error almost for free. Sec.~\ref{app:ambient_basis_simple_variants} formalizes this via a rearrangement-inequality argument on the surrogate $\tilde{\mathcal L}_K(R_k) = \tr(R_k^\top C_Q R_k E_K)$: under the frozen-residual assumption, $R_K = U_Q$ is the orthogonal rotation that minimizes the importance-weighted error.

This is the right first step but it leaves a per-group quantization problem. The eigenspectrum $\Lambda_Q$ of $C_Q$ in real LLM activations is sharply anisotropic: a few eigenvalues carry most of $\tr(C_Q)$. After $R_K = U_Q$, the rotated key channels inherit this imbalance, with a few channels having much larger variance than the rest. A per-group min--max quantizer with a single scale per group then sets its scale based on the largest channel in each group, and the smaller channels in the same group are quantized at a granularity far coarser than their dynamic range. Hence $E_K$ has a few very large diagonal entries and many tiny ones, and $\tr(E_K)$ is dominated by the worst groups.

\paragraph{$H_{\mathrm{Had}}$: equalize on the importance metric, suppress outliers on $K$.}
The $U_Q$ step is chosen for the surrogate, which is a Q-side criterion ($C_Q = Q^\top Q$). The Hadamard composition serves two complementary purposes: on the surrogate side, it equalizes the rotated importance metric across channels; on the actual K distribution, it mixes outlier energy across channels so that the per-group min--max quantizer (Sec.~\ref{app:scale-determination}) sees comparable per-channel magnitudes. The first claim is exact; the second is what the standard outlier-suppression literature relies on~\citep{ashkboos2024quarot} and only requires that the Hadamard column weights all have equal magnitude.

The exact statement on the importance metric is the following.
\begin{lemma}[Diagonal equalization under Hadamard conjugation]
\label{lem:hadamard-diag}
Let $\Lambda \in \mathbb{R}^{d \times d}$ be diagonal and $H_{\mathrm{Had}}$ the normalized Walsh--Hadamard matrix. Then
\[
  \bigl(H_{\mathrm{Had}}^\top \Lambda\, H_{\mathrm{Had}}\bigr)_{ii} \;=\; \frac{1}{d}\tr(\Lambda), \qquad i = 1, \dots, d.
\]
\end{lemma}
\begin{proof}
Every entry of $H_{\mathrm{Had}}$ has magnitude $1/\sqrt{d}$. Then $(H_{\mathrm{Had}}^\top \Lambda H_{\mathrm{Had}})_{ii} = \sum_j H_{\mathrm{Had}}[j, i]^2 \, \Lambda_{jj} = \frac{1}{d}\sum_j \Lambda_{jj} = \frac{1}{d}\tr(\Lambda)$, independent of $i$.
\end{proof}
Applying the lemma to $\Lambda = \Lambda_Q$ (which is diagonal because of the $U_Q$ step) gives
\[
  \bigl(R_K^\top C_Q R_K\bigr)_{ii} \;=\; \bigl(H_{\mathrm{Had}}^\top \Lambda_Q H_{\mathrm{Had}}\bigr)_{ii} \;=\; \frac{1}{d}\tr(C_Q), \qquad \forall\, i.
\]
The peaky eigenspectrum that $U_Q$ exposed has been replaced by a uniform diagonal: every channel now carries the same importance weight $\tr(C_Q)/d$ in the rotated frame.

The K-side effect is heuristic but practically important. After $R_K = U_Q H_{\mathrm{Had}}$, each row of $\tilde K = K R_K$ is a Walsh-signed sum of the principal-direction projections $K U_Q$, with every entry of weight $\pm 1/\sqrt{d}$. A per-token outlier on any single principal direction is thus redistributed in magnitude $1/\sqrt{d}$ across all $d$ channels, shrinking the per-token max by roughly the same factor and bringing the per-group min--max scale into a much tighter range. The frozen residual $E_K$ becomes much smaller in trace and far more uniform across $i$. The price paid is that the importance metric is no longer diagonal: $R_K^\top C_Q R_K = H_{\mathrm{Had}}^\top \Lambda_Q H_{\mathrm{Had}}$ has off-diagonals that couple channels. In the regime where per-group quantization noise dominates $\tr(E_K)$, the reduction in $\tr(E_K)$ wins over the loss of importance alignment.

\paragraph{Empirical aside: $U_Q$ does \emph{not} accidentally diagonalize $\Sigma_K$.}
One might hope that, because $W_Q$ and $W_K$ are trained on the same activation distribution, $C_Q = Q^\top Q$ and $\Sigma_K = K^\top K$ would share an eigenbasis, so that $U_Q$ would simultaneously serve both the Q-side (importance) and K-side (variance) objectives. Empirically, this is not the case. On Qwen3-8B with $2000$ calibration tokens per layer, the top-8 eigenvector self-alignment $\sum_{k=1}^{8} |U_Q^{Q\top} U_Q^{K}|/8$ ranges from $0.05$ to $0.15$ across the 36 layers --- essentially indistinguishable from the random alignment $1/\sqrt{d} \approx 0.09$ expected at $d = 128$. Rotating $\Sigma_K$ by $U_Q$ typically \emph{reduces} the fraction of $\Sigma_K$'s energy concentrated on the diagonal (e.g., layer 1: 0.90 raw vs.\ 0.09 after $U_Q$ rotation; layer 16: 0.16 vs.\ 0.07). Figure~\ref{fig:qk-alignment-layer16} shows a representative layer.

\begin{figure}[h]
  \centering
  \includegraphics[width=\linewidth]{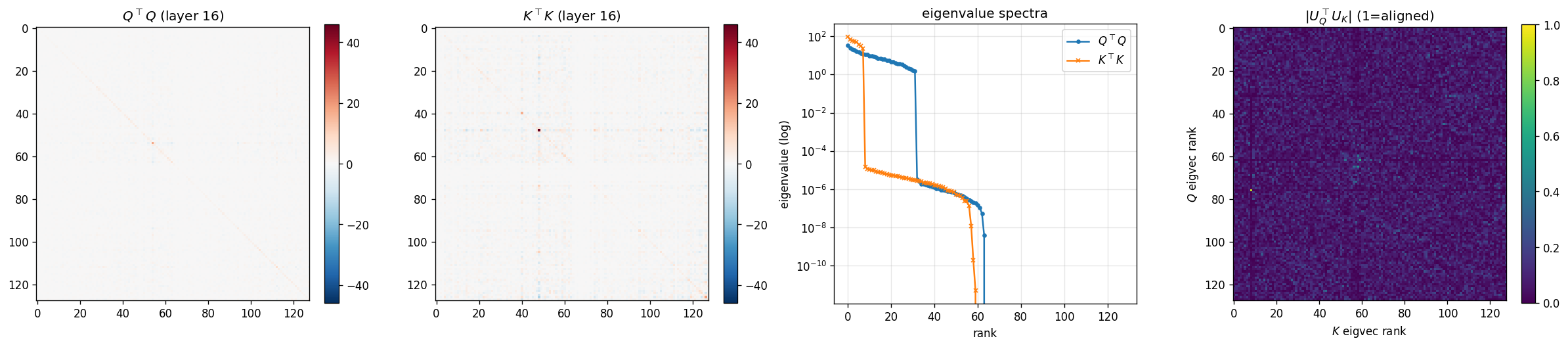}
  \caption{Layer 16 of Qwen3-8B. From left: heatmap of $Q^\top Q$; heatmap of $K^\top K$ (same colour scale); eigenvalue spectra; absolute eigenvector alignment $|U_Q^\top U_Q|$ with both bases sorted by descending eigenvalue. A diagonal-bright alignment matrix would indicate a shared eigenbasis; the observed uniform-texture pattern shows the eigenbases are nearly uncorrelated.}
  \label{fig:qk-alignment-layer16}
\end{figure}

The construction is nevertheless robust because the two factors target different things. Lemma~\ref{lem:hadamard-diag} guarantees exact diagonal equalization of the importance metric $C_Q$ regardless of the eigenbasis of $\Sigma_K$, since the lemma only relies on $\Lambda_Q$ being diagonal in the post-$U_Q$ frame. The K-side outlier-suppression effect of $H_{\mathrm{Had}}$ described above also does not require alignment between $U_Q$ and the eigenbasis of $\Sigma_K$: the Walsh-signed mixing shrinks per-token outliers by a factor of $1/\sqrt{d}$ no matter which orthogonal basis $K$ was placed in beforehand. The Q-side ($U_Q$) and K-side ($H_{\mathrm{Had}}$) factors thus address genuinely independent objectives and compose without interference, rather than relying on an approximate $W_Q \approx W_K$ similarity that does not hold in practice.

\paragraph{$P_K$: balance importance across groups.}
Hadamard mixing equalizes \emph{marginal} channel variance, but the per-group quantizer in Sec.~\ref{app:scale-determination} acts on contiguous blocks of $G_K$ channels at a time. What matters at the group level, beyond marginal variance, is whether each group sees a balanced sample of the underlying importance directions: a group whose $G_K$ channels are dominated by the same handful of large-$\lambda_j$ eigenvectors is statistically harder to quantize than a group whose channels span the full range of $\lambda_j$.

OSCAR uses a \emph{permuted bit-reversal} (PBR) ordering. Let $\sigma$ be the descending-eigenvalue permutation, so that $\lambda_{\sigma(0)} \ge \lambda_{\sigma(1)} \ge \cdots$, and let $\beta$ be the bit-reversal permutation on $\{0, 1, \dots, d-1\}$. The permutation matrix $P_K$ places the eigenvector with the $k$-th largest eigenvalue at position $\beta(k)$, i.e.\ $(P_K)_{:,\beta(k)} = e_{\sigma(k)}$. For $d = 128$, this yields the placement
\[
  \text{top-1} \mapsto 0,\quad
  \text{top-2} \mapsto 64,\quad
  \text{top-3} \mapsto 32,\quad
  \text{top-4} \mapsto 96,\quad
  \text{top-5..8} \mapsto 16, 80, 48, 112,\;\;\dots
\]
The defining property of bit-reversal is that for any power-of-two group size $G \mid d$, the top-$d/G$ eigenvectors land in $d/G$ \emph{distinct} groups, exactly one per group; the same recursive balance holds at every coarser binary level. Composed with the Hadamard, this means each per-group quantization block sees one representative from each level of the importance hierarchy, regardless of which group size $G_K$ is chosen at deployment.

\paragraph{Empirical aside: errors separate after attention.}
Figure~\ref{fig:app_error_propagation} compares layer-wise errors on Qwen3-4B-Thinking-2507, measured on raw KV tensors and on the downstream quantities consumed by attention.
OSCAR is not designed to minimize plain Euclidean reconstruction error of $K$ or $V$, so its raw $K$-MSE and $V$-MSE are not always dramatically smaller than rotation-only baselines.
The gap becomes much clearer after the attention computation: the errors in $QK^\top$ and $SV$ are substantially lower, and the advantage further propagates to attention-output errors.
This is exactly what the covariance targets optimize: the key rotation is aligned with query-side covariance, while the value rotation is aligned with the score-weighted value covariance.
The trend is consistent with the $U_K$ intuition above: aligning channels with query-importance directions need not minimize raw $K$ error, but it directly lowers the error after multiplication by $Q$; analogously, the value target is chosen so that the advantage appears after aggregation by $S$.

\begin{figure}[ht]
  \centering
  \captionsetup[subfigure]{justification=centering,singlelinecheck=true}
  \begin{subfigure}[t]{0.47\linewidth}
    \centering
    \includegraphics[width=\linewidth]{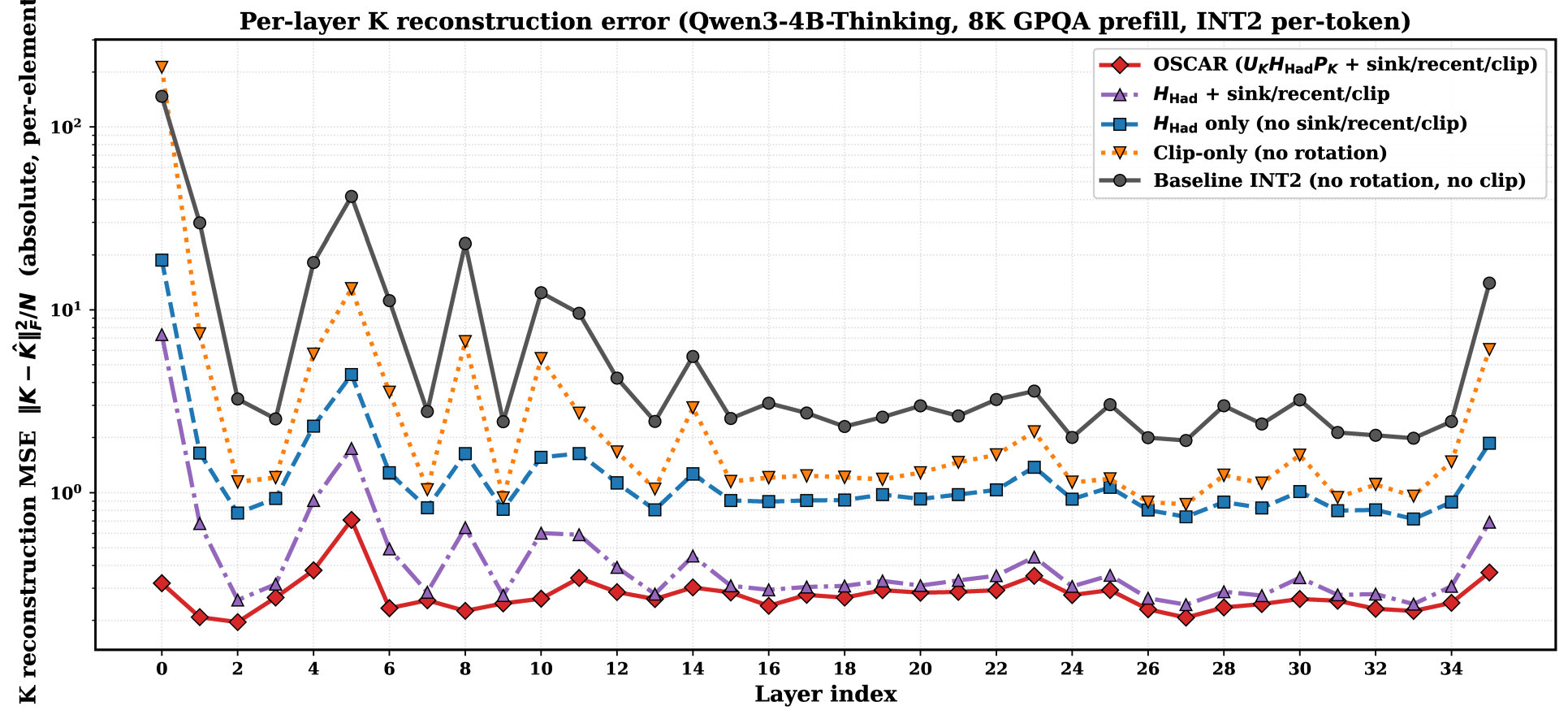}
    \caption{$K$ MSE}
  \end{subfigure}
  \hfill
  \begin{subfigure}[t]{0.47\linewidth}
    \centering
    \includegraphics[width=\linewidth]{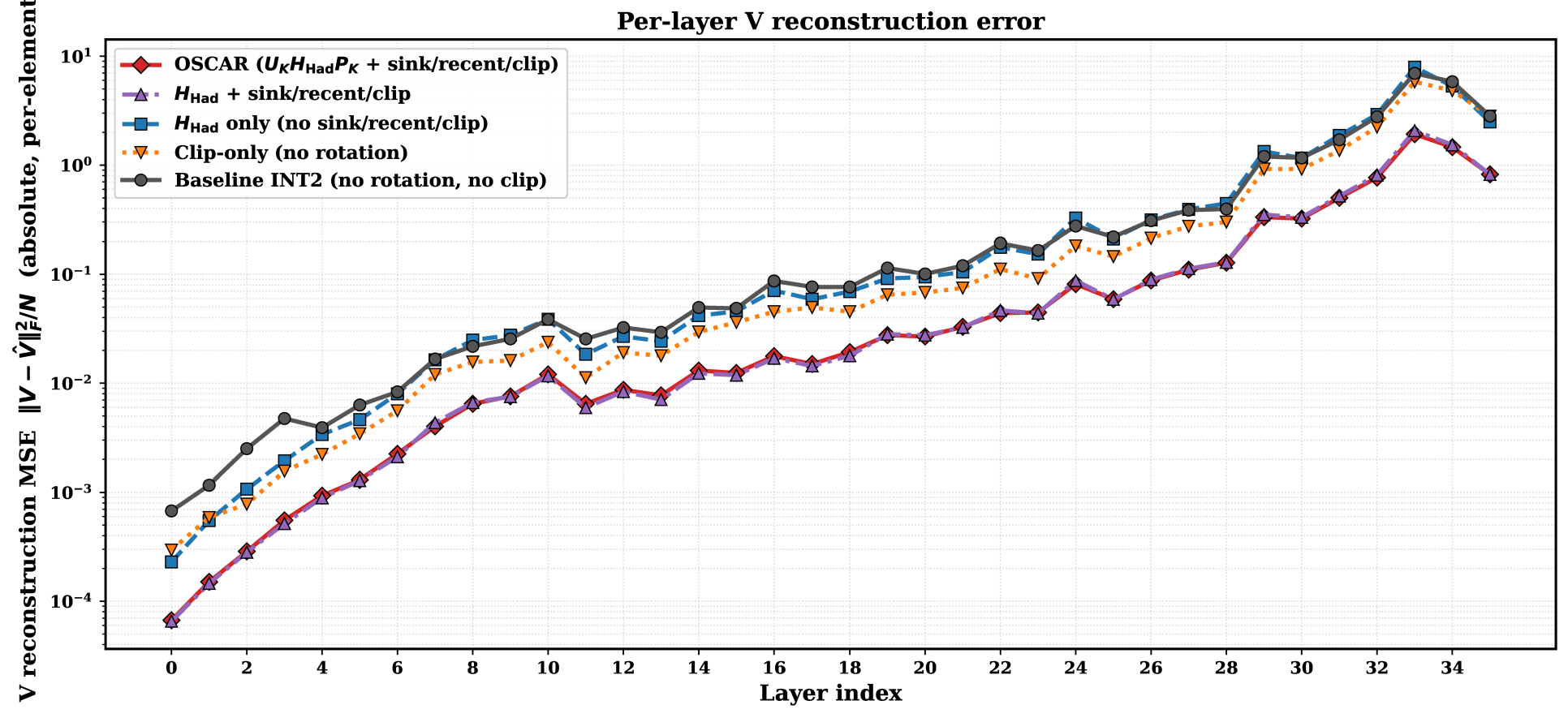}
    \caption{$V$ MSE}
  \end{subfigure}

  \vspace{0.4em}
  \begin{subfigure}[t]{0.47\linewidth}
    \centering
    \includegraphics[width=\linewidth]{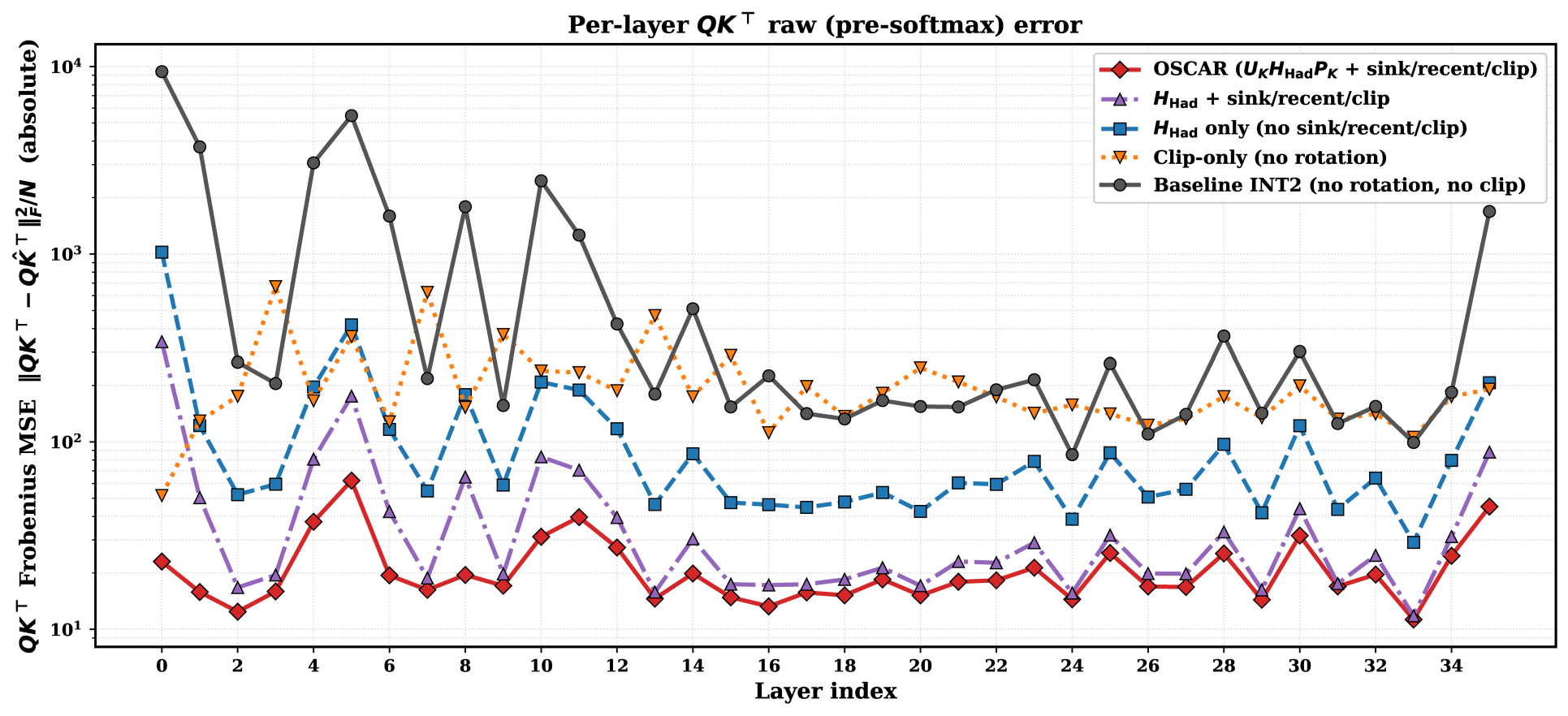}
    \caption{$QK^\top$ MSE}
  \end{subfigure}
  \hfill
  \begin{subfigure}[t]{0.47\linewidth}
    \centering
    \includegraphics[width=\linewidth]{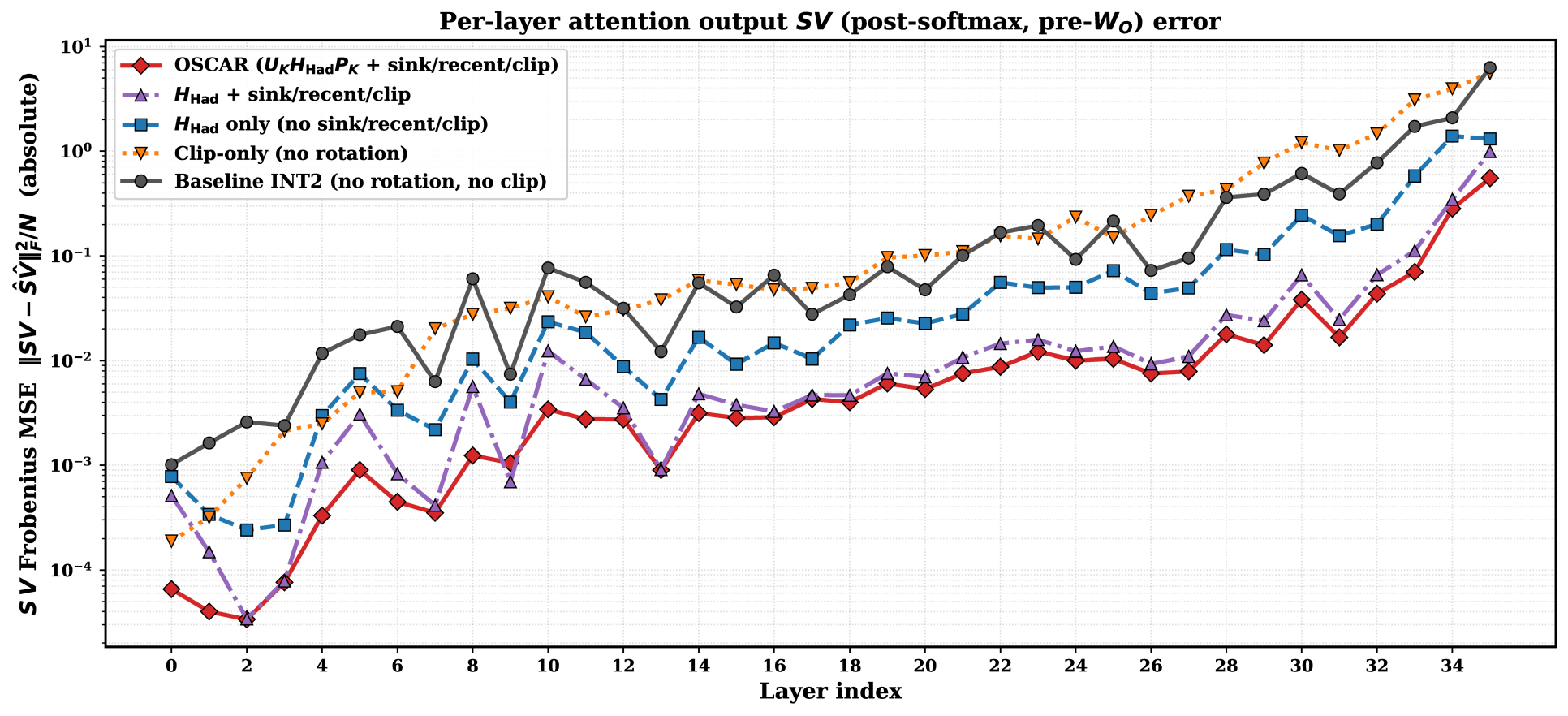}
    \caption{$SV$ MSE}
  \end{subfigure}

  \vspace{0.4em}
  \begin{subfigure}[t]{0.47\linewidth}
    \centering
    \includegraphics[width=\linewidth]{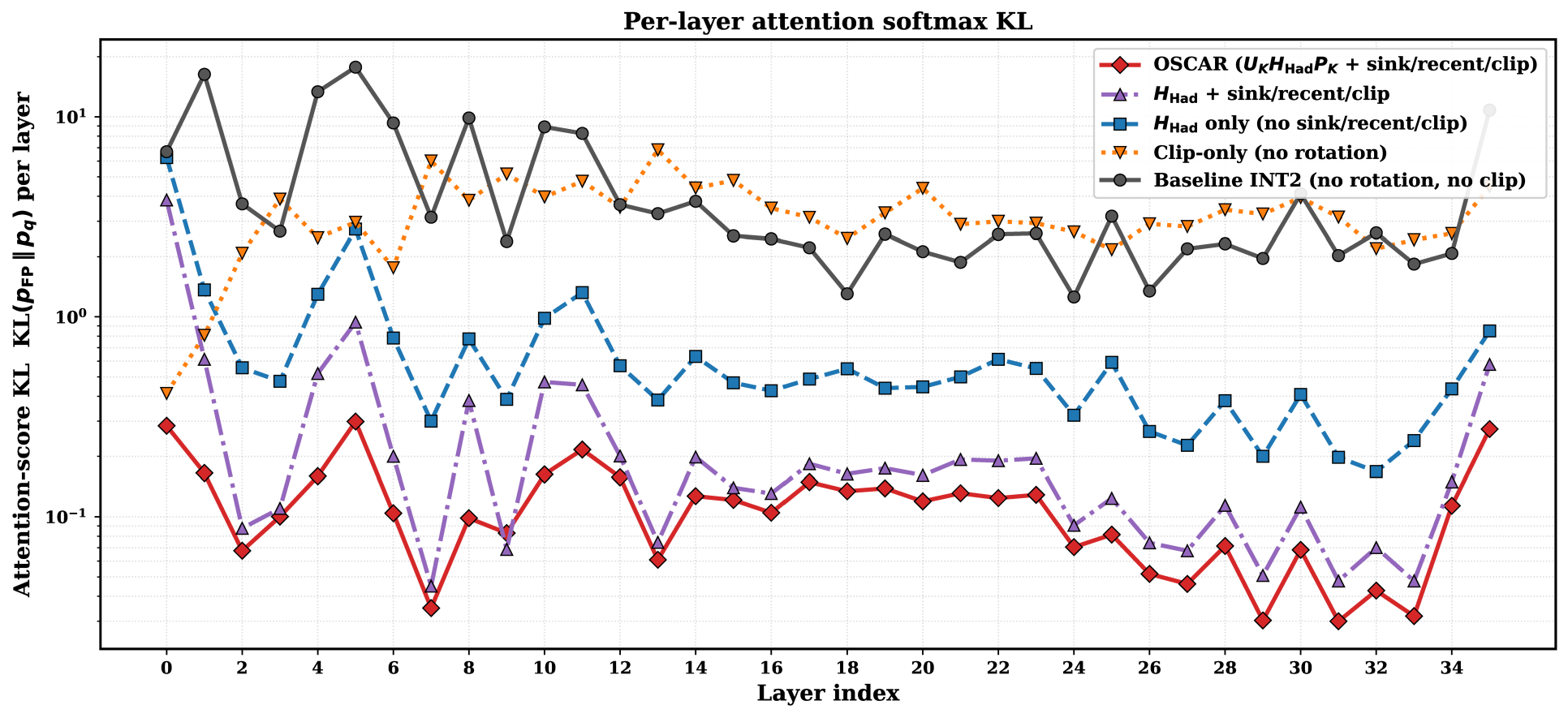}
    \caption{Attention KL}
  \end{subfigure}
  \hfill
  \begin{subfigure}[t]{0.47\linewidth}
    \centering
    \includegraphics[width=\linewidth]{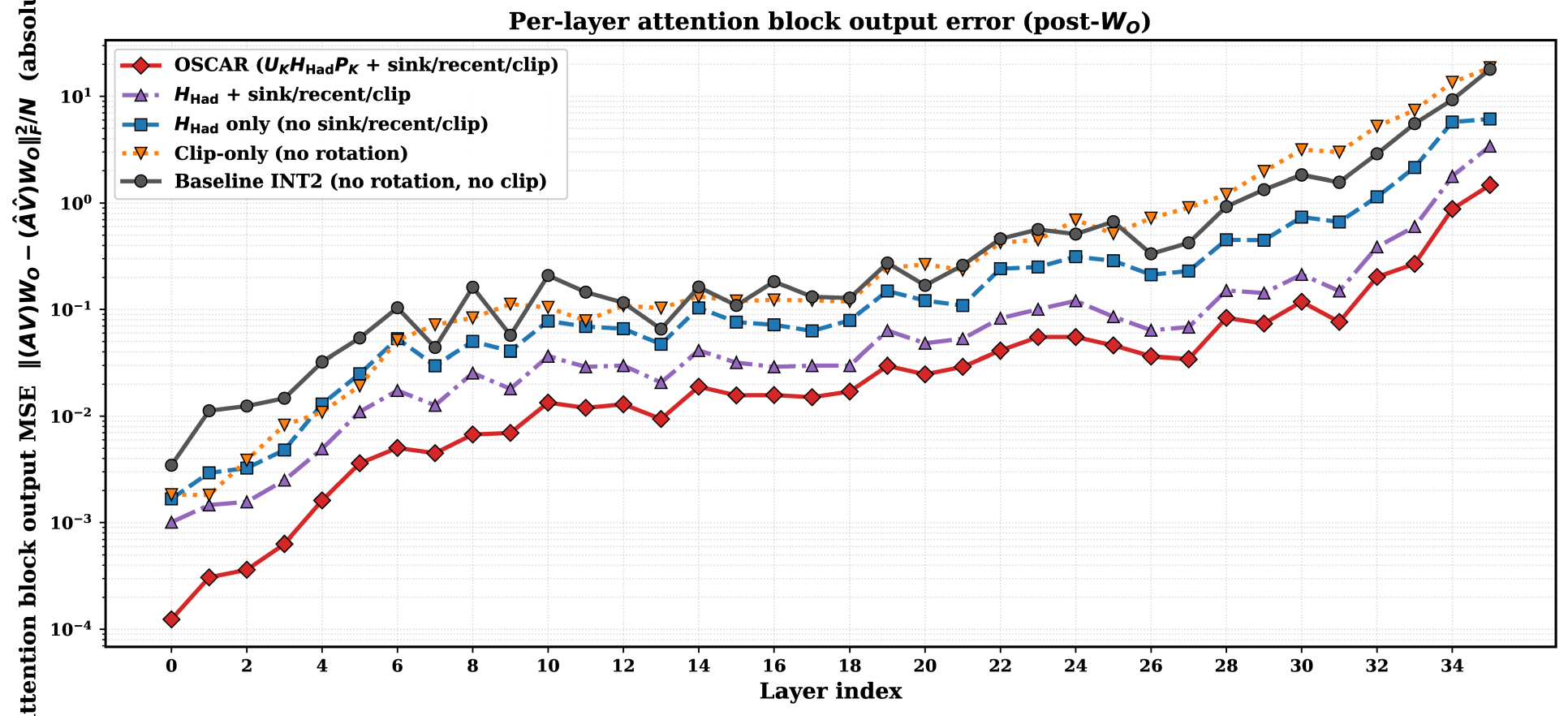}
    \caption{Attention-output MSE}
  \end{subfigure}
  \caption{\textbf{OSCAR's advantage appears on attention-consumed quantities on Qwen3-4B-Thinking-2507.}
  Raw $K$/$V$ reconstruction MSE alone does not fully explain downstream quality, because OSCAR does not optimize raw-cache reconstruction.
  Instead, it uses query- and score-aware covariance targets, which more directly reduce $QK^\top$, $SV$, and attention-output errors.}
  \label{fig:app_error_propagation}
\end{figure}

\paragraph{Worked example: Qwen3-4B-Thinking, layer 10, $d = 128$.}
We illustrate the three factors on real activations rather than a stylized spectrum. Calibration sample: layer 10 of Qwen3-4B-Thinking-2507 (32 query heads, 8 key/value heads, GQA ratio $= 4$), $T = 8000$ tokens of GPQA-diamond context. We use the production importance metric --- $Q$-side query covariance, averaged within each GQA group of $4$ query heads and then across all $8$ kv-heads,
\[
  C_Q \;=\; \frac{1}{H_{\mathrm{kv}}} \sum_{h=1}^{H_{\mathrm{kv}}} \frac{1}{T \cdot g} \sum_{i \in \mathcal G_h} Q_i^\top Q_i,
\]
The decode-time per-group min--max quantizer uses INT2 with $G_K = 64$ (two groups per head); we deliberately pick $G_K < d$ here so the PBR step has a non-trivial effect (see Step~3 for the $G_K = d$ degenerate case). The $K$ activations we display are from kv-head 0; the shared rotation $R_K = U_Q H_{\mathrm{Had}} P_K$ is the one produced by \texttt{compute\_qqt}.

\textit{Spectrum at layer 10.} Eigendecomposing $C_Q = U_Q \Lambda_Q U_Q^\top$ yields an anisotropic but not pathologically peaky spectrum (sorted descending),
\[
  \begin{aligned}
  \lambda_{1\dots 8} &\approx (162.5,\;85.5,\;35.6,\;17.4,\;11.6,\;7.5,\;6.1,\;5.0),\\
  \lambda_{121\dots 128} &\approx (0.24,\;0.22,\;0.22,\;0.20,\;0.18,\;0.05,\;0.027,\;0.009).
  \end{aligned}
\]
The trace is $\tr(C_Q) \approx 443$. The top eigendirection $\lambda_1$ alone accounts for $36.6\%$ of the trace and the top eight carry $74.7\%$: $\lambda_1 / \bar\lambda \approx 46.9$ where $\bar\lambda = \tr(C_Q)/d = 3.46$. Importance is concentrated in a low-rank subspace, but not as extremely as one might guess from $K$'s own variance --- instead, covariance $Q^\top Q$ reflects what the attention layer \emph{actually queries}, which is a softer concentration than where $K$'s variance happens to land.

\textit{One token, all 128 channels, split into two $G_K = 64$ groups.} The INT2 per-group min--max quantizer fits its scale to $\bigl(\max_{i \in g} \tilde K_{t,i} - \min_{i \in g} \tilde K_{t,i}\bigr) / 3$ per token $t$ per group $g$. We visualize what the quantizer sees by printing all $128$ channels for a single typical non-sink token ($t = 5$, kv-head 0), reshaped into an $8 \times 16$ grid (channel $j$ at grid position $(\lfloor j/16 \rfloor,\; j \bmod 16)$). The horizontal rule splits the grid into the two $G_K = 64$ groups: rows $0$--$3$ are group $0$ (channels $0$--$63$), rows $4$--$7$ are group $1$ (channels $64$--$127$). Each matrix is annotated with the per-group max--min.

The raw token in the channel basis:
\begin{equation*}
\resizebox{\textwidth}{!}{$
K_{t,\cdot} \;=\;
\left[
\begin{array}{rrrrrrrrrrrrrrrr}
 3.42 & -0.61 &  1.59 &  0.27 & -3.14 & -0.05 &  3.55 &  1.09 &  2.05 & -1.05 &  1.86 & -0.03 & -0.42 & -0.89 & -1.54 &  2.06 \\
-1.74 & -0.15 & -0.39 & -0.89 & -0.02 &  2.55 &  0.11 &  1.79 & -0.28 &  0.25 & -1.40 & -6.09 & -1.04 & -0.12 &  0.64 & -0.95 \\
 9.06 &  1.17 &  0.08 & -1.49 &  0.29 &  1.13 &  0.47 &  0.17 &  1.22 &  0.55 & 14.19 & -0.26 & -0.09 &  0.20 &  2.39 &  0.40 \\
 1.09 & -0.06 & \mathbf{-30.62} &  0.00 &  1.03 &  0.57 & -1.30 &  2.47 & -0.11 & -0.08 &  0.56 & -0.47 &  0.45 & -0.88 &  2.31 & -1.05 \\
\hline
 0.50 &  0.44 &  0.10 & -0.35 & -0.78 &  1.46 & -1.16 &  2.56 &  1.64 & -0.09 &  0.90 &  2.33 &  1.09 &  2.72 &  0.95 &  0.20 \\
 1.30 &  1.82 &  0.54 &  1.27 & -1.05 &  0.50 & -4.47 &  1.20 &  0.32 &  1.36 & -0.81 & -0.04 &  2.27 & -0.77 &  0.64 &  1.27 \\
 0.03 & -0.20 &  0.41 &  0.58 &  0.79 &  1.05 &  0.46 & -0.09 &  0.49 &  0.29 & -0.02 &  0.20 &  1.36 &  0.77 & -1.01 & -2.33 \\
 1.19 & -1.77 & -0.01 & -0.49 &  0.47 &  0.91 &  0.42 &  0.17 &  1.38 & -0.39 & -0.00 &  0.43 & -0.98 & -0.50 &  0.17 & -0.50
\end{array}
\right]
$}
\end{equation*}
\centerline{group $0$: $\max - \min = \boldsymbol{44.81}$ \quad|\quad group $1$: $\max - \min = \boldsymbol{7.19}$.}

The two groups are wildly unbalanced: group~$0$ holds both outliers ($-30.62$ at $j = 50$ and $14.19$ at $j = 42$); group~$1$ is benign. Group $0$'s INT2 scale must fit a range of $44.81$, wasting two bits on a few extreme values while the other $\sim 60$ channels of that group quantize at granularity an order of magnitude coarser than their actual dynamic range.

\textit{Step 1: $\tilde K = K U_Q$ (eigenbasis rotation).} Projecting $K$ onto the $Q^\top Q$ eigendirections gives:
\begin{equation*}
\resizebox{\textwidth}{!}{$
(K U_Q)_{t,\cdot} \;=\;
\left[
\begin{array}{rrrrrrrrrrrrrrrr}
 0.16 & -1.75 &  1.86 & -2.42 &  1.11 &  1.05 & -0.69 &  0.26 &  1.11 & -0.55 &  3.39 & -1.80 &  1.78 &  1.65 & -1.49 & -1.32 \\
 0.05 & -0.18 & -2.00 & -1.27 &  3.20 & -1.47 &  0.31 & -1.07 &  0.87 & -1.54 & -3.78 & -0.89 &  1.60 &  1.03 &  0.50 &  0.55 \\
 0.49 &  0.55 & -0.90 & -0.05 &  0.17 &  0.20 &  0.78 & -1.21 & -1.06 & -1.70 &  0.63 &  0.12 & -0.28 &  2.99 &  2.88 & -0.97 \\
-1.42 &  2.12 & -0.94 & -2.88 &  1.30 & -0.42 &  1.37 &  0.58 & -1.46 & -2.04 &  0.28 & -0.28 & -2.63 & -0.20 & -1.77 & -0.83 \\
\hline
 0.50 &  0.12 & -1.19 &  0.81 & -1.90 &  0.40 & -1.31 & -0.15 & -2.31 &  1.52 &  3.20 & -0.46 & -0.77 & -0.09 & -1.34 &  1.72 \\
 0.40 & -0.72 & -0.57 &  0.40 &  0.84 & -0.72 & -1.94 & -0.11 &  1.26 &  0.07 &  0.09 & -0.06 &  0.87 & -0.92 & -0.63 &  1.19 \\
-2.31 &  1.56 &  1.90 &  1.26 &  1.16 & -0.72 & -0.56 & -0.48 &  1.21 &  0.20 & -0.26 & -1.78 &  1.71 & -1.28 &  1.07 &  1.27 \\
-3.32 &  0.47 & -0.43 & -1.77 & -1.42 &  1.42 &  4.99 & -1.91 & -0.25 & -1.02 &  0.51 & -1.65 & -1.11 & -8.07 & 15.28 & \mathbf{29.77}
\end{array}
\right]
$}
\end{equation*}
\centerline{group $0$: $\max - \min = \boldsymbol{7.16}$ \quad|\quad group $1$: $\max - \min = \boldsymbol{37.85}$.}

The raw outlier hasn't been destroyed; it has been \emph{redistributed} into the $Q^\top Q$ eigenbasis. Because $U_Q$ is built from $Q^\top Q$, not $K^\top K$, the K-spike at $j=50$ does not project onto a single $U_Q$ direction --- it spreads across the lower-eigenvalue tail with its largest projection landing in the bottom-rank eigenchannels (the matrix's bottom-right entries: $29.77, 15.28, -8.07, \ldots$). The diagonal of $U_Q^\top C_Q U_Q$ is exactly $\Lambda_Q$, so the per-channel \emph{importance} is now exactly the eigenvalue, but the K-variance has been pushed into low-importance channels with no improvement in per-group dynamic range --- in fact the imbalance between groups has flipped and grown: group~$1$'s range jumps from $7.19$ (raw) to $37.85$, while group~$0$'s drops from $44.81$ to $7.16$. Per-group INT2 is now harder, not easier.

\textit{Step 2: $\tilde K = K U_Q H_{\mathrm{Had}}$ (Walsh-Hadamard mixing).} Multiplying by the orthonormal Walsh-Hadamard $H_{\mathrm{Had}} \in \mathbb R^{128 \times 128}$ rewrites each output channel as a $\pm 1$-weighted sum of all $128$ eigenchannels (each scaled by $1/\sqrt{128}$). The single-channel spike of $29.77$ is distributed across all $128$ outputs with mixed signs:
\begin{equation*}
\resizebox{\textwidth}{!}{$
(K U_Q H_{\mathrm{Had}})_{t,\cdot} \;=\;
\left[
\begin{array}{rrrrrrrrrrrrrrrr}
 2.24 &  1.08 & -3.73 &  0.26 & -5.76 &  0.92 &  4.08 & -3.39 & -3.72 &  0.29 &  5.55 & -2.67 &  3.19 & -3.70 & -4.16 &  1.80 \\
-0.87 &  0.64 &  4.12 & -3.18 &  5.12 &  1.16 & -6.79 &  0.97 &  1.91 & -2.08 & -4.91 &  4.21 & -1.67 &  1.97 &  4.47 & -0.63 \\
-2.86 &  0.66 &  6.73 &  1.83 &  4.37 &  1.49 & -5.77 &  0.72 &  1.52 & -0.05 & -4.94 &  3.97 & -3.09 &  2.28 &  4.34 & -0.37 \\
 1.90 & -1.30 & -6.06 & -1.90 & -2.58 &  3.27 &  4.18 & -3.10 & -2.31 &  0.88 &  5.04 & -2.61 & -0.07 & -1.10 & -2.17 &  4.09 \\
\hline
-3.71 &  2.31 &  6.89 & -1.79 &  1.12 & -0.12 & -5.61 &  3.44 &  4.10 &  0.85 & -3.94 &  3.17 & -3.51 &  0.62 &  \mathbf{7.03} & -5.15 \\
 4.11 &  1.23 & -4.84 &  0.19 & -3.60 &  2.16 &  3.36 & -0.01 & -4.31 &  0.17 &  3.97 & -2.02 &  2.49 & -0.04 & -4.11 &  2.57 \\
 3.72 &  2.20 & -3.51 &  1.67 & -5.30 & -0.01 &  4.53 & -2.49 & -3.30 &  0.32 &  2.06 & -4.40 &  3.14 & -1.10 & \mathbf{-6.54} &  1.40 \\
-2.85 &  1.72 &  5.21 & -0.32 &  4.97 &  1.00 & -5.68 & -1.24 &  3.87 & -1.45 & -5.40 &  1.67 & -4.30 &  1.20 &  3.73 & -3.39
\end{array}
\right]
$}
\end{equation*}
\centerline{group $0$: $\max - \min = \boldsymbol{13.52}$ \quad|\quad group $1$: $\max - \min = \boldsymbol{13.58}$.}

Every one of the $128$ entries now sits in a tight band $[-6.79,\, 7.03]$, the two groups are within $0.5\%$ of each other in dynamic range, and the absolute range collapses from $37.85$ (Step~1) to $13.58$ --- a $2.8{\times}$ reduction. The importance metric also equalizes exactly: Lemma~\ref{lem:hadamard-diag} applied to $\Lambda_Q$ predicts every diagonal entry of $H_{\mathrm{Had}}^\top \Lambda_Q H_{\mathrm{Had}}$ collapses to the constant $\tr(\Lambda_Q)/d = 443/128 \approx 3.46$, regardless of how peaky $\Lambda_Q$ was. Table~\ref{tab:worked-example-layer10} confirms $\max_i (R_K^\top C_Q R_K)_{ii} / \mathrm{mean}_i(\cdot) = 1.00$ exactly: \emph{the $46.9{\times}$ peak in the $Q^\top Q$ importance is compressed to $1.00{\times}$}. This is the key invariant Step~2 produces; the per-group dynamic range alone is not yet better than pure Hadamard (compare to Step~5 below).

\textit{Step 3: $\tilde K = K U_Q H_{\mathrm{Had}} P_K$ (permuted bit-reversal).} $P_K$ is a permutation matrix that reorders the $128$ output channels. After $U_Q H_{\mathrm{Had}}$ each Walsh output is a $\pm 1$-weighted sum of all eigenchannels, but the variances of those sums are not identical (driven by which Walsh row a high-eigenvalue eigvec activates most strongly in). PBR places the $k$-th largest eigenvector at Walsh output index $\beta(k)$ (the bit-reversal of $k$). At $G_K = 64$ this guarantees the top-$2$ eigenvectors split into the two groups (one each), the top-$4$ distribute one-per-quartile, and so on for any power-of-two grouping, so no single group inherits a disproportionate share of the high-energy eigenchannels.
\begin{equation*}
\resizebox{\textwidth}{!}{$
(K U_Q H_{\mathrm{Had}} P_K)_{t,\cdot} \;=\;
\left[
\begin{array}{rrrrrrrrrrrrrrrr}
 2.24 & -3.71 & -2.86 &  3.72 & -0.87 &  4.11 &  1.90 & -2.85 & -3.72 &  4.10 &  1.52 & -3.30 &  1.91 & -4.31 & -2.31 &  3.87 \\
-5.76 &  1.12 &  4.37 & -5.30 &  5.12 & -3.60 & -2.58 &  4.97 &  3.19 & -3.51 & -3.09 &  3.14 & -1.67 &  2.49 & -0.07 & -4.30 \\
-3.73 &  6.89 &  6.73 & -3.51 &  4.12 & -4.84 & -6.06 &  5.21 &  5.55 & -3.94 & -4.94 &  2.06 & -4.91 &  3.97 &  5.04 & -5.40 \\
 4.08 & -5.61 & -5.77 &  4.53 & -6.79 &  3.36 &  4.18 & -5.68 & -4.16 &  \mathbf{7.03} &  4.34 & \mathbf{-6.54} &  4.47 & -4.11 & -2.17 &  3.73 \\
\hline
 1.08 &  2.31 &  0.66 &  2.20 &  0.64 &  1.23 & -1.30 &  1.72 &  0.29 &  0.85 & -0.05 &  0.32 & -2.08 &  0.17 &  0.88 & -1.45 \\
 0.92 & -0.12 &  1.49 & -0.01 &  1.16 &  2.16 &  3.27 &  1.00 & -3.70 &  0.62 &  2.28 & -1.10 &  1.97 & -0.04 & -1.10 &  1.20 \\
 0.26 & -1.79 &  1.83 &  1.67 & -3.18 &  0.19 & -1.90 & -0.32 & -2.67 &  3.17 &  3.97 & -4.40 &  4.21 & -2.02 & -2.61 &  1.67 \\
-3.39 &  3.44 &  0.72 & -2.49 &  0.97 & -0.01 & -3.10 & -1.24 &  1.80 & -5.15 & -0.37 &  1.40 & -0.63 &  2.57 &  4.09 & -3.39
\end{array}
\right]
$}
\end{equation*}
\centerline{group $0$: $\max - \min = \boldsymbol{13.82}$ \quad|\quad group $1$: $\max - \min = \boldsymbol{9.36}$.}

The \emph{set} of $128$ values is identical to Step~2 --- only column-index positions are shuffled --- but $P_K$ has gathered most of the high-magnitude entries ($\{7.03, -6.54, -6.79, -6.06, 6.89, 6.73, \ldots\}$) into group~$0$ (rows~$0$--$3$) and left group~$1$ (rows~$4$--$7$) with a much tighter $[-5.15, 4.21]$ range. On this particular token PBR has increased one group's dynamic range slightly ($13.52 \to 13.82$) while shrinking the other's ($13.58 \to 9.36$); the residual on the tighter group drops as $(9.36/13.58)^2 \approx 0.48$ while the looser group is essentially unchanged. Averaged over all $T = 8000$ tokens, this gathering effect reduces the mean per-group max--min from $13.40$ ($U_Q H_{\mathrm{Had}}$) to $11.58$ (OSCAR), translating to a $1.23{\times}$ reduction in $\tr(E_K)$ ($208 \to 169$). At this layer, \emph{PBR carries most of OSCAR's quantization-residual improvement} over $U_Q H_{\mathrm{Had}}$ alone; the role of $U_Q H_{\mathrm{Had}}$ is to make Lemma~\ref{lem:hadamard-diag} fire so that the importance metric is equalized exactly ($\max/\mathrm{mean} = 1.00$).

\textbf{Degenerate case ($G_K = d$).}  pure Hadamard, QuaRot-style. A natural simpler alternative is to drop $U_Q$ entirely and rotate by $H_{\mathrm{Had}}$ alone:
\begin{equation*}
\resizebox{\textwidth}{!}{$
(K H_{\mathrm{Had}})_{t,\cdot} \;=\;
\left[
\begin{array}{rrrrrrrrrrrrrrrr}
 2.20 & -0.20 &  3.44 &  2.25 & -1.33 &  2.33 &  3.83 &  2.71 & -2.34 & -5.96 &  4.20 &  5.20 & -3.46 & -0.77 &  4.42 &  4.48 \\
 \mathbf{7.15} &  5.23 & -4.17 & -3.54 &  5.85 &  5.28 & -3.61 & -1.41 &  1.05 &  3.94 & -1.22 & -2.85 &  1.48 &  1.03 &  2.05 & -0.46 \\
 1.02 & -1.19 & -2.61 & -2.01 &  0.93 &  3.00 & -1.01 & -3.19 &  3.09 &  2.20 & -5.85 & -5.58 &  4.90 &  3.06 & -4.39 & -3.55 \\
-3.07 & -3.65 &  0.48 &  4.83 & -3.59 & -2.92 &  3.38 &  5.09 & -2.81 & -0.89 &  1.22 &  0.46 & -3.49 & -2.21 &  2.88 &  0.84 \\
\hline
-1.53 &  1.03 &  0.94 &  0.19 & -2.30 &  0.20 &  5.16 &  1.89 & -1.37 & -3.52 &  4.74 &  4.87 & -3.51 & -4.25 &  4.14 &  5.80 \\
 5.48 &  5.35 & -4.44 & -1.88 &  7.12 &  5.46 & -1.44 & -2.71 &  1.37 &  3.77 &  0.45 & -2.59 &  3.21 &  2.03 & -0.67 &  0.81 \\
-1.55 &  2.51 & -2.80 & -3.34 &  0.29 &  1.81 & -2.03 & -0.88 &  5.68 &  4.70 & \mathbf{-6.89} & -5.69 &  1.38 &  1.07 & -4.79 & -2.27 \\
-3.94 & -4.43 &  2.41 &  5.23 & -2.12 & -5.19 &  2.62 &  2.32 & -1.78 & -2.21 &  0.40 &  0.43 & -0.79 & -1.48 &  1.86 &  2.20
\end{array}
\right]
$}
\end{equation*}
\centerline{group $0$: $\max - \min = \boldsymbol{13.12}$ \quad|\quad group $1$: $\max - \min = \boldsymbol{14.01}$.}

Hadamard mixing on its own balances the two groups ($13.12 \approx 14.01$) and on this token gives a per-group dynamic range comparable to $U_Q H_{\mathrm{Had}}$ (Step~2). The crucial structural difference is in the importance metric: $C_Q$ is not diagonal in the channel basis, so Lemma~\ref{lem:hadamard-diag} does not apply to $H_{\mathrm{Had}}^\top C_Q H_{\mathrm{Had}}$. We measure $\max / \mathrm{mean} = 1.72$ rather than $1.00$ --- the importance \emph{floor} is not equalized.

\begin{table}[ht]
  \centering\footnotesize
  \setlength{\tabcolsep}{4pt}
  \resizebox{\textwidth}{!}{%
  \begin{tabular}{lcccc}
    \toprule
    & \multicolumn{2}{c}{\emph{what the quantizer sees on $\tilde K = K R_K$}} & \emph{importance} & \emph{INT2 residual} \\
    \cmidrule(lr){2-3}\cmidrule(lr){4-4}\cmidrule(lr){5-5}
    Rotation $R_K$ & $\max|\tilde K|$ & per-group max--min ($G_K = 64$) & $\dfrac{\max_i (R_K^\top C_Q R_K)_{ii}}{\mathrm{mean}_i (\cdot)}$ & $\tr(E_K)$ \\
    \midrule
    $I$ (no rotation)                                  & 36.3  & 25.7          & 43.0          & 233           \\
    $H_{\mathrm{Had}}$ \ (pure Hadamard, QuaRot-style) & 11.1  & 13.3          & 1.72          & 206           \\
    $U_Q$ \ (eigenbasis alone)                         & 29.4  & 22.4          & 46.9          & \textbf{1354} \\
    $U_Q\,H_{\mathrm{Had}}$                            & 7.46  & 13.4          & \textbf{1.00} & 208           \\
    $U_Q\,H_{\mathrm{Had}}\,P_K$ (OSCAR)               & 7.46  & \textbf{11.6} & \textbf{1.00} & \textbf{169}  \\
    \bottomrule
  \end{tabular}
  }
  \caption{Layer 10 of Qwen3-4B-Thinking, kv-head 0, $d = 128$, INT2 per-group ($G_K = 64$, two groups per head), $T = 8000$ tokens. $C_Q = Q^\top Q$ averaged over GQA groups (production setting). $\max |\tilde K|$ is the largest entry across all $T \times d$ positions; \emph{per-group max--min} is $\mathbb E_t\bigl[\max_{i \in g} \tilde K_{t,i} - \min_{i \in g} \tilde K_{t,i}\bigr]$ averaged over the two $G_K = 64$ groups. The importance column measures how peaky the rotated importance metric's diagonal is ($1.00 = $ all channels equally important; Lemma~\ref{lem:hadamard-diag} forces this when $C_Q$ is pre-diagonalized by $U_Q$). $\tr(E_K)$ is the unweighted INT2 residual. At $G_K = d = 128$ (one group per head) the permutation $P_K$ would be a no-op and OSCAR collapses to $U_Q H_{\mathrm{Had}}$; we deliberately use $G_K = 64$ here so the PBR step has a non-trivial effect. Bold entries are extreme values referenced in the text.}
  \label{tab:worked-example-layer10}
\end{table}

The two ``what the quantizer sees'' columns make the OSCAR advantage on quantization residual concrete. Pure Hadamard cuts the per-group max--min from $25.7$ (raw) to $13.3$ ($1.9{\times}$); OSCAR cuts it further to $11.6$ ($2.2{\times}$ over raw). Since INT2 residual scales as the square of dynamic range, the $\tr(E_K)$ ratios should be $(25.7/11.6)^2 \approx 4.9$ and $(13.3/11.6)^2 \approx 1.3$, matching empirical $233/169 = 1.4$ and $206/169 = 1.2$ (the empirical ratios are smaller than the squared dynamic range because the per-token correlation structure of the residual matters; the relation is only proportional, not exact). PBR alone contributes the bulk of OSCAR's advantage over $U_Q H_{\mathrm{Had}}$ at this layer ($\tr(E_K)$: $208 \to 169$, $1.23{\times}$). The importance column tells the orthogonal story: only the $U_Q H_{\mathrm{Had}}$ and OSCAR rows achieve $\max/\mathrm{mean} = 1.00$, the exact diagonal equalization predicted by Lemma~\ref{lem:hadamard-diag}; pure Hadamard alone leaves it at $1.72$, and raw / $U_Q$ alone leave it at $43$--$47$.

\paragraph{Summary.}
The three factors are not interchangeable, and each provides a guarantee the others cannot:
\begin{itemize}[leftmargin=*]
  \item $U_Q$ exposes the importance spectrum by diagonalizing $C_Q$, so per-channel importance becomes exactly $\lambda_j$. By itself it makes per-group dynamic range \emph{worse} on real activations because the $K$-variance gets pushed into low-importance directions, not flattened.
  \item $H_{\mathrm{Had}}$, composed with $U_Q$, equalizes the importance-metric diagonal to a constant $\tr(C_Q)/d$ via Lemma~\ref{lem:hadamard-diag} ($\max/\mathrm{mean} = 1.00$ exactly). Without the $U_Q$ pre-step (pure Hadamard alone), the importance metric is not equalized ($\max/\mathrm{mean} = 1.72$ here), so quantization error in any single rotated channel can land disproportionately on a high-importance direction in $C_Q$ depending on the layer/token distribution.
  \item $P_K$ ensures that for any power-of-two group size $G_K < d$, the importance hierarchy is recursively balanced across groups. This is where most of OSCAR's empirical $\tr(E_K)$ reduction over $U_Q H_{\mathrm{Had}}$ comes from on real activations. When $G_K = d$ it is a column permutation within a single group and has no effect; the PBR step is strictly beneficial only for finer grouping.
\end{itemize}
The same three-factor construction defines $R_V = U_S H_{\mathrm{Had}} P_V$ on the value side, with $C_S = V^\top S^\top S V$ (the attention-score-weighted value covariance) in place of $C_Q$.

\subsection{Scale Determination and Per-token Clipping Details}
\label{app:scale-determination}
Quantization is applied after the composed rotations in Eq.~\eqref{eq:method_composed_rot}.  At decoding step $t$, the key and value rows are first mapped into the rotated coordinates
\[
\widetilde k_t = k_t R_K,
\qquad
\widetilde v_t = v_t R_V.
\]
Before quantization, the implementation applies token-wise percentile clipping separately to keys and values.  Given clip ratios $\rho_K,\rho_V\in(0,1]$, define
\begin{equation*}
\tau_t^{(K)}
=
\operatorname{quantile}_{\rho_K}
\Bigl(\{|\widetilde k_{t,c}|\}_{c=1}^{d}\Bigr),
\qquad
\tau_t^{(V)}
=
\operatorname{quantile}_{\rho_V}
\Bigl(\{|\widetilde v_{t,c}|\}_{c=1}^{d}\Bigr).
\end{equation*}
The clipped rows are
\[
\widetilde k_t^{\mathrm{clip}}
=
\operatorname{clip}\!\left(\widetilde k_t,-\tau_t^{(K)},\tau_t^{(K)}\right),
\qquad
\widetilde v_t^{\mathrm{clip}}
=
\operatorname{clip}\!\left(\widetilde v_t,-\tau_t^{(V)},\tau_t^{(V)}\right).
\]

In the current implementation, the default clip ratios are calibrated offline from rotated calibration activations, yielding $(\rho_K,\rho_V)=(0.96,0.92)$.

For keys, the default implementation uses asymmetric INT2 quantization, optionally with groups along the channel dimension. Let $b=2$ and $q_{\max}=2^b-1=3$. For a group $g$ of a clipped rotated key row, write
\[
\widetilde k_{t,g}^{\mathrm{clip}}\in\mathbb{R}^{1\times G_K},
\]
where $G_K$ is the key quantization group size. The implementation computes a dynamic min--max scale and zero point for each token and group:
\[
a_{t,g}^{(K)}
=
\min_c \widetilde k_{t,g,c}^{\mathrm{clip}},
\qquad
b_{t,g}^{(K)}
=
\max_c \widetilde k_{t,g,c}^{\mathrm{clip}},
\]

\[
s_{t,g}^{(K)}
=
\frac{b_{t,g}^{(K)}-a_{t,g}^{(K)}}{q_{\max}},
\qquad
z_{t,g}^{(K)}
=
-\frac{a_{t,g}^{(K)}}{s_{t,g}^{(K)}}.
\]

The quantized and dequantized key group is then
\[
\mathcal Q^+\!\left(\widetilde k_{t,g}^{\mathrm{clip}}\right)
=
\operatorname{clip}\!\left(
\left\lfloor
\frac{\widetilde k_{t,g}^{\mathrm{clip}}}{s_{t,g}^{(K)}}+z_{t,g}^{(K)}
\right\rceil,
0,q_{\max}
\right),
\]

and
\[
\mathcal Q\!\left(\widetilde k_{t,g}^{\mathrm{clip}}\right)
=
\mathcal Q^-\!\left(
\mathcal Q^+\!\left(\widetilde k_{t,g}^{\mathrm{clip}}\right)
\right)
=
s_{t,g}^{(K)}
\left(
\mathcal Q^+\!\left(\widetilde k_{t,g}^{\mathrm{clip}}\right)
-
z_{t,g}^{(K)}
\right).
\]

Thus, unlike a fixed Frobenius-norm scale, the key scale is computed dynamically from the clipped rotated key row, independently for each token and each group. In our default configuration, $G_K=64$.

Values use the same affine asymmetric INT2 quantization rule. For a group $g$ of a clipped rotated value row, write
\[
\widetilde v_{t,g}^{\mathrm{clip}}\in\mathbb{R}^{1\times G_V},
\]
where $G_V$ is the value quantization group size. The implementation computes
\[
a_{t,g}^{(V)}
=
\min_c \widetilde v_{t,g,c}^{\mathrm{clip}},
\qquad
b_{t,g}^{(V)}
=
\max_c \widetilde v_{t,g,c}^{\mathrm{clip}},
\]
and
\[
s_{t,g}^{(V)}
=
\frac{b_{t,g}^{(V)}-a_{t,g}^{(V)}}{q_{\max}},
\qquad
z_{t,g}^{(V)}
=
-\frac{a_{t,g}^{(V)}}{s_{t,g}^{(V)}}.
\]
The quantized and dequantized value group is
\[
\mathcal Q^+\!\left(\widetilde v_{t,g}^{\mathrm{clip}}\right)
:=
\operatorname{clip}\!\left(
\left\lfloor
\frac{\widetilde v_{t,g}^{\mathrm{clip}}}{s_{t,g}^{(V)}}+z_{t,g}^{(V)}
\right\rceil,
0,q_{\max}
\right),
\]
and
\[
\mathcal Q\!\left(\widetilde v_{t,g}^{\mathrm{clip}}\right)
:=
s_{t,g}^{(V)}
\left(
\mathcal Q^+\!\left(\widetilde v_{t,g}^{\mathrm{clip}}\right)
-
z_{t,g}^{(V)}
\right).
\]

Finally, after quantization--dequantization in the rotated coordinates, the approximations are mapped back to the original coordinate system:
\[
\widehat k_t
=
\mathcal Q\!\left(\widetilde k_t^{\mathrm{clip}}\right)R_K^\top,
\qquad
\widehat v_t
=
\mathcal Q\!\left(\widetilde v_t^{\mathrm{clip}}\right)R_V^\top.
\]

\subsection{Proof of Theorem~\ref{thm:ambient_basis_simple_variants}}
\label{app:ambient_basis_simple_variants}

We prove the key-side statement first.  The row-vector rotation convention gives the frozen-error surrogate
\[
\widetilde{\mathcal L}_K(R_k)
=
\tr\!\left(R_k^\top C_Q R_k E_K\right),
\qquad R_k^\top R_k=I_d,
\]
where
\[
C_Q=U_Q\Lambda_QU_Q^\top,
\qquad
\Lambda_Q=\diag(\lambda_1,\dots,\lambda_d),
\qquad
\lambda_1\ge\cdots\ge\lambda_d\ge 0,
\]
and
\[
E_K=\diag(\mu_1,\dots,\mu_d),
\qquad
\mu_1\le\cdots\le\mu_d.
\]
Here $E_K$ is treated as a fixed frozen residual covariance matrix.

Let
\[
R_k=U_QZ,
\qquad
Z^\top Z=I_d.
\]
Substituting into the objective gives
\begin{align*}
\widetilde{\mathcal L}_K(R_k)
&=
\tr\!\left(
Z^\top U_Q^\top C_Q U_Q Z E_K
\right)\\
&=
\tr\!\left(
Z^\top \Lambda_Q Z E_K
\right).
\end{align*}
Writing $Z=(z_{ij})$, we have
\[
\bigl(Z^\top\Lambda_QZ\bigr)_{ii}
=
\sum_{j=1}^d \lambda_j z_{ji}^2.
\]
Therefore,
\begin{equation}
\widetilde{\mathcal L}_K(R_k)
=
\sum_{i=1}^d\sum_{j=1}^d
\mu_i\lambda_j z_{ji}^2.
\label{eq:key_ambient_correct_expansion}
\end{equation}
Since $Z$ is orthogonal, the matrix $W=(w_{ji})$ with $w_{ji}=z_{ji}^2$ is doubly stochastic. Hence Eq.~\eqref{eq:key_ambient_correct_expansion} is a linear function over the Birkhoff polytope, and its minimum is attained at a permutation matrix.

Thus,
\[
\min_{R_k^\top R_k=I_d}
\widetilde{\mathcal L}_K(R_k)
=
\min_{\Pi\in\mathcal P_d}
\tr\!\left(\Pi^\top \Lambda_Q\Pi E_K\right),
\]
where $\mathcal P_d$ denotes the set of $d\times d$ permutation matrices. 
By the rearrangement inequality, since
\[
\lambda_1\ge\cdots\ge\lambda_d
\qquad\text{and}\qquad
\mu_1\le\cdots\le\mu_d,
\]
the minimum is achieved by pairing the largest $\lambda_j$ with the smallest $\mu_i$. 
This is achieved by the identity permutation. Therefore, one minimizer is
\[
Z=I_d,
\]
and hence
\[
R_k=U_Q.
\]

The value-side proof is identical. 
The value surrogate is
\[
\widetilde{\mathcal L}_V(R_v)
=
\tr\!\left(R_v^\top C_S R_v E_V\right),
\qquad
R_v^\top R_v=I_{d},
\]
where
\[
C_S=U_S\Lambda_SU_S^\top,
\qquad
\Lambda_S=\diag(\nu_1,\dots,\nu_{d}),
\qquad
\nu_1\ge\cdots\ge\nu_{d}\ge 0,
\]
and
\[
E_V=\diag(\eta_1,\dots,\eta_{d}),
\qquad
\eta_1\le\cdots\le\eta_{d}.
\]
Again, $E_V$ is treated as a fixed frozen residual covariance matrix.

Let
\[
R_v=U_SZ_V,
\qquad
Z_V^\top Z_V=I_{d}.
\]
Then
\begin{align*}
\widetilde{\mathcal L}_V(R_v)
&=
\tr\!\left(
Z_V^\top U_S^\top C_S U_S Z_V E_V
\right)\\
&=
\tr\!\left(
Z_V^\top\Lambda_SZ_VE_V
\right).
\end{align*}
Writing $Z_V=(\zeta_{ij})$, we obtain
\[
\widetilde{\mathcal L}_V(R_v)
=
\sum_{i=1}^{d}\sum_{j=1}^{d}
\eta_i\nu_j\zeta_{ji}^2.
\]
As before, $(\zeta_{ji}^2)$ is doubly stochastic.  The minimum is therefore attained at a permutation matrix. 
Since
\[
\nu_1\ge\cdots\ge\nu_{d}
\qquad\text{and}\qquad
\eta_1\le\cdots\le\eta_{d},
\]
the rearrangement inequality shows that the identity permutation is optimal. 
Thus one minimizer is
\[
Z_V=I_{d},
\]
and therefore
\[
R_v=U_S.
\]
This proves both statements.

\subsection{Justification of Surrogate Objectives in Theorem~\ref{thm:ambient_basis_simple_variants}}
\label{app:surrogate_objective_justification}
To justify the surrogate objectives in Theorem~\ref{thm:ambient_basis_simple_variants}, we start from the downstream attention errors induced by quantizing keys and values. For the key branch, the empirical logit distortion on the calibration set is
\[
\ell_K(R_k)
=
\sum_{i=1}^N \sum_{j=1}^{i}
\bigl[q_i(k_j-\hat k_j)^\top\bigr]^2,
\]
where $\hat k_j=\mathcal Q(k_jR_k)R_k^\top$. Writing the rotated quantization residual as
\[
e_j^{(K)}(R_k):=\mathcal Q(k_jR_k)-k_jR_k,
\]
we have $\hat k_j-k_j=e_j^{(K)}(R_k)R_k^\top$, so that
\[
\ell_K(R_k)
=
\sum_{j=1}^N e_j^{(K)}(R_k)\,R_k^\top C_j R_k\,e_j^{(K)}(R_k)^\top,
\qquad
C_j:=\sum_{i=j}^N q_i^\top q_i.
\]
To obtain a tractable objective, we approximate the position-dependent matrices $C_j$ by the global query target covariance
\[
C_Q=\frac1N\sum_{i=1}^N q_i^\top q_i.
\]
Defining the frozen residual covariance (where $R_k$ is fixed)
\[
E_K:=\sum_{j=1}^N e_j^{(K)}(R_k)^\top e_j^{(K)}(R_k),
\]
the key-side frozen-error objective becomes, up to the constant normalization factor $\frac{1}{N}$,
\[
\tilde{\mathcal L}_K(R_k)
=
\tr(R_k^\top C_Q R_k E_K).
\]



Under a similar approximation, the value-side downstream error is heuristically represented by the frozen-error surrogate
\[
\tilde{\mathcal L}_V(R_v)
=
\tr\!\bigl(R_v^\top C_S R_v E_V\bigr).
\]
On long calibration sequences where materializing the full $T\times T$ attention matrix $S$ is expensive, our implementation approximates $C_S^\star=V^\top S^\top SV$ with a streaming score-energy estimator that weights each value row $v_{t,h}$ of KV head $h$ by $w_{t,h}=k_{t,h}C_{Q,h}k_{t,h}^\top$ and accumulates
\[
C_S=\frac{1}{H_{\mathrm{kv}}}\sum_{h=1}^{H_{\mathrm{kv}}}\frac{\sum_t w_{t,h}\,v_{t,h}^\top v_{t,h}}{\sum_t w_{t,h}}.
\]

\subsection{Empirical Check of the Diagonal Residual Assumption}
\label{app:residual_check}

To evaluate how closely the deployed quantizer satisfies the diagonal residual assumption in Theorem~\ref{thm:ambient_basis_simple_variants}, we measured the uncentered residual second moment $E=\mathbb{E}[\epsilon^\top\epsilon]$ in the rotated coordinates across all 36 layers of Qwen3-4B-Thinking-2507 using the deployed clipped, grouped, asymmetric INT2 path (measuring $\epsilon$ relative to the pre-clipping rotated activation so that clipping error is included).
Table~\ref{tab:residual_check} reports the layer means of the off-diagonal Frobenius ratio $r=\|E-\diag(E)\|_F/\|E\|_F$, the diagonal squared-Frobenius energy fraction $1-r^2$, and the max/mean diagonal ratio.
Under the full OSCAR rotation ($UHP$), $98.4\%$ of key residual energy and $95.5\%$ of value residual energy lies on the diagonal, and the max/mean diagonal ratio drops to $1.27$ and $1.09$, supporting the diagonal surrogate as a reasonable first-order approximation while confirming that the residual depends on the rotation.

\begin{table}[ht]
\centering
\footnotesize
\caption{Residual second-moment diagnostics of the deployed INT2 quantizer, averaged across all 36 layers of Qwen3-4B-Thinking-2507.}
\label{tab:residual_check}
\setlength{\tabcolsep}{5pt}
\begin{tabular}{@{}l ccc ccc@{}}
\toprule
& \multicolumn{3}{c}{\textbf{Keys}} & \multicolumn{3}{c}{\textbf{Values}} \\
\cmidrule(lr){2-4}\cmidrule(lr){5-7}
\textbf{Statistic} & $I$ & $U$ & $UHP$ & $I$ & $U$ & $UHP$ \\
\midrule
Off-diagonal ratio $r$ & 0.248 & 0.218 & 0.126 & 0.301 & 0.165 & 0.213 \\
Diagonal energy $1-r^2$ & 93.8\% & 95.2\% & 98.4\% & 90.9\% & 97.3\% & 95.5\% \\
Max/mean diagonal & 38.38 & 42.29 & 1.27 & 6.66 & 13.82 & 1.09 \\
\bottomrule
\end{tabular}
\end{table}  



\section{Algorithm Flow and Serving Procedure}
\label{app:algorithm_flow}

Algorithm~\ref{alg:OSCAR_pipeline} gives the concrete OSCAR procedure used in our experiments and SGLang implementation.

\begin{algorithm}[t]
\caption{OSCAR calibration, prefill, and decode}
\label{alg:OSCAR_pipeline}
\begin{algorithmic}[1]
\STATE \textbf{Hyperparameters:} sink length $S_0$, recent window $W$, block size $G{=}128$, bit-width $b{=}2$
\STATE \textbf{State:} BF16 sink/recent cache, rotated BF16 staging cache, packed INT2 history cache

\STATE \textbf{procedure Calibrate}$(\mathcal D)$
\FOR{each layer and KV head}
    \STATE $Q,K,V \leftarrow \mathrm{ForwardDump}(\mathcal D)$
    \STATE $S \leftarrow \softmax_{\mathrm{row}}(QK^\top/\sqrt d+M)$
    \STATE $C_K \leftarrow \frac{1}{N}\sum_n q_n^\top q_n$
    \STATE $C_S \leftarrow \frac{1}{N}V^\top S^\top S V$
    \STATE $U_Q \leftarrow \mathrm{EigVec}(C_K)$,\quad $U_S \leftarrow \mathrm{EigVec}(C_S)$
    \STATE $R_K \leftarrow U_Q H_{\mathrm{Had}}P_{\mathrm{br}}$,\quad $R_V \leftarrow U_S H_{\mathrm{Had}}P_{\mathrm{br}}$
    \STATE $c_K,c_V \leftarrow \mathrm{CalibrateClip}(K R_K, V R_V)$
\ENDFOR
\STATE \textbf{return} $\{R_K,R_V,c_K,c_V\}$
\STATE \textbf{end procedure}

\STATE \textbf{procedure Prefill}$(X)$
\STATE $Q,K,V \leftarrow \mathrm{Forward}(X)$
\STATE $\widetilde K \leftarrow K R_K$,\quad $\widetilde V \leftarrow V R_V$ \COMMENT{rotate before prefill write}
\STATE $\mathrm{Cache}_{\mathrm{sink}} \leftarrow (K,V)_{[:S_0]}$
\STATE $\mathrm{Cache}_{\mathrm{recent}} \leftarrow (K,V)_{[-W:]}$
\FOR{$i=S_0$ \TO $|X|-W-1$}
    \STATE \textsc{QuantizeAndWrite}$(\widetilde K_i,\widetilde V_i,i)$
\ENDFOR
\STATE \textbf{end procedure}

\STATE \textbf{procedure DecodeStep}$(x_t)$
\STATE $q_t,k_t,v_t \leftarrow \mathrm{Forward}(x_t)$
\STATE $\widetilde k_t \leftarrow k_tR_K$,\quad $\widetilde v_t \leftarrow v_tR_V$
\STATE Append $(k_t,v_t)$ to $\mathrm{Cache}_{\mathrm{recent}}$; append $(\widetilde k_t,\widetilde v_t)$ to rotated staging
\IF{$|\mathrm{Cache}_{\mathrm{recent}}|>W$}
    \STATE $(\widetilde k_j,\widetilde v_j) \leftarrow \mathrm{PopOldestRotatedStaging}()$
    \STATE \textsc{QuantizeAndWrite}$(\widetilde k_j,\widetilde v_j,j)$
\ENDIF
\STATE $\widehat K_{\mathrm{hist}},\widehat V_{\mathrm{hist}} \leftarrow \mathrm{DequantHistory}(\mathrm{Cache}_{\mathrm{INT2}})$
\STATE $\widehat K_{\mathrm{hist}}\leftarrow \widehat K_{\mathrm{hist}}R_K^\top$,\quad $\widehat V_{\mathrm{hist}}\leftarrow \widehat V_{\mathrm{hist}}R_V^\top$
\STATE $K_{\mathrm{all}},V_{\mathrm{all}} \leftarrow \mathrm{Concat}(\mathrm{Cache}_{\mathrm{sink}},\widehat K_{\mathrm{hist}},\mathrm{Cache}_{\mathrm{recent}})$
\STATE $o_t \leftarrow \softmax(q_tK_{\mathrm{all}}^\top/\sqrt d+M)V_{\mathrm{all}}$
\STATE \textbf{return} $o_t$
\STATE \textbf{end procedure}

\STATE \textbf{function QuantizeAndWrite}$(\widetilde k,\widetilde v,i)$
\STATE $\bar k \leftarrow \clip(\widetilde k,c_K)$,\quad $\bar v \leftarrow \clip(\widetilde v,c_V)$
\STATE $K_i^+ \leftarrow \mathrm{AffineINT2}(\bar k,G)$,\quad $V_i^+ \leftarrow \mathrm{AffineINT2}(\bar v,G)$
\STATE $\mathrm{Cache}_{\mathrm{INT2}}[i] \leftarrow (K_i^+,V_i^+,\mathrm{scale},\mathrm{zero})$
\STATE \textbf{end function}
\end{algorithmic}
\end{algorithm}

\subsection{Grouped-Query Attention and Index Conventions}
\label{app:gqa_details}

Let $H_q$ and $H_{\mathrm{kv}}$ denote the numbers of query and KV heads, $g=H_q/H_{\mathrm{kv}}$ the GQA group size, and $\mathcal G_h$ the $g$ query heads sharing KV head $h$.
For each KV head $h$, we compute $C_{Q,h}=\frac{1}{Tg}\sum_{i\in\mathcal G_h}Q_i^\top Q_i$, average $C_{Q,h}$ across the $H_{\mathrm{kv}}$ heads of the layer to obtain $C_Q$, and similarly average the score-weighted value moments to obtain $C_S$ (Table~\ref{tab:gqa_shapes}).
One $d\times d$ key rotation $R_K$ and one $d\times d$ value rotation $R_V$ are stored per layer and shared across its KV heads, leaving the runtime GQA head mapping unchanged.
Keys are rotated after RoPE, queries are rotated by $R_K$ so that $q_tR_K(k_jR_K)^\top=q_tk_j^\top$, and $R_V$ is folded into $W_V$ and $W_O$.
For the bit-reversal permutation $P_{\mathrm{br}}$~\citep{cooley1965fft,vanloan1992computational}, let $\beta(i)$ reverse the $\log_2 d$ binary bits of $i\in\{0,\dots,d-1\}$; with eigenvectors sorted in descending eigenvalue order $\sigma$, $P_{\mathrm{br}}$ places coordinate $\sigma(i)$ at index $\beta(i)$.

\begin{table}[ht]
\centering
\footnotesize
\caption{Shapes and conventions of calibration quantities ($T$: calibration tokens, $d$: head dimension).}
\label{tab:gqa_shapes}
\begin{tabular}{@{}l c l@{}}
\toprule
\textbf{Quantity} & \textbf{Shape} & \textbf{Convention} \\
\midrule
$Q_i$ & $T\times d$ & post-RoPE query head $i\in\mathcal G_h$ \\
$K_h$, $V_h$ & $T\times d$ & post-RoPE key and value of KV head $h$ \\
$C_{Q,h}$, $C_Q$, $C_S$ & $d\times d$ & calibration second-moment targets \\
$R_K$, $R_V$ & $d\times d$ & one pair per layer, shared across KV heads \\
\bottomrule
\end{tabular}
\end{table}


\section{Related Work}
\label{sec:related_work}
\label{app:extended_related}

\textbf{KV-cache compression.}
The KV cache has become a major target for post-training compression because its memory footprint grows with both context length and batch size.
One line of work reduces or reprioritizes cached content through eviction, budgeted retention, attention-sink preservation, or saliency-aware precision assignment, as in H2O, SnapKV, PyramidKV, StreamingLLM, and ZipCache~\citep{zhang2023h2o,li2024snapkv,cai2024pyramidkv,xiao2024streamingllm,barbero2025first,he2024zipcache}.
Other systems compress the cache through low-rank structure or cross-layer sharing, including GEAR, PALU, xKV, and MatryoshkaKV~\citep{kang2024gear,chang2024palu,chang2025xkv,lin2024matryoshkakv}.
Another line keeps all tokens but lowers the precision of the cached tensors, which is the setting most closely related to OSCAR.
KIVI, KVQuant, and WKVQuant study low-bit KV storage and identify channel-wise outliers, especially in keys, as a central obstacle to aggressive quantization~\citep{liu2024kivi,hooper2024kvquant,yue2024wkvquant}.
These methods show that KV quantization can be accurate, but they also leave open how far precision can be reduced without changing the serving layout.

\textbf{2-bit KV quantization.}
Pushing KV cache quantization to 2 bits usually requires additional structure beyond naive rounding.
SKVQ preserves a recent sliding window in higher precision and combines channel reordering with clipped dynamic quantization~\citep{duanmu2024skvq}.
Kitty promotes sensitive key-cache channels to higher precision and combines this with a residual cache~\citep{xia2025kitty}, while PM-KVQ uses progressive mixed precision and block-wise memory allocation for long-CoT models~\citep{liu2025pmkvq}.
Related mixed-precision analyses also argue that keys should generally receive more precision than values under a fixed memory budget~\citep{hariri2025quantize}.
RotateKV uses outlier-aware adaptive rotations for 2-bit KV quantization~\citep{su2025rotatekv}.
These methods improve accuracy, but their use of channel-wise precision, residual buffers, progressive bit allocation, or adaptive layouts can complicate integration with paged KV-cache serving and fused decode kernels.
OSCAR instead targets the same INT2 regime with fixed token-wise transforms and a uniform paged layout, so the quantization path can be fused into production serving.

\textbf{Rotation-based quantization.}
Orthogonal transforms are an attractive way to make low-bit quantization more accurate without changing tensor shapes.
QuaRot, HALO, and HOT show that rotations can reduce outlier concentration, redistribute activation energy, and improve low-bit rounding behavior~\citep{ashkboos2024quarot,ashkboos2025halo,kim2025hot,egiazarian2026bridging,jia2026saw}, while DuQuant and MixQuant combine block rotations with channel permutations to balance outlier mass across blocks~\citep{lin2024duquant,sanjeet2026mixquant}.
In the KV-cache setting, TurboQuant develops a data-oblivious online vector-quantization method based on random rotations, per-coordinate Lloyd--Max quantization, and QJL-style residual coding~\citep{zandieh2025turboquant}.
KVLinC combines Hadamard rotation with lightweight linear correction to reduce attention errors under aggressive KV quantization~\citep{kvlinc2025}.
OSCAR follows the same broad rotation principle, but asks which rotation target is appropriate for INT2 KV cache quantization. While \cite{liu2024spinquant} effectively mitigates outliers at 4-bit using learned rotations, it collapses at extreme INT2 levels because its rotations fail to align with the actual covariance structures consumed by downstream attention. Furthermore, it focuses primarily on theoretical algorithmic gains, lacking the custom kernels and system-level integration necessary for real-world deployment in modern serving frameworks.

\textbf{Covariance-aware calibration.}
A related family uses calibration data or covariance information to choose quantization or compression directions.
For weight--activation quantization, GPTQ and QuIP use Hessian information to guide low-bit rounding~\citep{frantar2022gptq,chee2023quip}, and WUSH derives closed-form adaptive transforms from weight and activation second moments~\citep{chen2026wush}; for low-rank compression, Drone, ASVD, SVD-LLM, SVD-LLMv2, and CorDA use calibration activations or data-aware statistics to guide factorization~\citep{chen2021drone,yuan2023asvd,wang2024svdllm,wang2025svdllmv2,yang2024corda}.
On the KV-cache side, HaPPI, CARE, RecalKV, and CommonKV apply covariance- or Hessian-aware factorization to keys and values for KV-side projection or compression~\citep{kim2025happi,zhou2026care,yan2025recalkv,wang2025commonkv}.
Post-training quantization methods such as SmoothQuant, AWQ, GPTQ, and OmniQuant also use calibration data to estimate scales, equivalent transformations, clipping thresholds, or second-order surrogates before deployment~\citep{xiao2023smoothquant,lin2023awq,frantar2022gptq,shao2023omniquant}.
OSCAR specializes this calibration paradigm to attention: it estimates target covariances induced by downstream attention ($Q^\top Q$ and score-weighted $V$), rather than raw-cache reconstruction moments ($K^\top K, V^\top V$), and uses them to derive fixed orthogonal post-RoPE rotations and clip thresholds for paged INT2 KV serving (Table~\ref{tab:prior_work}).

\begin{table}[ht]
\centering
\footnotesize
\caption{Comparison of OSCAR with representative low-bit quantizers.}
\label{tab:prior_work}
\setlength{\tabcolsep}{3pt}
\resizebox{\textwidth}{!}{%
\begin{tabular}{@{}l l l l c l@{}}
\toprule
\textbf{Method} & \textbf{Tensors} & \textbf{Transform (statistic)} & \textbf{Granularity} & \textbf{Bits} & \textbf{Runtime layout / overhead} \\
\midrule
KIVI~\citep{liu2024kivi} & KV & none & K per-channel, V per-token & 2 & FP residual buffer \\
Kitty~\citep{xia2025kitty} & KV & none (channel sensitivity) & per-channel boost & $\approx$2.4 & mixed-width channels \\
TurboQuant~\citep{zandieh2025turboquant} & KV & random rotation & per-coordinate Lloyd--Max & 2.5--4 & QJL residual \\
QuaRot~\citep{ashkboos2024quarot} & W/A/KV & randomized Hadamard & per-token group & 4 & online Hadamard \\
SpinQuant~\citep{liu2024spinquant} & W/A/KV & learned rotation (training) & per-token group & 4 & online Hadamard \\
DuQuant, MixQuant~\citep{lin2024duquant,sanjeet2026mixquant} & W/A & rotation + permutation & per-channel/token & 4 & merged into weights \\
WUSH~\citep{chen2026wush} & W/A & non-orthogonal ($W, X$ moments) & block & 4 & fused block transform \\
\rowcolor{blue!8}
\textbf{OSCAR (ours)} & KV & $U H_{\mathrm{Had}} P_{\mathrm{br}}$ ($Q^\top Q$, score-weighted $V$) & per-token group & 2 & fused into writes and $W_V, W_O$; paged \\
\bottomrule
\end{tabular}%
}
\end{table}

\section{Discussion}
\label{sec:discussion}

\paragraph{Limitations.}
Our theoretical analysis explains why the proposed covariance-target rotations are optimal under a frozen-error surrogate with explicit assumptions (Appendix~\ref{app:theory}), but it is not yet a full proof of optimality for the complete autoregressive decoding process.
Tightening this gap---for example, by proving when the calibration estimator is unbiased or optimal for the true attention objective---is an important direction for future work.
We also focus on INT2 KV-cache quantization and do not explore other numerical formats or mixed-format designs, such as NVFP4, different weight precisions, or weight/KV precision combinations, which may create different accuracy--systems trade-offs.
Similarly, OSCAR uses a fixed bit-width for quantized history tokens and does not study per-token bit allocation or thought-adaptive token precision, which recent work such as ThinKV suggests may be useful for long reasoning traces~\citep{ramachandran2025thinkv}.
On the systems side, our implementation targets SGLang-style paged serving on H100 GPUs; additional tuning is needed to understand the best kernel design on other hardware such as B200.
Finally, OSCAR is orthogonal to channel-wise or mixed-precision methods: applying attention-aware/Hadamard rotations inside channel-wise schemes may further improve accuracy, but combining these ideas while preserving an efficient serving layout remains open.

\paragraph{Broader impact.}
OSCAR reduces KV-cache memory and serving cost, which can make long-context LLM inference more accessible and energy-efficient.
At the same time, cheaper long-context inference may also lower the cost of deploying powerful generative models, including applications with misuse risks; OSCAR does not introduce new model capabilities, but should be deployed with the same safety and usage controls as the underlying LLMs.

\section{Additional Experimental Results}
\label{app:additional_results}

\subsection{Additional Motivation Figures}
\label{app:motivation_headline}
Additional activation distributions across layers are shown in Figures~\ref{fig:app_activation_uniformity} and~\ref{fig:app_activation_uniformity_l11}.
\begin{figure}[ht]
    \centering
    \includegraphics[width=0.72\textwidth]{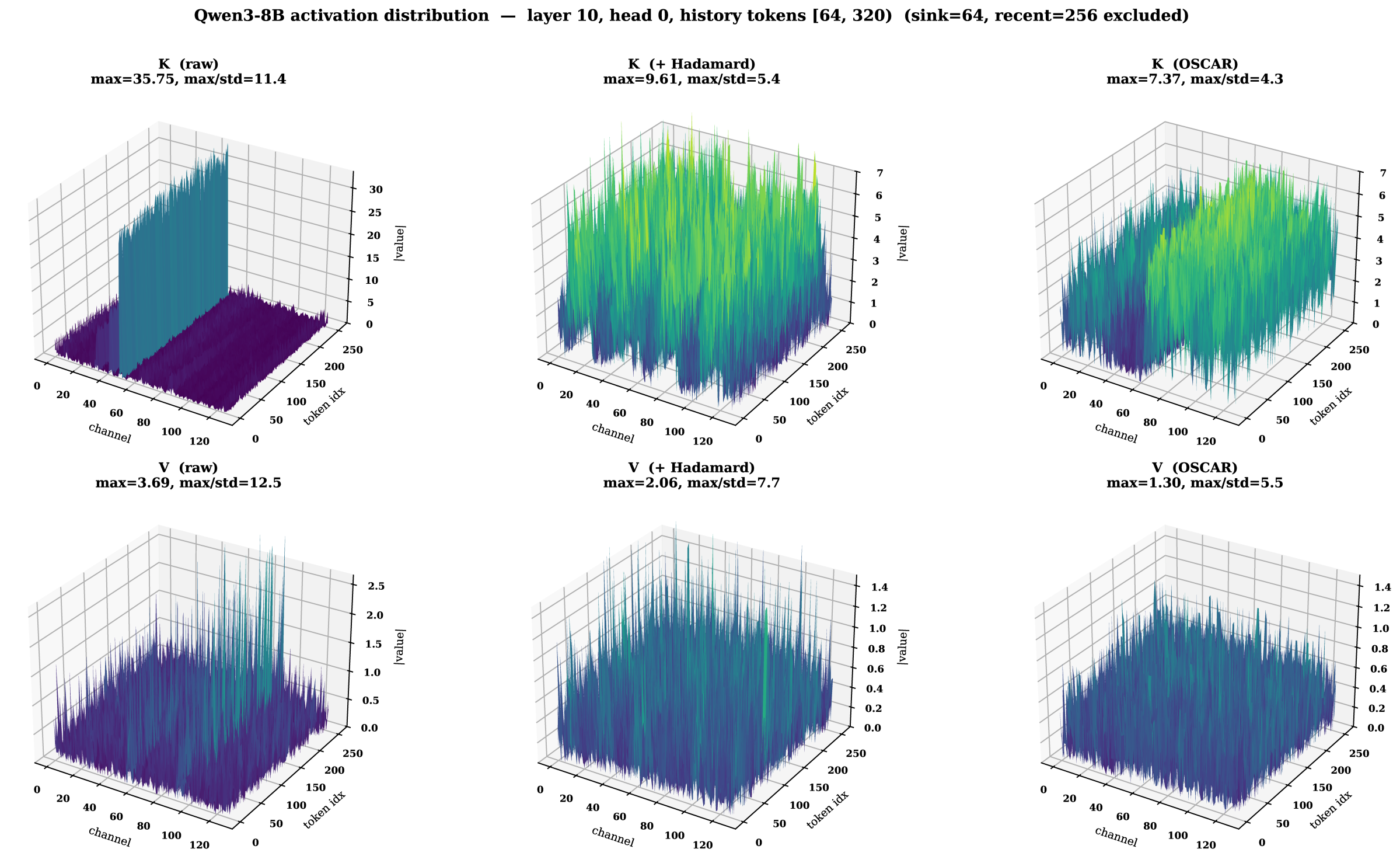}
    \caption{\textbf{Attention-aware rotations make history-token activations easier to quantize (Layer 10, Head 0).}
    Hadamard mixing flattens raw activation peaks by spreading energy across channels. OSCAR goes further: the target covariance separates directions that matter more or less to attention, and the Hadamard transform then mixes each part into a more uniform range.}
    \label{fig:app_activation_uniformity}
\end{figure}

\begin{figure}[ht]
    \centering
    \includegraphics[width=0.72\textwidth]{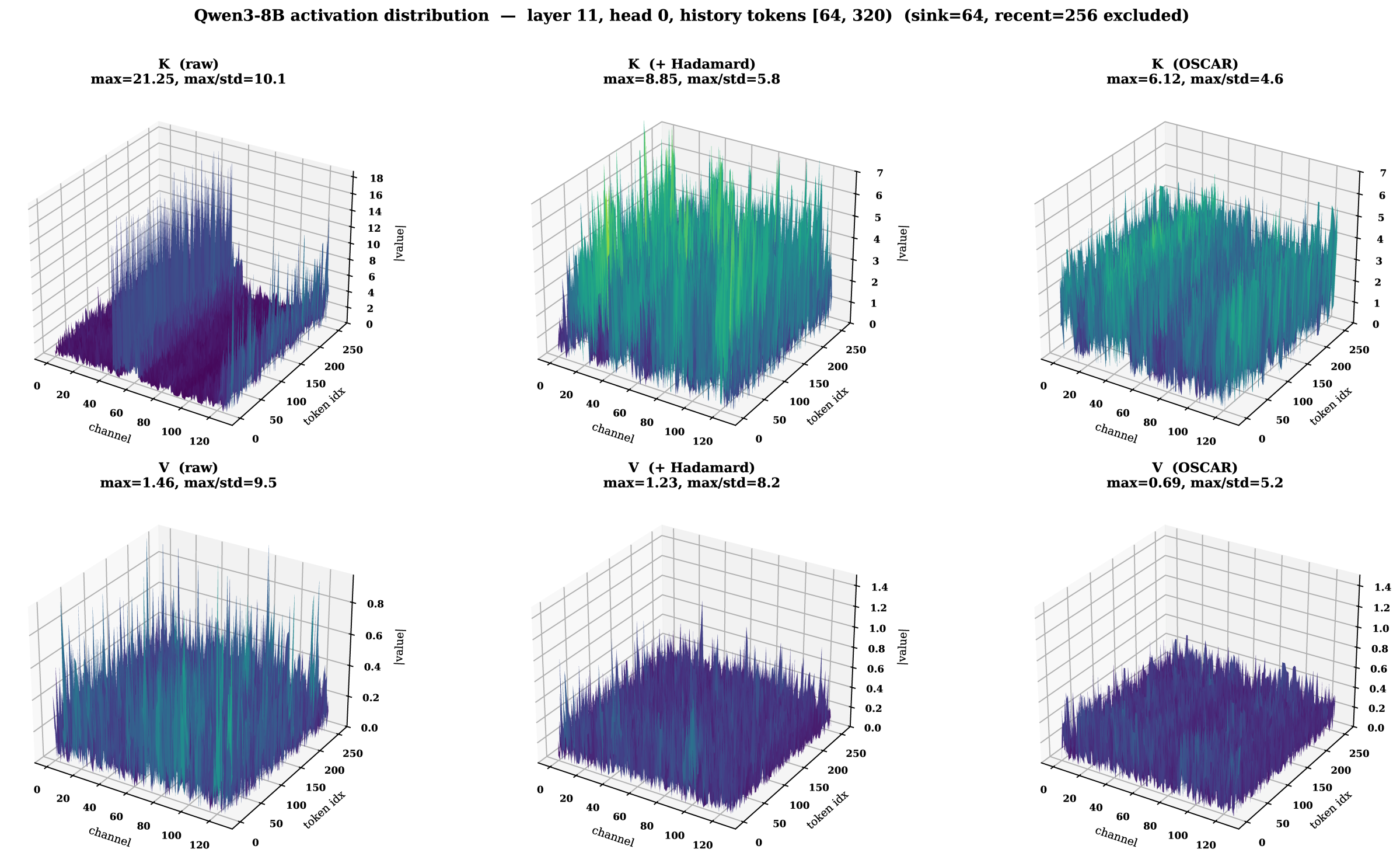}
    \caption{\textbf{Attention-aware rotations make history-token activations easier to quantize (Layer 11, Head 0).}
    Hadamard mixing flattens raw activation peaks by spreading energy across channels. OSCAR goes further: the target covariance separates directions that matter more or less to attention, and the Hadamard transform then mixes each part into a more uniform range.}
    \label{fig:app_activation_uniformity_l11}
\end{figure}

\subsection{Full Table For Main Accuracy Run}
\label{app:qwen3_8b_full_runs}

Table~\ref{tab:app_qwen3_4b_main_full} reports the five independent runs for the Qwen3-4B-Thinking main OSCAR result. Table~\ref{tab:app_qwen3_8b_group0}--\ref{tab:app_qwen3_8b_blockg32} report the five independent runs used to summarize Qwen3-8B OSCAR accuracy under three INT2 grouping configurations. Table~\ref{tab:app_qwen3_32b_g128}--\ref{tab:app_qwen3_32b_g32} report the Qwen3-32B OSCAR runs. Table~\ref{tab:app_glm47_naive_int2}--\ref{tab:app_glm47_OSCAR_full} report the GLM-4.7-FP8 detailed runs for Naive INT2, QuaRot/Hadamard rotation, and OSCAR.

\begin{table}[ht]
\centering
\footnotesize
\caption{Full Qwen3-4B-Thinking results for the main OSCAR configuration (INT2 group size=128, sink=64, recent=256, $c_K=0.96$, $c_V=0.92$).}
\label{tab:app_qwen3_4b_main_full}
\setlength{\tabcolsep}{4pt}
\begin{tabular}{l c c c c c c c}
\toprule
\textbf{Task} & \textbf{Run 1} & \textbf{Run 2} & \textbf{Run 3} & \textbf{Run 4} & \textbf{Run 5} & \textbf{Mean} & \textbf{Std} \\
\midrule
GPQA             & 66.16 & 63.64 & 64.14 & 64.65 & 66.16 & 64.95 & 1.16 \\
HumanEval        & 93.17 & 91.59 & 92.44 & 90.85 & 93.17 & 92.24 & 1.02 \\
LiveCodeBench v6 & 45.03 & 43.27 & 45.61 & 48.54 & 44.44 & 45.38 & 1.97 \\
AIME25           & 63.33 & 63.33 & 63.33 & 60.00 & 70.00 & 64.00 & 3.65 \\
MATH500          & 92.59 & 92.79 & 92.99 & 93.19 & 92.18 & 92.75 & 0.39 \\
\bottomrule
\end{tabular}
\end{table}

\begin{table}[ht]
\centering
\footnotesize
\caption{Full Qwen3-8B results for INT2 group=128.}
\label{tab:app_qwen3_8b_group0}
\setlength{\tabcolsep}{4pt}
\begin{tabular}{l c c c c c c c}
\toprule
\textbf{Task} & \textbf{Run 1} & \textbf{Run 2} & \textbf{Run 3} & \textbf{Run 4} & \textbf{Run 5} & \textbf{Mean} & \textbf{Std} \\
\midrule
GPQA             & 53.03 & 54.04 & 54.55 & 57.58 & 55.05 & 54.85 & 1.70 \\
HumanEval        & 86.59 & 88.54 & 88.90 & 87.56 & 87.68 & 87.85 & 0.91 \\
LiveCodeBench v6 & 45.03 & 45.03 & 42.11 & 47.95 & 46.20 & 45.26 & 2.13 \\
AIME25           & 66.67 & 63.33 & 56.67 & 63.33 & 63.33 & 62.67 & 3.65 \\
MATH500          & 92.59 & 91.38 & 92.59 & 92.18 & 91.98 & 92.14 & 0.50 \\
\bottomrule
\end{tabular}
\end{table}

\begin{table}[ht]
\centering
\footnotesize
\caption{Full Qwen3-8B results for INT2 group=64.}
\label{tab:app_qwen3_8b_blockg64}
\setlength{\tabcolsep}{4pt}
\begin{tabular}{l c c c c c c c}
\toprule
\textbf{Task} & \textbf{Run 1} & \textbf{Run 2} & \textbf{Run 3} & \textbf{Run 4} & \textbf{Run 5} & \textbf{Mean} & \textbf{Std} \\
\midrule
GPQA             & 53.54 & 56.06 & 54.04 & 57.07 & 54.55 & 55.05 & 1.47 \\
HumanEval        & 87.56 & 87.56 & 88.29 & 87.56 & 88.41 & 87.88 & 0.44 \\
LiveCodeBench v6 & 47.37 & 44.44 & 49.12 & 45.03 & 45.61 & 46.32 & 1.91 \\
AIME25           & 63.33 & 70.00 & 66.67 & 63.33 & 70.00 & 66.67 & 3.33 \\
MATH500          & 92.38 & 92.79 & 91.58 & 91.18 & 93.19 & 92.22 & 0.83 \\
\bottomrule
\end{tabular}
\end{table}

\begin{table}[ht]
\centering
\footnotesize
\caption{Full Qwen3-8B results for INT2 group=32.}
\label{tab:app_qwen3_8b_blockg32}
\setlength{\tabcolsep}{4pt}
\begin{tabular}{l c c c c c c c}
\toprule
\textbf{Task} & \textbf{Run 1} & \textbf{Run 2} & \textbf{Run 3} & \textbf{Run 4} & \textbf{Run 5} & \textbf{Mean} & \textbf{Std} \\
\midrule
GPQA             & 52.02 & 55.56 & 59.09 & 52.53 & 56.57 & 55.15 & 2.93 \\
HumanEval        & 87.93 & 87.32 & 88.05 & 88.78 & 88.66 & 88.15 & 0.59 \\
LiveCodeBench v6 & 45.61 & 47.37 & 43.86 & 43.86 & 45.61 & 45.26 & 1.47 \\
AIME25           & 66.67 & 53.33 & 63.33 & 66.67 & 60.00 & 62.00 & 5.58 \\
MATH500          & 91.78 & 92.38 & 92.59 & 92.18 & 92.38 & 92.26 & 0.30 \\
\bottomrule
\end{tabular}
\end{table}

\begin{table}[ht]
\centering
\footnotesize
\caption{Full Qwen3-32B results for OSCAR INT2 group=128.}
\label{tab:app_qwen3_32b_g128}
\setlength{\tabcolsep}{4pt}
\begin{tabular}{l c c c c c c c}
\toprule
\textbf{Task} & \textbf{Run 1} & \textbf{Run 2} & \textbf{Run 3} & \textbf{Run 4} & \textbf{Run 5} & \textbf{Mean} & \textbf{Std} \\
\midrule
GPQA             & 59.60 & 59.09 & 57.58 & 60.10 & --    & 59.09 & 1.09 \\
HumanEval        & 90.37 & 89.76 & 89.27 & 89.76 & 89.27 & 89.68 & 0.45 \\
LiveCodeBench v6 & 54.39 & 49.71 & 47.95 & 52.63 & 53.80 & 51.70 & 2.76 \\
AIME25           & 63.33 & 66.67 & 76.67 & 76.67 & 70.00 & 70.67 & 5.96 \\
MATH500          & 92.79 & 93.19 & 92.99 & 91.58 & 93.19 & 92.75 & 0.67 \\
\bottomrule
\end{tabular}
\end{table}

\begin{table}[ht]
\centering
\footnotesize
\caption{Full Qwen3-32B results for OSCAR INT2 group=64.}
\label{tab:app_qwen3_32b_g64}
\setlength{\tabcolsep}{4pt}
\begin{tabular}{l c c c c c c c}
\toprule
\textbf{Task} & \textbf{Run 1} & \textbf{Run 2} & \textbf{Run 3} & \textbf{Run 4} & \textbf{Run 5} & \textbf{Mean} & \textbf{Std} \\
\midrule
GPQA             & 60.61 & 63.64 & 61.11 & 62.12 & 54.55 & 60.40 & 3.47 \\
HumanEval        & 89.88 & 89.76 & 89.39 & 91.22 & 90.37 & 90.12 & 0.71 \\
LiveCodeBench v6 & 53.80 & 52.63 & 53.80 & 54.39 & 53.22 & 53.57 & 0.67 \\
AIME25           & 76.67 & 70.00 & 80.00 & 66.67 & 76.67 & 74.00 & 5.48 \\
MATH500          & 93.19 & 92.18 & 93.19 & 92.38 & 92.79 & 92.75 & 0.46 \\
\bottomrule
\end{tabular}
\end{table}

\begin{table}[ht]
\centering
\footnotesize
\caption{Full Qwen3-32B results for OSCAR INT2 group=32.}
\label{tab:app_qwen3_32b_g32}
\setlength{\tabcolsep}{4pt}
\begin{tabular}{l c c c c c c c}
\toprule
\textbf{Task} & \textbf{Run 1} & \textbf{Run 2} & \textbf{Run 3} & \textbf{Run 4} & \textbf{Run 5} & \textbf{Mean} & \textbf{Std} \\
\midrule
GPQA             & 63.64 & 57.07 & 56.57 & 60.10 & 59.60 & 59.40 & 2.83 \\
HumanEval        & 92.20 & 90.85 & 89.76 & 90.73 & 90.24 & 90.76 & 0.91 \\
LiveCodeBench v6 & 58.48 & 54.39 & 54.39 & 50.29 & 55.56 & 54.62 & 2.94 \\
AIME25           & 66.67 & 70.00 & 70.00 & 66.67 & 63.33 & 67.33 & 2.79 \\
MATH500          & 93.39 & 92.99 & 92.79 & 93.59 & 93.19 & 93.19 & 0.32 \\
\bottomrule
\end{tabular}
\end{table}

\begin{table}[ht]
\centering
\footnotesize
\caption{Full GLM-4.7-FP8 results for the Naive INT2 baseline.}
\label{tab:app_glm47_naive_int2}
\setlength{\tabcolsep}{5pt}
\begin{tabular}{l c c c c c}
\toprule
\textbf{Task} & \textbf{Run 1} & \textbf{Run 2} & \textbf{Run 3} & \textbf{Mean} & \textbf{Std} \\
\midrule
GPQA             & 58.08 & 51.01 & 54.55 & 54.55 & 3.54 \\
HumanEval        & 85.61 & 88.15 & 85.73 & 86.50 & 1.43 \\
LiveCodeBench v6 & 33.33 & 44.44 & 33.33 & 37.03 & 6.41 \\
AIME25           & 40.00 & 36.67 & 36.67 & 37.78 & 1.92 \\
MATH500          & 87.40 & 85.80 & 86.60 & 86.60 & 0.80 \\
\bottomrule
\end{tabular}
\end{table}

\begin{table}[ht]
\centering
\footnotesize
\caption{Full GLM-4.7-FP8 results for the QuaRot/Hadamard rotation baseline.}
\label{tab:app_glm47_quarot_hadamard}
\setlength{\tabcolsep}{5pt}
\begin{tabular}{l c c c c c}
\toprule
\textbf{Task} & \textbf{Run 1} & \textbf{Run 2} & \textbf{Run 3} & \textbf{Mean} & \textbf{Std} \\
\midrule
GPQA             & 70.20 & 63.64 & 70.20 & 68.01 & 3.79 \\
HumanEval        & 91.71 & 89.76 & 88.15 & 89.87 & 1.78 \\
LiveCodeBench v6 & 47.37 & 53.80 & 44.44 & 48.54 & 4.79 \\
AIME25           & 80.00 & 80.00 & 76.67 & 78.89 & 1.92 \\
MATH500          & 90.20 & 90.20 & 90.80 & 90.40 & 0.35 \\
\bottomrule
\end{tabular}
\end{table}

\begin{table}[ht]
\centering
\footnotesize
\caption{Full GLM-4.7-FP8 results for the OSCAR rotation group size=128 configuration.}
\label{tab:app_glm47_OSCAR_full}
\setlength{\tabcolsep}{5pt}
\begin{tabular}{l c c c c c}
\toprule
\textbf{Task} & \textbf{Run 1} & \textbf{Run 2} & \textbf{Run 3} & \textbf{Mean} & \textbf{Std} \\
\midrule
GPQA             & 73.23 & 74.24 & 73.23 & 73.57 & 0.58 \\
HumanEval        & 90.85 & 91.34 & 90.98 & 91.06 & 0.25 \\
LiveCodeBench v6 & 53.80 & 52.05 & 52.05 & 52.63 & 1.01 \\
AIME25           & 76.67 & 80.00 & 80.00 & 78.89 & 1.92 \\
MATH500          & 93.79 & 94.99 & 95.19 & 94.66 & 0.76 \\
\bottomrule
\end{tabular}
\end{table}

\clearpage
\subsection{Ablation Detail Runs}
\label{app:ablation_detail_runs}

Table~\ref{tab:app_rotation_decomposition_full_runs} reports the individual runs behind the rotation-decomposition study in Table~\ref{tab:ablation_composed}.

\begin{table}[ht]
\centering
\tiny
\caption{Full individual runs for the rotation-decomposition study on Qwen3-8B.}
\label{tab:app_rotation_decomposition_full_runs}
\setlength{\tabcolsep}{2pt}
\begin{tabular}{l l c c c c c}
\toprule
\textbf{Configuration} & \textbf{Task} & \textbf{Run 1} & \textbf{Run 2} & \textbf{Run 3} & \textbf{Mean} & \textbf{Std} \\
\midrule
\multirow{5}{*}{Full OSCAR} 
& GPQA             & 56.06 & 54.55 & 57.07 & 55.89 & 1.27 \\
& HumanEval        & 87.93 & 87.80 & 88.17 & 87.97 & 0.19 \\
& LiveCodeBench v6 & 43.86 & 47.37 & 46.78 & 46.00 & 1.88 \\
& AIME25           & 66.67 & 70.00 & 66.67 & 67.78 & 1.92 \\
& MATH500          & 92.18 & 92.38 & 92.59 & 92.38 & 0.20 \\
\midrule
\multirow{5}{*}{Tensor-recon. target}
& GPQA             & 37.37 & 39.90 & 38.89 & 38.72 & 1.27 \\
& HumanEval        & 35.49 & 36.22 & 34.88 & 35.53 & 0.67 \\
& LiveCodeBench v6 & 4.09  & 4.09  & 2.92  & 3.70  & 0.68 \\
& AIME25           & 13.33 & 16.67 & 10.00 & 13.33 & 3.33 \\
& MATH500          & 64.73 & 63.73 & 64.53 & 64.33 & 0.53 \\
\midrule
\multirow{5}{*}{w/o $P_{\mathrm{br}}$}
& GPQA             & 53.54 & 54.04 & 51.01 & 52.86 & 1.62 \\
& HumanEval        & 87.20 & 85.73 & 87.20 & 86.71 & 0.84 \\
& LiveCodeBench v6 & 45.03 & 45.03 & 45.61 & 45.22 & 0.34 \\
& AIME25           & 73.33 & 60.00 & 56.67 & 63.33 & 8.82 \\
& MATH500          & 92.79 & 90.18 & 92.59 & 91.85 & 1.45 \\
\midrule
\multirow{5}{*}{w/o $H_{\mathrm{Had}}$}
& GPQA             & 50.00 & 53.54 & 50.51 & 51.35 & 1.91 \\
& HumanEval        & 84.15 & 83.41 & 83.78 & 83.78 & 0.37 \\
& LiveCodeBench v6 & 19.30 & 17.54 & 21.64 & 19.49 & 2.05 \\
& AIME25           & 23.33 & 20.00 & 20.00 & 21.11 & 1.92 \\
& MATH500          & 83.77 & 82.97 & 82.16 & 82.97 & 0.80 \\
\midrule
\multirow{5}{*}{w/o $U$}
& GPQA             & 41.92 & 43.43 & 36.87 & 40.74 & 3.44 \\
& HumanEval        & 37.68 & 39.88 & 38.05 & 38.54 & 1.18 \\
& LiveCodeBench v6 & 4.09  & 4.09  & 5.26  & 4.48  & 0.68 \\
& AIME25           & 13.33 & 16.67 & 16.67 & 15.56 & 1.92 \\
& MATH500          & 65.33 & 63.33 & 65.73 & 64.80 & 1.29 \\
\midrule
\multirow{5}{*}{No rotation}
& GPQA             & 18.69 & 21.72 & 15.15 & 18.52 & 3.29 \\
& HumanEval        & 0.00  & 0.00  & 0.00  & 0.00  & 0.00 \\
& LiveCodeBench v6 & 0.00  & 1.17  & 0.00  & 0.39  & 0.68 \\
& AIME25           & 3.33  & 0.00  & 3.33  & 2.22  & 1.92 \\
& MATH500          & 0.00  & 0.00  & 0.00  & 0.00  & 0.00 \\
\bottomrule
\end{tabular}
\end{table}

Table~\ref{tab:app_sink_recent_full_runs} reports the individual runs behind the sink/recent-window study in Table~\ref{tab:ablation_sink_recent}.

\begin{table}[ht]
\centering
\tiny
\caption{Full individual runs for the sink/recent-window study on Qwen3-4B-Thinking-2507.}
\label{tab:app_sink_recent_full_runs}
\setlength{\tabcolsep}{2pt}
\begin{tabular}{l l c c c c c c c}
\toprule
\textbf{$(S,R)$} & \textbf{Task} & \textbf{Run 1} & \textbf{Run 2} & \textbf{Run 3} & \textbf{Run 4} & \textbf{Run 5} & \textbf{Mean} & \textbf{Std} \\
\midrule
\multirow{5}{*}{$(0,0)$}
& GPQA             & 0.00 & 0.00 & 0.00 & -- & -- & 0.00 & 0.00 \\
& HumanEval        & 0.00 & 0.00 & 0.00 & -- & -- & 0.00 & 0.00 \\
& LiveCodeBench v6 & 0.00 & 0.00 & 0.00 & -- & -- & 0.00 & 0.00 \\
& AIME25           & 0.00 & 0.00 & 0.00 & -- & -- & 0.00 & 0.00 \\
& MATH500          & 0.00 & 0.00 & 0.00 & -- & -- & 0.00 & 0.00 \\
\midrule
\multirow{5}{*}{$(32,128)$}
& GPQA             & 60.10 & 56.57 & 56.06 & -- & -- & 57.58 & 2.20 \\
& HumanEval        & 91.95 & 92.07 & 91.71 & -- & -- & 91.91 & 0.19 \\
& LiveCodeBench v6 & 40.35 & 43.27 & 38.60 & -- & -- & 40.74 & 2.36 \\
& AIME25           & 56.67 & 60.00 & 50.00 & -- & -- & 55.56 & 5.09 \\
& MATH500          & 93.39 & 92.99 & 91.58 & -- & -- & 92.65 & 0.95 \\
\midrule
\multirow{5}{*}{$(64,256)$}
& GPQA             & 66.16 & 63.64 & 64.14 & 64.65 & 66.16 & 64.95 & 1.16 \\
& HumanEval        & 93.17 & 91.59 & 92.44 & 90.85 & 93.17 & 92.24 & 1.02 \\
& LiveCodeBench v6 & 45.03 & 43.27 & 45.61 & 48.54 & 44.44 & 45.38 & 1.97 \\
& AIME25           & 63.33 & 63.33 & 63.33 & 60.00 & 70.00 & 64.00 & 3.65 \\
& MATH500          & 92.59 & 92.79 & 92.99 & 93.19 & 92.18 & 92.75 & 0.39 \\
\midrule
\multirow{5}{*}{$(128,512)$}
& GPQA             & 65.15 & 66.16 & 64.65 & -- & -- & 65.32 & 0.77 \\
& HumanEval        & 93.29 & 92.80 & 93.41 & -- & -- & 93.17 & 0.32 \\
& LiveCodeBench v6 & 45.03 & 45.61 & 45.03 & -- & -- & 45.22 & 0.33 \\
& AIME25           & 70.00 & 66.67 & 66.67 & -- & -- & 67.78 & 1.92 \\
& MATH500          & 93.39 & 93.19 & 93.39 & -- & -- & 93.32 & 0.12 \\
\midrule
\multirow{5}{*}{$(256,1024)$}
& GPQA             & 66.16 & 65.15 & 63.64 & -- & -- & 64.98 & 1.27 \\
& HumanEval        & 92.20 & 93.66 & 93.54 & -- & -- & 93.13 & 0.81 \\
& LiveCodeBench v6 & 45.03 & 48.54 & 45.03 & -- & -- & 46.20 & 2.03 \\
& AIME25           & 66.67 & 70.00 & 66.67 & -- & -- & 67.78 & 1.92 \\
& MATH500          & 93.19 & 93.39 & 93.39 & -- & -- & 93.32 & 0.12 \\
\bottomrule
\end{tabular}
\end{table}

Table~\ref{tab:app_clip_threshold_full_sweep} reports the per-task results behind the clip-threshold grid in Table~\ref{tab:ablation_clip}.

\begin{table}[ht]
\centering
\tiny
\caption{Full per-task results for the clip-threshold sweep on Qwen3-4B-Thinking-2507. Each row is one $(c_K,c_V)$ setting with one run per task.}
\label{tab:app_clip_threshold_full_sweep}
\setlength{\tabcolsep}{2pt}
\begin{tabular}{c c c c c c c c}
\toprule
$c_K$ & $c_V$ & \textbf{GPQA} & \textbf{HumanE} & \textbf{LCB v6} & \textbf{AIME25} & \textbf{MATH500} & \textbf{Mean} \\
\midrule
0.88 & 0.88 & 59.09 & 91.83 & 39.18 & 53.33 & 92.59 & 67.20 \\
0.88 & 0.92 & 62.12 & 91.22 & 39.77 & 56.67 & 91.58 & 68.27 \\
0.88 & 0.96 & 57.07 & 91.10 & 42.69 & 56.67 & 92.99 & 68.10 \\
0.88 & 0.98 & 62.12 & 91.95 & 37.43 & 53.33 & 92.99 & 67.56 \\
0.88 & 1.00 & 57.58 & 91.59 & 46.78 & 53.33 & 91.98 & 68.25 \\
\midrule
0.92 & 0.88 & 61.62 & 93.17 & 41.52 & 56.67 & 92.79 & 69.15 \\
0.92 & 0.92 & 62.12 & 91.22 & 43.86 & 60.00 & 92.79 & 70.00 \\
0.92 & 0.96 & 64.14 & 92.32 & 46.20 & 56.67 & 92.99 & 70.46 \\
0.92 & 0.98 & 65.15 & 93.90 & 48.54 & 53.33 & 92.79 & 70.74 \\
0.92 & 1.00 & 64.65 & 93.90 & 46.78 & 53.33 & 92.99 & 70.33 \\
\midrule
0.96 & 0.88 & 61.62 & 92.80 & 39.77 & 56.67 & 92.99 & 68.77 \\
\textbf{0.96} & \textbf{0.92} & \textbf{64.14} & \textbf{93.54} & \textbf{46.20} & \textbf{56.67} & \textbf{92.38} & \textbf{70.59} \\
0.96 & 0.96 & 63.13 & 91.59 & 43.27 & 63.33 & 92.18 & 70.70 \\
0.96 & 0.98 & 62.12 & 92.68 & 40.35 & 53.33 & 92.38 & 68.17 \\
0.96 & 1.00 & 62.12 & 91.59 & 42.11 & 56.67 & 93.19 & 69.14 \\
\midrule
0.98 & 0.88 & 55.56 & 91.71 & 37.43 & 50.00 & 91.38 & 65.22 \\
0.98 & 0.92 & 55.05 & 91.59 & 36.84 & 53.33 & 91.98 & 65.76 \\
0.98 & 0.96 & 48.48 & 90.37 & 36.84 & 53.33 & 91.58 & 64.12 \\
0.98 & 0.98 & 55.56 & 91.71 & 37.43 & 46.67 & 92.59 & 64.79 \\
0.98 & 1.00 & 56.57 & 91.34 & 39.77 & 56.67 & 92.18 & 67.31 \\
\midrule
1.00 & 0.88 & 48.48 & 89.88 & 30.99 & 40.00 & 92.38 & 60.35 \\
1.00 & 0.92 & 44.95 & 90.98 & 32.16 & 36.67 & 91.18 & 59.19 \\
1.00 & 0.96 & 45.45 & 90.37 & 30.41 & 33.33 & 91.18 & 58.15 \\
1.00 & 0.98 & 43.43 & 90.73 & 32.16 & 36.67 & 90.18 & 58.63 \\
1.00 & 1.00 & 47.47 & 90.73 & 30.41 & 23.33 & 88.18 & 56.02 \\
\bottomrule
\end{tabular}
\end{table}

Table~\ref{tab:app_calibration_regime_full_runs} reports the individual runs behind the calibration-data regime study in Table~\ref{tab:ablation_calibration}.

\begin{table}[ht]
\centering
\tiny
\caption{Full individual runs for the calibration-data regime study on Qwen3-4B-Thinking-2507.}
\label{tab:app_calibration_regime_full_runs}
\setlength{\tabcolsep}{2pt}
\begin{tabular}{l l c c c c c}
\toprule
\textbf{Calibration setting} & \textbf{Task} & \textbf{Run 1} & \textbf{Run 2} & \textbf{Run 3} & \textbf{Mean} & \textbf{Std} \\
\midrule
\multirow{5}{*}{MMLU 2k}
& GPQA             & 61.62 & 58.08 & 62.12 & 60.61 & 2.20 \\
& HumanEval        & 92.68 & 92.44 & 92.44 & 92.52 & 0.14 \\
& LiveCodeBench v6 & 37.43 & 39.18 & 42.11 & 39.57 & 2.36 \\
& AIME25           & 50.00 & 56.67 & 60.00 & 55.56 & 5.09 \\
& MATH500          & 93.39 & 93.39 & 92.59 & 93.12 & 0.46 \\
\midrule
\multirow{5}{*}{MMLU 8k}
& GPQA             & 62.63 & 58.59 & 60.10 & 60.44 & 2.04 \\
& HumanEval        & 91.22 & 92.68 & 92.32 & 92.07 & 0.76 \\
& LiveCodeBench v6 & 40.94 & 41.52 & 39.77 & 40.74 & 0.89 \\
& AIME25           & 66.67 & 56.67 & 56.67 & 60.00 & 5.77 \\
& MATH500          & 92.18 & 93.19 & 92.18 & 92.52 & 0.58 \\
\midrule
\multirow{5}{*}{MMLU 16k}
& GPQA             & 63.64 & 64.65 & 64.14 & 64.14 & 0.51 \\
& HumanEval        & 93.29 & 92.32 & 93.17 & 92.93 & 0.53 \\
& LiveCodeBench v6 & 46.78 & 42.69 & 42.11 & 43.86 & 2.55 \\
& AIME25           & 73.33 & 53.33 & 56.67 & 61.11 & 10.72 \\
& MATH500          & 92.18 & 92.99 & 91.98 & 92.38 & 0.53 \\
\midrule
\multirow{5}{*}{MMLU 32k}
& GPQA             & 62.63 & 61.62 & 59.60 & 61.28 & 1.54 \\
& HumanEval        & 92.07 & 92.56 & 91.95 & 92.20 & 0.32 \\
& LiveCodeBench v6 & 48.54 & 45.61 & 43.86 & 46.00 & 2.36 \\
& AIME25           & 53.33 & 63.33 & 56.67 & 57.78 & 5.09 \\
& MATH500          & 92.59 & 91.98 & 92.18 & 92.25 & 0.31 \\
\midrule
\multirow{5}{*}{WikiText 8k}
& GPQA             & 61.11 & 60.61 & 61.62 & 61.11 & 0.51 \\
& HumanEval        & 93.41 & 92.20 & 92.56 & 92.72 & 0.63 \\
& LiveCodeBench v6 & 39.77 & 43.86 & 42.69 & 42.11 & 2.11 \\
& AIME25           & 56.67 & 56.67 & 63.33 & 58.89 & 3.85 \\
& MATH500          & 93.19 & 92.59 & 92.18 & 92.65 & 0.50 \\
\midrule
\multirow{5}{*}{GPQA-Diamond 8k}
& GPQA             & 63.13 & 64.14 & 61.62 & 62.96 & 1.27 \\
& HumanEval        & 92.56 & 92.56 & 92.56 & 92.56 & 0.00 \\
& LiveCodeBench v6 & 46.20 & 44.44 & 43.27 & 44.64 & 1.47 \\
& AIME25           & 63.33 & 63.33 & 60.00 & 62.22 & 1.92 \\
& MATH500          & 92.38 & 92.59 & 92.99 & 92.65 & 0.31 \\
\bottomrule
\end{tabular}
\end{table}

\clearpage
\subsection{Rotation Factor, Permutation, and Window Diagnostics}
\label{app:permutation}
\label{app:windows}
\label{app:other_ablations}

\textbf{Pure eigenbasis ($U$-only) and post-Hadamard permutations.}
Table~\ref{tab:pure_u} evaluates applying the covariance eigenbasis $U$ alone (without $H_{\mathrm{Had}}$ or $P_{\mathrm{br}}$) on Qwen3-4B-Thinking-2507: GPQA drops from $64.95\pm1.16$ ($U H_{\mathrm{Had}} P_{\mathrm{br}}$) to $30.98\pm1.27$ and $\Delta$PPL increases from $+0.063$ to $+0.477$, confirming that $U$ must be paired with Hadamard mixing to equalize channel dynamic ranges for 4-level quantization.
Table~\ref{tab:permutation} holds $U$, $H_{\mathrm{Had}}$, clipping, and windows fixed at $G{=}64$ and compares five channel-to-group permutations after $U H_{\mathrm{Had}}$, including adaptations of the ZigZag~\citep{lin2024duquant} and MassDiff~\citep{sanjeet2026mixquant} grouping heuristics.
Bit-reversal achieves the lowest token-logit KL, while all five permutations lie within $0.06$ perplexity of one another.

\begin{table}[!htb]
\centering
\footnotesize
\begin{minipage}[t]{0.52\textwidth}
\centering
\caption{Pure-$U$ vs.\ full $U H_{\mathrm{Had}} P_{\mathrm{br}}$ on Qwen3-4B-Thinking-2507 (WikiText-2 simulated, 2048 tokens; GPQA over 3 runs for pure $U$ and 5 seeds for OSCAR/BF16).}
\label{tab:pure_u}
\setlength{\tabcolsep}{3.5pt}
\begin{tabular}{@{}l c c c c@{}}
\toprule
\textbf{Rotation} & \textbf{PPL} & $\boldsymbol{\Delta}$\textbf{PPL} & \textbf{KL} & \textbf{GPQA} \\
\midrule
BF16 & 11.440 & -- & 0 & 67.27$\pm$1.80 \\
$U$ only & 11.917 & $+$0.477 & 0.0323 & 30.98$\pm$1.27 \\
\rowcolor{blue!8}
$U H_{\mathrm{Had}} P_{\mathrm{br}}$ & \textbf{11.503} & \textbf{$+$0.063} & \textbf{0.0160} & \textbf{64.95$\pm$1.16} \\
\bottomrule
\end{tabular}
\end{minipage}\hfill
\begin{minipage}[t]{0.45\textwidth}
\centering
\caption{Post-$U H_{\mathrm{Had}}$ permutations on Qwen3-4B-Thinking-2507 ($G{=}64$, WikiText-2 simulated; BF16 PPL = 11.440).}
\label{tab:permutation}
\setlength{\tabcolsep}{4pt}
\begin{tabular}{@{}l c c@{}}
\toprule
\textbf{Permutation after $U H_{\mathrm{Had}}$} & \textbf{PPL} & \textbf{KL} \\
\midrule
none ($U H_{\mathrm{Had}}$) & 11.552 & 0.01391 \\
random & \textbf{11.535} & 0.01303 \\
ZigZag-style~\citep{lin2024duquant} & 11.595 & 0.01374 \\
MassDiff-style~\citep{sanjeet2026mixquant} & 11.547 & 0.01253 \\
\rowcolor{blue!8}
bit reversal (OSCAR) & 11.585 & \textbf{0.01117} \\
\bottomrule
\end{tabular}
\end{minipage}
\end{table}

\textbf{One-sided windows and positional attention shift.}
Table~\ref{tab:gemma_windows} decouples the sink and recent windows at a fixed 320-token BF16 budget on Gemma-4-12B-it: sink-only protection $(320,0)$ collapses to $0.00$ on both AIME25 and LiveCodeBench v6, whereas recent-only protection $(0,320)$ remains functional ($38.33$ and $50.86$), showing that preserving recent reasoning tokens is the primary requirement during long generation.
Table~\ref{tab:positional_bias} measures one-step decode attention on Qwen3-8B at 4K context over eight blocks: the $(64,256)$ windows lower attention KL from $0.551$ to $0.186$, while quantization noise slightly increases attention mass on the middle INT2 history ($+0.054$) rather than biasing attention toward the clean recent window ($-0.040$).

\begin{table}[!htb]
\centering
\footnotesize
\begin{minipage}[t]{0.47\textwidth}
\centering
\caption{One-sided and enlarged BF16 windows on Gemma-4-12B-it with OSCAR INT2.}
\label{tab:gemma_windows}
\setlength{\tabcolsep}{4pt}
\begin{tabular}{@{}l c c c@{}}
\toprule
$(S_0,W)$ & \textbf{AIME25} & \textbf{LCB v6} & \textbf{Seeds} \\
\midrule
BF16 & 54.17$\pm$0.72 & 62.86$\pm$0.99 & 3 \\
$(320, 0)$ sink only & 0.00 & 0.00 & 1 \\
$(0, 320)$ recent only & 38.33 & 50.86 & 1 \\
$(32, 288)$ & 40.00 & 52.57 & 1 \\
\rowcolor{blue!8}
$(64, 256)$ default & 40.14$\pm$2.44 & 54.10$\pm$1.84 & 3 \\
$(64, 512)$ & \textbf{48.19$\pm$2.77} & \textbf{60.95$\pm$0.66} & 3 \\
\bottomrule
\end{tabular}
\end{minipage}\hfill
\begin{minipage}[t]{0.50\textwidth}
\centering
\caption{Attention KL and change in attention mass (quantized minus FP16) by position bin on Qwen3-8B at 4K context.}
\label{tab:positional_bias}
\setlength{\tabcolsep}{3pt}
\begin{tabular}{@{}l c c c c@{}}
\toprule
\textbf{Policy} & \textbf{Attn KL} & $\boldsymbol{\Delta}$\textbf{sink} & $\boldsymbol{\Delta}$\textbf{hist} & $\boldsymbol{\Delta}$\textbf{recent} \\
\midrule
\rowcolor{blue!8}
default $(64,256)$ & \textbf{0.186} & $-$0.014 & $+$0.054 & $-$0.040 \\
no windows & 0.551 & $-$0.023 & $+$0.025 & $-$0.003 \\
sink only $(320,0)$ & 0.448 & $-$0.023 & $+$0.020 & $+$0.002 \\
recent only $(0,320)$ & 0.271 & $-$0.013 & $+$0.055 & $-$0.042 \\
\bottomrule
\end{tabular}
\end{minipage}
\end{table}

\clearpage
\subsection{Long-Context Diagnostics, Additional Models, and Baseline Comparisons}
\label{app:long_context}
\label{app:more_models}
\label{app:baselines}

\textbf{64K K/V isolation and large-model RULER.}
Table~\ref{tab:kv_isolation} isolates K-only and V-only quantization on Qwen3-8B after a 64K-token WikiText-2 prefix (512 scored tokens): V-only INT2 is virtually indistinguishable from BF16 (PPL $8.218$ vs.\ $8.217$, KL $0.0039$), whereas K-only INT2 accounts for the full perplexity rise ($23.318$, KL $1.1371$), confirming that long-context degradation stems from key-side logit error.
On MiniMax-M2.7 (229B MoE, TP=4), paired evaluation across all 500 examples of RULER \texttt{niah\_single\_1} achieves $100.0\%$ retrieval for both BF16 and OSCAR at 64K and 128K tokens, matching GLM-4.7 ($97.2\%$ vs.\ $97.2\%$ at 128K in Table~\ref{tab:long_context_ruler}).

\begin{table}[!htb]
\centering
\footnotesize
\caption{K/V and window isolation on Qwen3-8B after a 64K-token prefix (simulated WikiText-2, 512 scored tokens).}
\label{tab:kv_isolation}
\setlength{\tabcolsep}{6pt}
\begin{tabular}{@{}l c c c c@{}}
\toprule
\textbf{Configuration} & \textbf{INT2 tensors} & $(S_0,W)$ & \textbf{PPL} & \textbf{KL} \\
\midrule
BF16 & -- & -- & 8.217 & 0 \\
V-only INT2 & V & $(64,256)$ & 8.218 & 0.0039 \\
K-only INT2 & K & $(64,256)$ & 23.318 & 1.1371 \\
\rowcolor{blue!8}
K+V INT2 (default) & K, V & $(64,256)$ & 21.862 & 1.0867 \\
K+V INT2, recent=512 & K, V & $(64,512)$ & 21.712 & 1.0558 \\
recent-only & K, V & $(0,320)$ & 23.969 & 1.1483 \\
sink-only & K, V & $(320,0)$ & 47.492 & 1.8010 \\
no windows & K, V & $(0,0)$ & 57.190 & 1.9484 \\
\bottomrule
\end{tabular}
\end{table}

\textbf{Broader model families, tool calling (BFCL), and reasoning-length truncation.}
Table~\ref{tab:more_models} extends evaluation to six additional models spanning hybrid linear/full attention (Qwen3.5-4B using uniform INT2; Qwen3.5-35B-A3B using Lloyd--Max levels), sliding/global attention (Gemma-4-12B-it), large MoE (MiniMax-M2.7 at 229B, which also scores $70.8\%$ on SWE-bench Verified under OSCAR INT2; and MiniMax-M3 at 428B under dense attention, with calibrated rotation on GPQA and Hadamard rotation on HumanEval/AIME25/MATH500), and 753B MLA (GLM-5.2 quantizing the 512-d latent $c_{\mathrm{kv}}$ with covariance rotation and Lloyd--Max while keeping $k_{\mathrm{pe}}$ in BF16).
Table~\ref{tab:bfcl} reports multi-turn function calling on BFCL~\citep{patil2025bfcl}.
On GLM-5.2 GPQA at a 32K generation cap, INT2 lengthens median reasoning traces by $1.69\times$ and raises truncation from $14\%$ to $28\%$, while option agreement among jointly answered questions is $97.1\%$, showing that larger drops on long-CoT tasks are primarily driven by budget truncation~\citep{lotfi2026quantized}.

\begin{table}[!htb]
\centering
\footnotesize
\begin{minipage}[t]{0.64\textwidth}
\centering
\caption{BF16 $\to$ INT2 accuracy (\%) on additional model families (3 seeds at 32K cap; Qwen3.5-35B-A3B AIME25: 8 seeds; MiniMax-M3 GPQA: 1 run; $^\dagger$MiniMax-M2.7 cells report single-seed OSCAR-only accuracy, with a 95K cap on AIME25/LCB and $70.8\%$ on SWE-bench Verified).}
\label{tab:more_models}
\setlength{\tabcolsep}{2.5pt}
\resizebox{\linewidth}{!}{%
\begin{tabular}{@{}l c c c c c c@{}}
\toprule
\textbf{Model} & \textbf{GPQA} & \textbf{HumanE} & \textbf{AIME25} & \textbf{MATH500} & \textbf{LCB v6} & \textbf{MMMU-Pro} \\
\midrule
Qwen3.5-4B & 76.9$\to$75.8 & 81.7$\to$83.9 & 47.8$\to$46.7 & 89.5$\to$88.0 & -- & -- \\
Qwen3.5-35B-A3B & 83.3$\to$84.0 & 83.9$\to$86.6 & 66.7$\to$62.1 & 92.8$\to$91.7 & -- & -- \\
Gemma-4-12B-it & 63.47$\to$60.94 & 93.70$\to$93.09 & 54.17$\to$40.14 & 93.73$\to$90.13 & 62.86$\to$54.10 & 42.89$\to$40.67 \\
MiniMax-M2.7$^\dagger$ & 80.3$^\dagger$ & 92.1$^\dagger$ & 90.0$^\dagger$ & 94.6$^\dagger$ & 68.4$^\dagger$ & -- \\
MiniMax-M3 & 88.38$\to$85.35 & 93.1$\to$91.7 & 77.8$\to$67.8 & 79.7$\to$78.7 & -- & -- \\
GLM-5.2 & 81.3$\to$68.7 & 91.5$\to$88.2 & 76.7$\to$75.6 & 94.8$\to$93.8 & -- & -- \\
\bottomrule
\end{tabular}%
}
\end{minipage}\hfill
\begin{minipage}[t]{0.34\textwidth}
\centering
\caption{BFCL multi-turn accuracy (\%), BF16 $\to$ INT2 (single runs).}
\label{tab:bfcl}
\setlength{\tabcolsep}{3.5pt}
\begin{tabular}{@{}l c c@{}}
\toprule
\textbf{Model} & \textbf{Overall} & \textbf{Long-ctx} \\
\midrule
Qwen3-8B & 38.4$\to$37.1 & 30.0$\to$23.0 \\
Qwen3-32B & 49.1$\to$44.4 & 44.5$\to$36.0 \\
Qwen3.5-4B & 50.1$\to$47.6 & 48.0$\to$49.5 \\
\bottomrule
\end{tabular}
\end{minipage}
\end{table}

\textbf{Comparisons with KIVI, Kitty, and TurboQuant, and upstream integration validation.}
Table~\ref{tab:baselines_app} compares teacher-forced WikiText-2 perplexity and token-logit KL on Qwen3-8B (2048 scored tokens after 512-token prefixes and 512 scored tokens after 64K-token prefixes, matching OSCAR's BF16 windows on TurboQuant) alongside same-harness 4K-cap downstream evaluations using Kitty's official HF evaluation protocol (Table~\ref{tab:kitty_harness}).
On 5-shot GSM8K~\citep{cobbe2021gsm8k} (n=1319, 4K cap) on Qwen3-4B-Thinking-2507, TurboQuant's native mixed-precision vLLM preset (K3V3+NC with 4 BF16 boundary layers, 4.61 BPE) achieves $81.05/51.55$ strict/flexible accuracy versus $83.09/58.07$ for OSCAR in SGLang (3.55 stored BPE) and $87.19/61.87$ for BF16.
Table~\ref{tab:long_decode} reports wall time for a fixed 64K-input, 32K-output decode benchmark on Qwen3-8B (batch size 1, $1\times$H100), and Table~\ref{tab:upstream_port} verifies our fresh-upstream SGLang port~\citep{oscar2026sglangpr} against the research fork reference on Qwen3-8B.

\begin{table}[!htb]
\centering
\footnotesize
\begin{minipage}[t]{0.53\textwidth}
\centering
\caption{Teacher-forced WikiText-2 PPL and token-logit KL on Qwen3-8B (simulated quantization).}
\label{tab:baselines_app}
\setlength{\tabcolsep}{3pt}
\begin{tabular}{@{}l c c c c@{}}
\toprule
\textbf{Method} & \textbf{K/V} & \textbf{PPL} & \textbf{KL} & \textbf{PPL@64K} \\
\midrule
FP16 & 16/16 & 10.155 & 0 & 8.217 \\
KIVI~\citep{liu2024kivi} & 2/2 & 11.251 & 0.1316 & -- \\
KIVI* (32-tok sink) & 2/2 & 10.623 & 0.0440 & -- \\
Kitty (12.5\% K INT4) & 2/2 & 10.311 & 0.0249 & -- \\
TurboQuant~\citep{zandieh2025turboquant} & 2/2 & 50.769 & 1.9467 & 973.368 \\
TurboQuant~\citep{zandieh2025turboquant} & 4/2 & 11.810 & 0.1814 & \textbf{9.967} \\
\rowcolor{blue!8}
\textbf{OSCAR} & 2/2 & \textbf{10.167} & \textbf{0.0106} & 21.862 \\
\bottomrule
\end{tabular}
\end{minipage}\hfill
\begin{minipage}[t]{0.45\textwidth}
\centering
\caption{Same-harness 4K-cap evaluation on Qwen3-8B using Kitty's official protocol (single runs).}
\label{tab:kitty_harness}
\setlength{\tabcolsep}{3.5pt}
\begin{tabular}{@{}l c c c@{}}
\toprule
\textbf{Method} & \textbf{GPQA} & \textbf{HumanE} & \textbf{AIME25} \\
\midrule
FP16 & 42.35 & 85.98 & 26.67 \\
KIVI & 33.16 & 75.61 & 13.33 \\
KIVI* & 46.43 & 84.76 & 20.00 \\
Kitty & 43.88 & 87.80 & 13.33 \\
\rowcolor{blue!8}
\textbf{OSCAR} & 43.37 & \textbf{87.80} & \textbf{20.00} \\
\bottomrule
\end{tabular}
\end{minipage}
\end{table}

\begin{table}[!htb]
\centering
\footnotesize
\begin{minipage}[t]{0.48\textwidth}
\centering
\caption{Fixed 64K-input + 32K-output decode on Qwen3-8B ($1\times$H100, batch size 1).}
\label{tab:long_decode}
\setlength{\tabcolsep}{3.5pt}
\begin{tabular}{@{}l c c@{}}
\toprule
\textbf{Method / Runtime} & \textbf{Wall time} & \textbf{Output tok/s} \\
\midrule
\rowcolor{blue!8}
\textbf{OSCAR (SGLang)} & \textbf{5.43 min} & \textbf{100.65} \\
HF \texttt{generate} (FP16) & 55.29 min & 9.88 \\
KIVI Qwen3 simulator & 61.07 min & 8.94 \\
Kitty released path & 93.27 min & 5.86 \\
\bottomrule
\end{tabular}
\end{minipage}\hfill
\begin{minipage}[t]{0.50\textwidth}
\centering
\caption{Upstream SGLang integration validation on Qwen3-8B (GPQA-Diamond 198 questions, 32K cap).}
\label{tab:upstream_port}
\setlength{\tabcolsep}{2.5pt}
\resizebox{\linewidth}{!}{%
\begin{tabular}{@{}l c c c@{}}
\toprule
\textbf{Implementation} & \textbf{GPQA} & \textbf{8.6K Needle} & \textbf{Failures} \\
\midrule
Research fork reference & 56.06 & -- & 0 \\
\rowcolor{blue!8}
\textbf{Fresh-upstream port} & \textbf{57.07} & \textbf{exact (8,659 tok)} & \textbf{0} \\
\bottomrule
\end{tabular}%
}
\end{minipage}
\end{table}

\clearpage
\section*{NeurIPS Paper Checklist}

\begin{enumerate}

\item {\bf Claims}
    \item[] Question: Do the main claims made in the abstract and introduction accurately reflect the paper's contributions and scope?
    \item[] Answer: \answerYes{}.
    \item[] Justification: The abstract and Section~\ref{sec:introduction} state the contributions and their scope, including the remaining long-context gap on Qwen3 and the surrogate nature of the theory; Sections~\ref{sec:method}--\ref{sec:experiments} support them, and Section~\ref{sec:introduction} lists which components are adopted from prior work.

\item {\bf Limitations}
    \item[] Question: Does the paper discuss the limitations of the work performed by the authors?
    \item[] Answer: \answerYes{}.
    \item[] Justification: Limitations are discussed in Appendix~\ref{sec:discussion}, covering the surrogate scope of the theoretical analysis, numerical format and precision choices, hardware and serving-stack scope, and long-context behavior.

\item {\bf Theory assumptions and proofs}
    \item[] Question: For each theoretical result, does the paper provide the full set of assumptions and a complete (and correct) proof?
    \item[] Answer: \answerYes{}.
    \item[] Justification: Theorem~\ref{thm:ambient_basis_simple_variants} states its assumptions (positive semidefinite target, fixed diagonal residual with the stated ordering); the proof is in Appendix~\ref{app:ambient_basis_simple_variants}, the approximations leading to the surrogate are in Appendix~\ref{app:surrogate_objective_justification}, and Appendix~\ref{app:residual_check} checks the diagonal assumption empirically.

    \item {\bf Experimental result reproducibility}
    \item[] Question: Does the paper fully disclose all the information needed to reproduce the main experimental results of the paper to the extent that it affects the main claims and/or conclusions of the paper (regardless of whether the code and data are provided or not)?
    \item[] Answer: \answerYes{}.
    \item[] Justification: Section~\ref{subsec:exp_setup} gives the models, tasks, decoding parameters, seeds, calibration data, and OSCAR hyperparameters; Appendix~\ref{app:algorithm_flow} gives the algorithm, GQA conventions, permutation index map, quantizer formulas, and bit accounting; Appendix~\ref{app:additional_results} gives per-run results and the protocols of all additional experiments.

\item {\bf Open access to data and code}
    \item[] Question: Does the paper provide open access to the data and code, with sufficient instructions to faithfully reproduce the main experimental results, as described in supplemental material?
    \item[] Answer: \answerYes{}.
    \item[] Justification: The SGLang implementation is available at \url{https://github.com/FutureMLS-Lab/OSCAR}, and the calibrated rotations at \url{https://huggingface.co/Zhongzhu/OSCAR-RotationZoo}; all evaluation data are public benchmarks.

\item {\bf Experimental setting/details}
    \item[] Question: Does the paper specify all the training and test details (e.g., data splits, hyperparameters, how they were chosen, type of optimizer) necessary to understand the results?
    \item[] Answer: \answerYes{}.
    \item[] Justification: Section~\ref{subsec:exp_setup} specifies the models, benchmarks, generation settings, number of seeds, calibration data, and hardware; OSCAR involves no training, and its hyperparameters (window sizes, clip ratios, group sizes) are fixed as stated, with sensitivity studies in Appendix~\ref{app:additional_results}.

\item {\bf Experiment statistical significance}
    \item[] Question: Does the paper report error bars suitably and correctly defined or other appropriate information about the statistical significance of the experiments?
    \item[] Answer: \answerYes{}.
    \item[] Justification: Accuracy tables report the mean and sample standard deviation over five seeds (Qwen3) or three runs (GLM-4.7), which capture sampling randomness during generation; single-run results are marked as such, and Appendix~\ref{app:additional_results} lists every run.

\item {\bf Experiments compute resources}
    \item[] Question: For each experiment, does the paper provide sufficient information on the computer resources (type of compute workers, memory, time of execution) needed to reproduce the experiments?
    \item[] Answer: \answerYes{}.
    \item[] Justification: The main experiments run on NVIDIA H100 80\,GB GPUs (one GPU for the 4B and 8B models, two for Qwen3-32B, eight for GLM-4.7, four for MiniMax-M2.7), and the largest additional models in Appendix~\ref{app:more_models} use two 8-GPU nodes; calibration is a single prefill pass over 8,878 tokens followed by $d\times d$ eigendecompositions. Appendix~\ref{app:baselines} reports wall times for long decoding. The project also used compute for preliminary experiments that are not reported.
    
\item {\bf Code of ethics}
    \item[] Question: Does the research conducted in the paper conform, in every respect, with the NeurIPS Code of Ethics \url{https://neurips.cc/public/EthicsGuidelines}?
    \item[] Answer: \answerYes{}.
    \item[] Justification: The work evaluates publicly released models on public benchmarks to improve serving efficiency; we have reviewed the NeurIPS Code of Ethics and identify no deviation.

\item {\bf Broader impacts}
    \item[] Question: Does the paper discuss both potential positive societal impacts and negative societal impacts of the work performed?
    \item[] Answer: \answerYes{}.
    \item[] Justification: Appendix~\ref{sec:discussion} discusses the lower cost and energy of long-context serving, the corresponding reduction in the cost of misuse, and the risk of quality changes that standard benchmarks miss.
    
\item {\bf Safeguards}
    \item[] Question: Does the paper describe safeguards that have been put in place for responsible release of data or models that have a high risk for misuse (e.g., pre-trained language models, image generators, or scraped datasets)?
    \item[] Answer: \answerNA{}.
    \item[] Justification: This paper does not release a pretrained model or dataset with high misuse risk; it proposes a KV-cache quantization method for existing models.

\item {\bf Licenses for existing assets}
    \item[] Question: Are the creators or original owners of assets (e.g., code, data, models), used in the paper, properly credited and are the license and terms of use explicitly mentioned and properly respected?
    \item[] Answer: \answerYes{}.
    \item[] Justification: We cite every model, benchmark, serving system, and baseline implementation we use (Section~\ref{subsec:exp_setup} and Appendix~\ref{app:baselines}); all are publicly released for research use under their respective licenses (e.g., Apache-2.0 for Qwen3, SGLang, and vLLM; MIT for HumanEval, GSM8K, and MATH; CC BY 4.0 for GPQA), and we use them only for evaluation.

\item {\bf New assets}
    \item[] Question: Are new assets introduced in the paper well documented and is the documentation provided alongside the assets?
    \item[] Answer: \answerYes{}.
    \item[] Justification: We release the SGLang implementation and the calibrated rotation matrices; their usage is documented in the linked repositories. We release no new dataset or model weights.

\item {\bf Crowdsourcing and research with human subjects}
    \item[] Question: For crowdsourcing experiments and research with human subjects, does the paper include the full text of instructions given to participants and screenshots, if applicable, as well as details about compensation (if any)? 
    \item[] Answer: \answerNA{}.
    \item[] Justification: The work does not involve crowdsourcing or human-subject experiments.

\item {\bf Institutional review board (IRB) approvals or equivalent for research with human subjects}
    \item[] Question: Does the paper describe potential risks incurred by study participants, whether such risks were disclosed to the subjects, and whether Institutional Review Board (IRB) approvals (or an equivalent approval/review based on the requirements of your country or institution) were obtained?
    \item[] Answer: \answerNA{}.
    \item[] Justification: The work does not involve human subjects and therefore does not require IRB approval.

\item {\bf Declaration of LLM usage}
    \item[] Question: Does the paper describe the usage of LLMs if it is an important, original, or non-standard component of the core methods in this research? Note that if the LLM is used only for writing, editing, or formatting purposes and does \emph{not} impact the core methodology, scientific rigor, or originality of the research, declaration is not required.
    \item[] Answer: \answerNA{}.
    \item[] Justification: The core method does not use LLMs as an original or non-standard component; LLMs are only the inference targets evaluated by the quantization method.

\end{enumerate}

\end{document}